\documentclass[11pt,a4paper]{article}
\selectfont
\usepackage[a4paper,textwidth=15.2cm,lines=46,centering,headheight=12pt,headsep=22pt,footskip=30pt]{geometry}
\usepackage[table]{xcolor}
\usepackage{etoolbox}
\usepackage{amsmath}
\usepackage{amsthm}
\usepackage{fontspec}
\usepackage{unicode-math}

\newfontfamily\titlefamily{fonts.ttc}[FontIndex=8]
\newfontfamily\headingfamily{fonts.ttc}[FontIndex=7]
\newfontfamily\labelfamily{fonts.ttc}[FontIndex=7,LetterSpace=7]
\newcommand{\labelfont}{\labelfamily\fontsize{8.2}{10}\selectfont}

\usepackage{microtype}
\usepackage{booktabs}
\usepackage{array}
\usepackage{enumitem}
\usepackage{caption}
\usepackage{titlesec}
\usepackage{fancyhdr}
\usepackage[algo2e,ruled]{algorithm2e}
\usepackage[authoryear,round]{natbib}
\usepackage{hyperref}
\usepackage{cleveref}

\definecolor{srcdark}{HTML}{00657E}
\definecolor{extdark}{HTML}{323BB5}
\definecolor{extmid}{HTML}{4F63E7}
\definecolor{tgtdark}{HTML}{915100}
\definecolor{cmpdark}{HTML}{9E122F}
\definecolor{ink}{HTML}{1A1E3E}
\definecolor{greydark}{HTML}{434756}
\definecolor{greymid}{HTML}{828693}
\definecolor{greylight}{HTML}{CED1D9}
\colorlet{accent}{extmid}
\colorlet{accentdark}{extdark}

\arrayrulecolor{ink}
\newlength{\pillpad}
\newcommand{\tablestyle}{\small\setlength{\tabcolsep}{4pt}\renewcommand{\arraystretch}{1.3}}
\newcommand{\colhead}[1]{{\leavevmode\sffamily\bfseries\boldmath\color{ink}#1}}
\newcommand{\tablegroup}[1]{{\sffamily\itshape\color{accentdark}#1}}
\newcommand{\paneltitle}[1]{{\sffamily\bfseries\color{ink}#1}}
\newcommand{\grouprule}{\arrayrulecolor{greylight}\midrule\arrayrulecolor{ink}}
\newcommand{\pill}[1][c]{\if#1l\hspace*{-\pillpad}\fi\hvpic\if#1r\hspace*{-\pillpad}\fi}
\newcommand{\figrule}{\noalign{\vskip3pt{\color{greylight}\hrule height0.45pt}\vskip3pt}}


\everydisplay\expandafter{\the\everydisplay\lineskiplimit=0pt\relax}
\makeatletter
\pretocmd{\@array}{\lineskiplimit=0pt\relax}{}{\errmessage{array patch failed}}
\newcommand*{\@tune}[3]{\expandafter\ifx\csname @saved#1\endcsname\relax\expandafter\xdef\csname @saved#1\endcsname{\the#2}\fi#2=#3\relax\global\let\@resettune\@doresettune}
\newcommand*{\@restoretune}[2]{\expandafter\ifx\csname @saved#1\endcsname\relax\else\global#2=\csname @saved#1\endcsname\relax\global\expandafter\let\csname @saved#1\endcsname\relax\fi}
\newcommand*{\@doresettune}{\global\looseness=\z@\@restoretune{es}\emergencystretch\@restoretune{fh}\finalhyphendemerits\global\let\@resettune\relax}\let\@resettune\relax
\AddToHook{para/after}{\@resettune}
\everydisplay\expandafter{\the\everydisplay\@resettune}
\newcommand*{\tighter}{\@tune{es}\emergencystretch\z@\looseness=\m@ne\relax}\newcommand*{\looser}{\@tune{es}\emergencystretch\z@\looseness=\@ne\relax}
\newcommand*{\keeptail}{\@tune{fh}\finalhyphendemerits{100000000}}
\makeatother

\titleformat{\section}{\headingfamily\fontsize{14}{18}\selectfont\color{ink}}{\llap{\color{accent}\thesection\hspace{0.8em}}}{0pt}{}
\titleformat{\subsection}{\headingfamily\fontsize{11.5}{15}\selectfont\color{ink}}{\llap{\color{accent}\thesubsection\hspace{0.8em}}}{0pt}{}
\titlespacing*{\section}{0pt}{1.6\baselineskip}{0.3\baselineskip}
\titlespacing*{\subsection}{0pt}{1.1\baselineskip}{0.15\baselineskip}

\newcommand{\itemrosette}[1]{\hvpic}
\newcommand{\steprosette}[1]{\hvpic}
\newcommand{\flowertag}[1]{\hvpic}
\setlist{nosep,parsep=\parskip,beginpenalty=10000}
\setlist[enumerate,1]{label=\arabic*,ref=\textup{(\arabic*)},format=\itemrosette,leftmargin=0pt,labelwidth=2.4em,labelsep=0.65em,align=right}
\setlist[enumerate,2]{label=\textup{(\Roman*)},ref=\textup{(\Roman*)},leftmargin=*,widest=VII}
\newlist{steps}{enumerate}{1}
\setlist[steps]{label=\arabic*,ref=\arabic*,format=\steprosette,leftmargin=0pt,labelwidth=2.4em,labelsep=0.65em,align=right,beginpenalty=10000}
\SetEnumitemKey{hypothesis}{label=\textup{H\arabic*},ref=\textup{H\arabic*},format=\flowertag}

\newtheoremstyle{statement}{0.3\baselineskip}{0.3\baselineskip}{\upshape}{0pt}{\labelfont\color{accentdark}}{}{0.8em}{\MakeUppercase{\thmname{#1}}\thmnumber{ #2}\thmnote{ {\normalfont(#3)}}}
\theoremstyle{statement}
\newtheorem{theorem}{Theorem}[section]
\newtheorem{proposition}[theorem]{Proposition}
\newtheorem{lemma}[theorem]{Lemma}
\newtheorem{corollary}[theorem]{Corollary}
\newtheorem{definition}[theorem]{Definition}
\newtheorem{example}[theorem]{Example}
\newtheorem{remark}[theorem]{Remark}
\numberwithin{equation}{section}

\renewcommand{\qedsymbol}{\hvpic}
\renewcommand{\qed}{\leavevmode\unskip\penalty9999 \hbox{}\nobreak\hfill\hskip0.5em\relax\hbox{\qedsymbol}}
\makeatletter
\renewenvironment{proof}[1][\proofname]{\par\pushQED{\qed}\normalfont\topsep0.3\baselineskip\relax\trivlist\item[\hskip\labelsep{\labelfont\color{greydark}\MakeUppercase{#1}}\hspace{0.3em}]\ignorespaces}{\popQED\endtrivlist\@endpefalse}
\makeatother

\crefname{theorem}{Theorem}{Theorems}
\crefname{proposition}{Proposition}{Propositions}
\crefname{lemma}{Lemma}{Lemmas}
\crefname{corollary}{Corollary}{Corollaries}
\crefname{definition}{Definition}{Definitions}
\crefname{example}{Example}{Examples}
\crefname{remark}{Remark}{Remarks}
\crefname{figure}{Figure}{Figures}
\crefname{table}{Table}{Tables}
\crefname{section}{Section}{Sections}
\crefname{subsection}{Section}{Sections}
\crefname{appendix}{Appendix}{Appendices}
\crefname{algocf}{Algorithm}{Algorithms}
\Crefname{algocf}{Algorithm}{Algorithms}
\crefname{enumi}{item}{items}

\DeclareCaptionFont{label}{\labelfont\color{accentdark}}
\DeclareCaptionLabelFormat{upper}{\bothIfFirst{\MakeUppercase{#1}}{\nobreakspace}#2}
\SetAlgorithmName{ALGORITHM}{algorithm}{List of Algorithms}

\SetAlCapSty{algorithmcaptionlabel}
\SetAlCapFnt{\labelfont\color{accentdark}}
\SetAlCapNameFnt{\small}
\SetAlgoCaptionSeparator{\hspace{0.4em}}
\SetKwSty{algorithmkeyword}
\SetNlSty{algorithmlinenumber}{}{}
\SetAlgoNlRelativeSize{0}
\SetKwInput{KwInput}{Input}
\SetKwInput{KwOutput}{Output}
\DontPrintSemicolon
\LinesNumbered
\makeatletter
\newcommand{\eol}{\@endalgoln}
\makeatother
\SetKwFunction{PHfun}{PersistentHomology}
\SetKwFunction{Ckfun}{Cokernel}
\SetKwFunction{Bcfun}{Barcode}
\SetKwFunction{Bnfun}{Bottleneck}
\setcitestyle{citesep={,},aysep={,},yysep={,}}

\makeatletter
\newcommand*{\doi}{\begingroup\catcode`\_=12\relax\doi@link}
\DeclareUrlCommand\doitext{}
\newcommand*{\doi@link}[1]{\unskip\unskip~\hskip0.11em\relax\href{https://doi.org/#1}{doi:\doitext{#1}}\endgroup}
\renewcommand{\footnoterule}{\kern-5.2pt\nointerlineskip\hbox{\hvpic}\kern5.2pt\nointerlineskip}
\renewcommand{\@makefntext}[1]{\noindent\llap{\@makefnmark\kern0.15em}#1}
\apptocmd{\thebibliography}{\clubpenalty\@M\widowpenalty\@M\rightskip\z@\@plus8em\relax\spaceskip\fontdimen2\fontdimen2\newblock{\hskip0.11em\relax}\hyphenpenalty\@M}{}{}
\makeatother

\makeatletter
\newcommand*{\subtitle}[1]{\gdef\@subtitle{#1}}
\newcommand*{\shorttitle}[1]{\gdef\@shorttitle{#1}}
\newcommand*{\orcid}[1]{\gdef\@orcid{#1}}
\newcommand*{\email}[1]{\gdef\@email{#1}}
\newcommand*{\repository}[1]{\gdef\@repository{#1}}
\newcommand*{\affiliation}[1]{\gdef\@affiliation{#1}}
\newcommand*{\keywords}[1]{\gdef\@keywords{#1}}
\newcommand{\titlepagenotes}[1]{\gdef\@titlepagenotes{#1}}

\newcommand{\authorlinks}{\hvpic\hspace{0.32em}\raisebox{-0.18em}{\hvpic}\hspace{0.32em}\hvpic}

\renewcommand{\maketitle}{
  \thispagestyle{titlepage}
  \begingroup
  \raggedright
  \hvpic\par
  \vspace{13pt}
  {\titlefamily\fontsize{28}{32}\selectfont\color{ink}\@title\par}
  \vspace{6pt}
  {\headingfamily\fontsize{15}{19}\selectfont\color{greydark}\@subtitle\par}
  \vspace{18pt}
  {\headingfamily\fontsize{11}{14}\selectfont\color{ink}\@author\hspace{0.6em}\authorlinks\par}
  \vspace{2pt}
  {\sffamily\small\color{greydark}\@affiliation\par}
  \endgroup
  \vspace{20pt}}

\newsavebox{\abstractbox}
\renewenvironment{abstract}
  {\begin{lrbox}{\abstractbox}\begin{minipage}[t]{\dimexpr\linewidth-14pt\relax}\small\setlength{\parskip}{0.35\baselineskip}{\labelfont\color{accentdark}\MakeUppercase{\abstractname}}\par}
  {\par{\labelfont\color{accentdark}KEYWORDS}\hspace{0.7em}\@keywords\end{minipage}\end{lrbox}
   \par\noindent\hvpic\hspace{4pt}\usebox{\abstractbox}\par\vspace{0.6\baselineskip}}

\fancypagestyle{titlepage}{\fancyhf{}\fancyfoot[L]{\parbox[t]{\textwidth}{\raggedright\sffamily\scriptsize\color{greydark}\@titlepagenotes}}}
\makeatother

\title{Learning Transfers}
\subtitle{Kan Extensions of Topological Invariants\\of Learned Representations}
\shorttitle{Kan Extensions of Topological Invariants}
\author{Luciano Melodia}
\orcid{0000-0002-7584-7287}
\email{luciano.melodia@fau.de}
\repository{https://codeberg.org/Jiren/CatTrans}
\affiliation{Lehrstuhl für theoretische Informatik\\Friedrich-Alexander-Universität Erlangen-Nürnberg}
\titlepagenotes{Funded by the Deutsche Forschungsgemeinschaft (DFG, German Research Foundation) – project number 569130867.\\\copyright\,2026 Luciano Melodia. License: CC-BY 4.0, see \url{https://creativecommons.org/licenses/by/4.0/}.}

\hypersetup{colorlinks,linkcolor=accentdark,citecolor=accentdark,urlcolor=accentdark,pdftitle={Learning Transfers: Kan Extensions of Topological Invariants of Learned Representations},pdfauthor={Luciano Melodia},pdfsubject={Preprint},pdfkeywords={Kan extensions, persistent homology, transfer learning}}

\newcommand{\src}[1]{\textcolor{srcdark}{#1}}
\newcommand{\tgt}[1]{\textcolor{tgtdark}{#1}}
\newcommand{\ext}[1]{\textcolor{extdark}{#1}}
\newcommand{\cmp}[1]{\textcolor{cmpdark}{#1}}

\newcommand{\K}{{\mathbb{K}}}
\newcommand{\A}{{\mathcal{A}}}
\newcommand{\B}{{\mathcal{B}}}
\newcommand{\V}{{\mathcal{V}}}
\newcommand{\Vect}{\operatorname{Vec}_{\K}}
\newcommand{\Ch}{\operatorname{Ch}_{\K}}
\newcommand{\Pers}{\operatorname{Pers}_{\K}}
\newcommand{\Ob}{\operatorname{Ob}}
\newcommand{\Mor}{\operatorname{Mor}}
\newcommand{\Dgm}{\operatorname{Dgm}}
\newcommand{\Comp}{\operatorname{Comp}}
\newcommand{\Lan}{\operatorname{Lan}}
\newcommand{\Sep}{\operatorname{Sep}}
\newcommand{\coker}{\operatorname{coker}}
\newcommand{\VR}{\operatorname{VR}}
\newcommand{\PH}{\operatorname{PH}}
\newcommand{\dist}{\operatorname{dist}}
\newcommand{\tr}{\operatorname{tr}}
\newcommand{\Fun}[2]{#2^{#1}}
\newcommand{\coloneqq}{\coloneq}

\newcommand{\itemref}[2]{\Cref{#1}\labelcref{#2}}
\newcommand{\dc}{\rlap{\,,}}
\newcommand{\dd}{\rlap{\,.}}
\newcommand{\by}[2]{\mathrel{\vbox{\offinterlineskip\ialign{\hfil##\hfil\cr\hbox to 0pt{\hss\hvpic\hss}\cr\noalign{\kern1.8pt}\(#1\)\cr}}}}
\newcommand{\eqby}[1]{\by{=}{#1}}
\newcommand{\leby}[1]{\by{\leq}{#1}}
\newcommand{\geby}[1]{\by{\geq}{#1}}
\newcommand{\rref}[1]{\hvpic}
\newcommand{\resultrow}{\begin{minipage}[t]{\dimexpr0.5\linewidth-0.6em\relax}\hvpic\end{minipage}\hfill\begin{minipage}[t]{\dimexpr0.5\linewidth-0.6em\relax}\hvpic\end{minipage}}

\makeatletter
\newcount\hv@n
\newcount\hv@k
\def\hv@d#1#2#3#4{\global\advance\hv@k\@ne\expandafter\def\csname hv@\the\hv@k\endcsname{{#1}{#2}{#3}{#4}}}
\hv@d{13604579}{209716}{0}{1}
\hv@d{764148}{634392}{129756}{2}
\hv@d{764148}{764148}{0}{3}
\hv@d{779790}{634438}{129756}{4}
\hv@d{655360}{0}{16961335}{5}
\hv@d{27317536}{11100555}{0}{6}
\hv@d{484966}{484966}{39322}{7}
\hv@d{484966}{484966}{39322}{8}
\hv@d{484966}{484966}{39322}{9}
\hv@d{484966}{484966}{39322}{10}
\hv@d{969932}{648806}{268698}{11}
\hv@d{1018482}{648806}{268698}{12}
\hv@d{1025334}{648806}{268698}{13}
\hv@d{1040814}{648806}{268698}{14}
\hv@d{969932}{648806}{268698}{15}
\hv@d{1018482}{648806}{268698}{16}
\hv@d{1025334}{648806}{268698}{17}
\hv@d{1040814}{648806}{268698}{18}
\hv@d{24772622}{17898114}{0}{19}
\hv@d{28073160}{28352399}{0}{20}
\hv@d{969932}{648806}{268698}{21}
\hv@d{832811}{532592}{255260}{22}
\hv@d{810851}{596961}{214545}{23}
\hv@d{788179}{548758}{255260}{24}
\hv@d{803206}{561453}{214545}{25}
\hv@d{1018482}{648806}{268698}{26}
\hv@d{1025334}{648806}{268698}{27}
\hv@d{1040814}{648806}{268698}{28}
\hv@d{516686}{502330}{14356}{29}
\hv@d{24991052}{19020241}{0}{30}
\hv@d{484966}{484966}{39322}{31}
\hv@d{484966}{484966}{39322}{32}
\hv@d{969932}{648806}{268698}{33}
\hv@d{1018482}{648806}{268698}{34}
\hv@d{1025334}{648806}{268698}{35}
\hv@d{1040814}{648806}{268698}{36}
\hv@d{516686}{502330}{14356}{37}
\hv@d{969932}{648806}{268698}{38}
\hv@d{1018482}{648806}{268698}{39}
\hv@d{1025334}{648806}{268698}{40}
\hv@d{516686}{502330}{14356}{41}
\hv@d{27245332}{8567224}{0}{42}
\hv@d{19965814}{3042331}{3120232}{43}
\hv@d{969932}{648806}{268698}{44}
\hv@d{1018482}{648806}{268698}{45}
\hv@d{1025334}{648806}{268698}{46}
\hv@d{1040814}{648806}{268698}{47}
\hv@d{1021528}{648806}{268698}{48}
\hv@d{969932}{648806}{268698}{49}
\hv@d{1018482}{648806}{268698}{50}
\hv@d{1025334}{648806}{268698}{51}
\hv@d{1040814}{648806}{268698}{52}
\hv@d{1021528}{648806}{268698}{53}
\hv@d{516686}{502330}{14356}{54}
\hv@d{969932}{648806}{268698}{55}
\hv@d{844536}{541354}{255260}{56}
\hv@d{828419}{570222}{214545}{57}
\hv@d{820179}{534070}{255260}{58}
\hv@d{847615}{600233}{214545}{59}
\hv@d{1018482}{648806}{268698}{60}
\hv@d{1025334}{648806}{268698}{61}
\hv@d{874415}{556242}{255260}{0}
\hv@d{783189}{573466}{214545}{0}
\hv@d{814372}{531835}{255260}{0}
\hv@d{789458}{532817}{255260}{62}
\hv@d{813683}{595211}{214545}{63}
\hv@d{790631}{521253}{255260}{64}
\hv@d{783822}{525394}{255260}{65}
\hv@d{803677}{552078}{255260}{66}
\hv@d{801408}{569619}{214545}{67}
\hv@d{516686}{502330}{14356}{68}
\hv@d{969932}{648806}{268698}{69}
\hv@d{1018482}{648806}{268698}{70}
\hv@d{1025334}{648806}{268698}{71}
\hv@d{1040814}{648806}{268698}{72}
\hv@d{516686}{502330}{14356}{73}
\hv@d{969932}{648806}{268698}{74}
\hv@d{1018482}{648806}{268698}{75}
\hv@d{28006488}{4372313}{4372312}{76}
\hv@d{11973925}{2731629}{2731629}{77}
\hv@d{969932}{648806}{268698}{78}
\hv@d{1018482}{648806}{268698}{79}
\hv@d{856068}{567542}{255260}{80}
\hv@d{807573}{595433}{214545}{81}
\hv@d{776988}{539735}{255260}{82}
\hv@d{796185}{543114}{255260}{83}
\hv@d{837414}{523557}{255260}{84}
\hv@d{848714}{578879}{214545}{85}
\hv@d{1025334}{648806}{268698}{86}
\hv@d{831071}{544394}{255260}{87}
\hv@d{826622}{584275}{214545}{88}
\hv@d{842647}{524928}{255260}{89}
\hv@d{839904}{556182}{255260}{90}
\hv@d{834096}{555061}{255260}{91}
\hv@d{839272}{569413}{214545}{92}
\hv@d{802633}{520870}{255260}{93}
\hv@d{1040814}{648806}{268698}{94}
\hv@d{798012}{521230}{255260}{95}
\hv@d{824133}{593551}{214545}{96}
\hv@d{824461}{542703}{255260}{97}
\hv@d{883519}{594171}{214545}{98}
\hv@d{516686}{502330}{14356}{99}
\hv@d{969932}{648806}{268698}{100}
\hv@d{1018482}{648806}{268698}{101}
\hv@d{1025334}{648806}{268698}{102}
\hv@d{810125}{522368}{255260}{0}
\hv@d{813706}{571510}{214545}{0}
\hv@d{814159}{538687}{255260}{0}
\hv@d{815921}{514529}{255260}{0}
\hv@d{839075}{567493}{214545}{0}
\hv@d{874720}{579670}{255260}{0}
\hv@d{863391}{560608}{255260}{103}
\hv@d{858154}{554431}{214545}{104}
\hv@d{842810}{541886}{255260}{105}
\hv@d{826057}{542495}{255260}{106}
\hv@d{899778}{611495}{214545}{107}
\hv@d{796643}{537842}{255260}{108}
\hv@d{827184}{524553}{255260}{109}
\hv@d{846642}{569663}{214545}{110}
\hv@d{813124}{524233}{255260}{111}
\hv@d{1040814}{648806}{268698}{112}
\hv@d{796824}{533613}{255260}{113}
\hv@d{876291}{614255}{214545}{114}
\hv@d{808988}{536070}{255260}{115}
\hv@d{822432}{586680}{214545}{116}
\hv@d{1021528}{648806}{268698}{117}
\hv@d{516686}{502330}{14356}{118}
\hv@d{15525864}{4828343}{4828342}{119}
\hv@d{9374246}{2935965}{2935965}{120}
\hv@d{969932}{648806}{268698}{121}
\hv@d{1018482}{648806}{268698}{122}
\hv@d{826751}{549888}{255260}{0}
\hv@d{775629}{583089}{214545}{0}
\hv@d{868628}{536226}{255260}{0}
\hv@d{890737}{578437}{214545}{0}
\hv@d{809004}{542662}{255260}{123}
\hv@d{808413}{590711}{214545}{124}
\hv@d{854026}{554165}{255260}{125}
\hv@d{780279}{590985}{214545}{126}
\hv@d{834074}{518008}{255260}{127}
\hv@d{778002}{593794}{214545}{128}
\hv@d{876445}{556523}{255260}{129}
\hv@d{804680}{560446}{214545}{130}
\hv@d{1025334}{648806}{268698}{131}
\hv@d{516686}{502330}{14356}{132}
\hv@d{25385203}{16949675}{0}{133}
\hv@d{969932}{648806}{268698}{134}
\hv@d{838098}{563222}{255260}{0}
\hv@d{814310}{585593}{214545}{0}
\hv@d{810675}{534876}{255260}{0}
\hv@d{850241}{588559}{214545}{0}
\hv@d{811184}{550871}{255260}{135}
\hv@d{788456}{577452}{214545}{136}
\hv@d{771229}{546222}{255260}{137}
\hv@d{812268}{590962}{214545}{138}
\hv@d{818303}{519666}{255260}{139}
\hv@d{802544}{576206}{214545}{140}
\hv@d{872005}{572004}{255260}{141}
\hv@d{1018482}{648806}{268698}{142}
\hv@d{824819}{548222}{255260}{143}
\hv@d{803566}{574298}{214545}{144}
\hv@d{845915}{529725}{255260}{145}
\hv@d{799728}{546065}{255260}{146}
\hv@d{832905}{590077}{214545}{147}
\hv@d{811784}{537474}{255260}{148}
\hv@d{1025334}{648806}{268698}{149}
\hv@d{848834}{527423}{255260}{150}
\hv@d{862399}{586952}{214545}{151}
\hv@d{882382}{584165}{255260}{152}
\hv@d{815951}{598036}{214545}{153}
\hv@d{516686}{502330}{14356}{154}
\hv@d{808950}{548429}{255260}{155}
\hv@d{838564}{559829}{214545}{156}
\hv@d{812719}{558186}{255260}{157}
\hv@d{801579}{592310}{214545}{158}
\hv@d{516686}{502330}{14356}{159}
\hv@d{969932}{648806}{268698}{160}
\hv@d{1018482}{648806}{268698}{161}
\hv@d{11802670}{3521550}{301294}{162}
\hv@d{11802670}{3521550}{3239136}{163}
\hv@d{10166502}{3055207}{3055207}{164}
\hv@d{14846495}{2524100}{2524101}{165}
\hv@d{969932}{648806}{268698}{166}
\hv@d{1018482}{648806}{268698}{167}
\hv@d{780610}{557885}{255260}{168}
\hv@d{856887}{599834}{214545}{169}
\hv@d{896255}{554438}{255260}{170}
\hv@d{819771}{514279}{255260}{171}
\hv@d{818986}{587717}{214545}{172}
\hv@d{858883}{559104}{255260}{173}
\hv@d{1025334}{648806}{268698}{174}
\hv@d{795219}{529646}{255260}{0}
\hv@d{886848}{620519}{214545}{0}
\hv@d{803126}{528333}{255260}{0}
\hv@d{890386}{612589}{214545}{0}
\hv@d{816125}{543054}{255260}{0}
\hv@d{829551}{557581}{255260}{175}
\hv@d{904057}{604044}{214545}{176}
\hv@d{831081}{535756}{255260}{177}
\hv@d{805721}{588983}{214545}{178}
\hv@d{851419}{556732}{255260}{179}
\hv@d{880092}{567807}{255260}{180}
\hv@d{837978}{584702}{214545}{181}
\hv@d{800349}{528961}{255260}{182}
\hv@d{813992}{581177}{214545}{183}
\hv@d{516686}{502330}{14356}{184}
\hv@d{969932}{648806}{268698}{185}
\hv@d{1018482}{648806}{268698}{186}
\hv@d{1025334}{648806}{268698}{187}
\hv@d{516686}{502330}{14356}{188}
\hv@d{969932}{648806}{268698}{189}
\hv@d{1018482}{648806}{268698}{190}
\hv@d{1025334}{648806}{268698}{191}
\hv@d{1040814}{648806}{268698}{192}
\hv@d{516686}{502330}{14356}{193}
\hv@d{969932}{648806}{268698}{194}
\hv@d{1018482}{648806}{268698}{195}
\hv@d{1025334}{648806}{268698}{196}
\hv@d{1040814}{648806}{268698}{197}
\hv@d{969932}{648806}{268698}{198}
\hv@d{1018482}{648806}{268698}{199}
\hv@d{1025334}{648806}{268698}{200}
\hv@d{788204}{532825}{255260}{201}
\hv@d{800066}{578812}{214545}{202}
\hv@d{813022}{558920}{255260}{203}
\hv@d{798842}{573145}{214545}{204}
\hv@d{1040814}{648806}{268698}{205}
\hv@d{1021528}{648806}{268698}{206}
\hv@d{516686}{502330}{14356}{207}
\hv@d{27754090}{16948495}{0}{208}
\hv@d{969932}{648806}{268698}{209}
\hv@d{1018482}{648806}{268698}{210}
\hv@d{1025334}{648806}{268698}{211}
\hv@d{1040814}{648806}{268698}{212}
\hv@d{27656834}{17756307}{0}{213}
\hv@d{969932}{648806}{268698}{214}
\hv@d{1018482}{648806}{268698}{215}
\hv@d{1025334}{648806}{268698}{216}
\hv@d{1040814}{648806}{268698}{217}
\hv@d{516686}{502330}{14356}{218}
\hv@d{969932}{648806}{268698}{219}
\hv@d{1018482}{648806}{268698}{220}
\hv@d{1025334}{648806}{268698}{221}
\hv@d{516686}{502330}{14356}{222}
\hv@d{969932}{648806}{268698}{223}
\hv@d{1018482}{648806}{268698}{224}
\hv@d{1025334}{648806}{268698}{225}
\hv@d{969932}{648806}{268698}{226}
\hv@d{1018482}{648806}{268698}{227}
\hv@d{1025334}{648806}{268698}{228}
\hv@d{1040814}{648806}{268698}{229}
\hv@d{516686}{502330}{14356}{230}
\hv@d{969932}{648806}{268698}{231}
\hv@d{1018482}{648806}{268698}{232}
\hv@d{1025334}{648806}{268698}{233}
\hv@d{1040814}{648806}{268698}{234}
\hv@d{27511386}{15189151}{0}{235}
\hv@d{28062660}{13039452}{0}{236}
\hv@d{969932}{648806}{268698}{237}
\hv@d{1018482}{648806}{268698}{238}
\hv@d{1025334}{648806}{268698}{239}
\hv@d{22269332}{2880727}{2481242}{240}
\hv@d{1040814}{648806}{268698}{241}
\hv@d{1021528}{648806}{268698}{242}
\hv@d{1027110}{648806}{268698}{243}
\hv@d{994120}{648806}{268698}{244}
\hv@d{1027364}{648806}{268698}{245}
\hv@d{1027110}{648806}{268698}{246}
\hv@d{516686}{502330}{14356}{247}
\hv@d{969932}{648806}{268698}{248}
\hv@d{1018482}{648806}{268698}{249}
\hv@d{1025334}{648806}{268698}{250}
\hv@d{3174144}{1505066}{1505066}{251}
\hv@d{6779830}{2411882}{2411882}{252}
\hv@d{5597852}{3060088}{3060089}{253}
\hv@d{5930853}{3060088}{3060089}{254}
\hv@d{484966}{484966}{39322}{255}
\hv@d{969932}{648806}{268698}{256}
\hv@d{1018482}{648806}{268698}{257}
\hv@d{1025334}{648806}{268698}{258}
\hv@d{21201495}{21368044}{0}{259}
\hv@d{969932}{648806}{268698}{260}
\hv@d{1018482}{648806}{268698}{261}
\hv@d{1025334}{648806}{268698}{262}
\hv@d{1040814}{648806}{268698}{263}
\hv@d{1021528}{648806}{268698}{264}
\hv@d{1027110}{648806}{268698}{265}
\hv@d{807998}{553225}{255260}{0}
\hv@d{833478}{558892}{214545}{0}
\hv@d{827017}{540235}{255260}{0}
\hv@d{854965}{541695}{255260}{0}
\hv@d{856244}{558803}{255260}{266}
\hv@d{846638}{561690}{214545}{267}
\hv@d{795702}{540750}{255260}{268}
\hv@d{789928}{545757}{255260}{269}
\hv@d{807876}{554171}{255260}{270}
\hv@d{831736}{599222}{214545}{271}
\hv@d{801263}{548992}{255260}{272}
\hv@d{516686}{502330}{14356}{273}
\hv@d{516686}{502330}{14356}{274}
\hv@d{11838732}{597274}{9974874}{275}
\hv@d{969932}{648806}{268698}{276}
\hv@d{1018482}{648806}{268698}{277}
\hv@d{1025334}{648806}{268698}{278}
\hv@d{1040814}{648806}{268698}{279}
\hv@d{516686}{502330}{14356}{280}
\hv@d{775162}{558541}{255260}{281}
\hv@d{809528}{540902}{255260}{282}
\hv@d{516686}{502330}{14356}{283}
\hv@d{27123355}{13966396}{0}{284}
\hv@d{969932}{648806}{268698}{285}
\hv@d{1018482}{648806}{268698}{286}
\hv@d{1025334}{648806}{268698}{287}
\hv@d{1040814}{648806}{268698}{288}
\hv@d{1021528}{648806}{268698}{289}
\hv@d{8503024}{0}{0}{290}
\hv@d{969932}{648806}{268698}{291}
\hv@d{796604}{557468}{255260}{292}
\hv@d{798136}{552459}{255260}{293}
\hv@d{1018482}{648806}{268698}{294}
\hv@d{1025334}{648806}{268698}{295}
\hv@d{1040814}{648806}{268698}{296}
\hv@d{1021528}{648806}{268698}{297}
\hv@d{1027110}{648806}{268698}{298}
\hv@d{994120}{648806}{268698}{299}
\hv@d{1027364}{648806}{268698}{300}
\hv@d{1027110}{648806}{268698}{301}
\hv@d{1257272}{648806}{268698}{302}
\hv@d{819778}{529808}{255260}{0}
\hv@d{778307}{565121}{214545}{0}
\hv@d{821556}{538779}{255260}{0}
\hv@d{829477}{513277}{255260}{303}
\hv@d{897501}{597620}{214545}{304}
\hv@d{807536}{550364}{255260}{305}
\hv@d{882182}{547508}{255260}{306}
\hv@d{790404}{597698}{214545}{307}
\hv@d{812981}{522719}{255260}{308}
\hv@d{516686}{502330}{14356}{309}
\hv@d{1747579}{648806}{268698}{310}
\hv@d{1851621}{648806}{268698}{311}
\hv@d{1858473}{648806}{268698}{312}
\hv@d{1873953}{648806}{268698}{313}
\hv@d{1854667}{648806}{268698}{314}
\hv@d{1860249}{648806}{268698}{315}
\hv@d{1827260}{648806}{268698}{316}
\hv@d{1860503}{648806}{268698}{317}
\hv@d{4489786}{631731}{237697}{318}
\hv@d{4516574}{631731}{237697}{319}
\hv@d{4560590}{631731}{237697}{320}
\hv@d{4566870}{631731}{237697}{321}
\hv@d{4552952}{631731}{237697}{322}
\hv@d{4499262}{631731}{237697}{323}
\hv@d{4581178}{631731}{237697}{324}
\hv@d{4502086}{631731}{237697}{325}
\hv@d{4530840}{631731}{237697}{326}
\hv@d{4576934}{631731}{237697}{327}
\hv@d{4557884}{631731}{237697}{328}
\hv@d{4519508}{631731}{237697}{329}
\hv@d{4519868}{631731}{237697}{330}
\hv@d{4538382}{631731}{237697}{331}
\hv@d{4555244}{631731}{237697}{332}
\hv@d{4494530}{631731}{237697}{333}
\hv@d{3189652}{631731}{237697}{334}
\hv@d{3207804}{631731}{237697}{335}
\hv@d{3145290}{631731}{237697}{336}
\hv@d{3213024}{631731}{237697}{337}
\hv@d{3217880}{631731}{237697}{338}
\hv@d{3222822}{631731}{237697}{339}
\hv@d{3202938}{631731}{237697}{340}
\hv@d{3140726}{631731}{237697}{341}
\hv@d{3221404}{631731}{237697}{342}
\hv@d{3160854}{631731}{237697}{343}
\hv@d{3427434}{631731}{237697}{344}
\hv@d{3430826}{631731}{237697}{345}
\hv@d{3431162}{631731}{237697}{346}
\hv@d{3391834}{631731}{237697}{347}
\hv@d{3403090}{631731}{237697}{348}
\hv@d{3440340}{631731}{237697}{349}
\hv@d{3392682}{631731}{237697}{350}
\hv@d{3456602}{631731}{237697}{351}
\hv@d{3460802}{631731}{237697}{352}
\hv@d{13101533}{9261069}{0}{353}
\hv@d{12751131}{9261169}{0}{354}
\hv@d{1860249}{648806}{268698}{355}
\hv@d{2090412}{648806}{268698}{356}
\hv@d{7767178}{631731}{237697}{357}
\hv@d{7772498}{631731}{237697}{358}
\hv@d{7769524}{631731}{237697}{359}
\hv@d{7749350}{631731}{237697}{360}
\hv@d{1967337}{648806}{268698}{361}
\hv@d{4511690}{631731}{237697}{362}
\hv@d{4486978}{631731}{237697}{363}
\hv@d{4539094}{631731}{237697}{364}
\hv@d{4543082}{631731}{237697}{365}
\hv@d{4572218}{631731}{237697}{366}
\hv@d{4589120}{631731}{237697}{367}
\hv@d{4494182}{631731}{237697}{368}
\hv@d{4515092}{631731}{237697}{369}
\hv@d{4546072}{631731}{237697}{370}
\hv@d{12876243}{10636557}{0}{371}
\hv@d{12876243}{10636249}{0}{372}
\hv@d{1735398}{648806}{268698}{373}
\hv@d{1839441}{648806}{268698}{374}
\hv@d{1846292}{648806}{268698}{375}
\hv@d{1861772}{648806}{268698}{376}
\hv@d{5038722}{631731}{237697}{377}
\hv@d{5060586}{631731}{237697}{378}
\hv@d{5054236}{631731}{237697}{379}
\hv@d{5005314}{631731}{237697}{380}
\hv@d{4986948}{631731}{237697}{381}
\hv@d{13521175}{9106128}{0}{382}
\hv@d{12876265}{9106128}{0}{383}
\hv@d{2071380}{648806}{268698}{384}
\protected\def\hvpic{\global\advance\hv@n\@ne\expandafter\expandafter\expandafter\hv@put\csname hv@\the\hv@n\endcsname}
\def\hv@put#1#2#3#4{\leavevmode\setbox\z@\hbox{\csname hv@a\the\hv@n\endcsname\ifnum#4>\z@\kern-60pt\lower\dimexpr#3sp+60pt\relax\hbox{\XeTeXpdffile "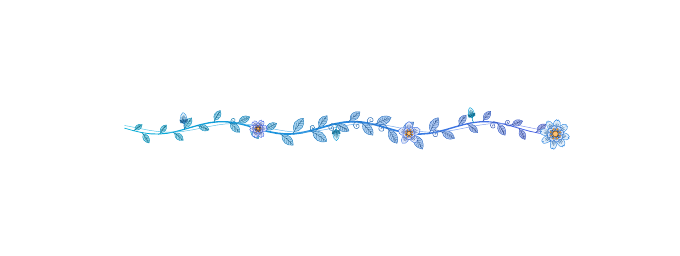" page #4\relax}\fi}\wd\z@#1sp\ht\z@#2sp\dp\z@#3sp\box\z@}
\def\hv@ann#1#2#3#4#5{\rlap{\kern#1bp\raise#2bp\hbox{}}}
\def\hv@uri#1#2#3#4#5{\rlap{\kern#1bp\raise#2bp\hbox{}}}
\expandafter\def\csname hv@a20\endcsname{\hv@ann{90.783}{418.834}{11.236}{11.631}{Item.5}\hv@ann{86.831}{369.227}{13.160}{11.631}{Item.6}\hv@ann{86.606}{278.518}{13.287}{11.630}{Item.7}\hv@ann{73.830}{187.808}{13.573}{11.630}{Item.8}\hv@ann{75.862}{82.924}{6.045}{11.631}{section.5}}
\expandafter\def\csname hv@a2\endcsname{\hv@uri{-0.000}{-1.975}{11.616}{11.617}{https://orcid.org/\@orcid}}
\expandafter\def\csname hv@a3\endcsname{\hv@uri{3.019}{3.016}{5.579}{5.579}{mailto:\@email}}
\expandafter\def\csname hv@a4\endcsname{\hv@uri{0.000}{-1.976}{11.854}{11.855}{\@repository}}
\DeclareRobustCommand{\colorkey}[1]{\hvpic}
\makeatother
\begin{document}
\maketitle

\begin{abstract}
When a learned representation is reused for a new task, which part of its structure should survive? We model a task as a small category, a change of task as a functor \(J\), and the invariants of a trained representation as a functor \(F\). The left Kan extension \(\Lan_JF\) is then the invariant that the source and the change of task impose on the target, since by its universal property every comparison of the source with an observed target invariant factors through it. We define the transfer discrepancy as the largest distance, over all target objects, between this extension and the observed invariant. For persistent homology the extension is a cokernel over comma categories. A functorial version of the discrepancy, which also compares the maps between target objects, is \(1\)-Lipschitz under interleavings, and on finite samples the discrepancy is a bottleneck distance that can be computed without approximation. When data are glued from two pieces whose exclusive parts lie farther apart than the freezing scale, a Mayer--Vietoris sequence describes exactly what is lost when homology is extended, whereas nothing is lost when chains are extended. For a merge of labelled point clouds, the degree-zero discrepancy vanishes if and only if some \(1\)-Lipschitz classifier has a margin above the freezing scale at every sample point. In experiments on sampled manifolds, trained encoders and pretrained image models, every exact prediction holds on every sample. The discrepancy of truncated invariants also recovers a hidden merge of separated parts far above chance and distinguishes the intended target of a merge from a permuted control. Across pretrained models, however, it does not track the accuracy of a linear probe.
\end{abstract}

\section{Introduction}\label{sec:introduction}
A representation learned on one task is routinely reused on another. Fine labels are merged into coarse ones, data domains are pooled, and overlapping parts of a data set, such as ranges of viewpoints, are glued into a whole. Whether the representation still serves the new task is judged by its accuracy on that task. Which of its structure has carried over is not measured at all. Classical domain-adaptation theory bounds the target risk by the source risk, a divergence between the two input distributions and a joint labelling term \citep{bendavid2010theory}. These quantities control how much accuracy can be lost, but they do not specify which structure of the representation has to survive the change of task.

We approach the question through mathematical structure, which has proved a reliable guide in the design of learning systems. Many of the architectures that work best are derived from structure before any training. For example, convolution is necessary, not merely sufficient, for equivariance to the action of a compact group \citep{kondor2018generalization}, invariant networks are characterised completely by their probabilistic symmetries \citep{bloemreddy2020probabilistic}, and a chosen symmetry group fixes the layers of a network \citep{bronstein2021geometric}. Structure also constrains procedures. If maps of metric spaces are required to induce maps of clusters, so that points of one cluster are sent into one cluster, then single linkage is the only hierarchical clustering method that also satisfies two normalisations \citep[Theorem~18]{carlsson2010characterization}. Gradient-based learning composes as a functor \citep{fong2019backprop}, and programmatic papers have drawn the same conclusion for geometry \citep{bronstein2021geometric} and for category theory \citep{gavranovic2024categorical}.

Learning algorithms often satisfy such structure only approximately \citep[Section~4.1.3]{shiebler2021category}, so the structure has to be measured. Equivariance is measured by fitted linear maps \citep[Section~2]{lenc2015understanding}, compositionality by the distance to a homomorphism \citep{andreas2019measuring} and disentanglement by equivariance to a prescribed decomposition of a group \citep{higgins2018towards}. In each case a representation is judged against the structure that a prescribed transformation forces. A change of task is such a transformation, but we know of no criterion of this kind for it.

Two families of tools compare representations across tasks, but neither of them uses a map between the tasks. Transferability estimators take the features or predictions of a pretrained model together with the labels of a target data set and predict the accuracy after fine-tuning, in order to rank source models \citep{tran2019transferability,bao2019information,nguyen2020leep,you2021logme}. No single estimator performs best across all setups \citep{agostinelli2022stable}. Where the labels of both tasks enter, as in NCE \citep{tran2019transferability}, they enter as a joint distribution on shared inputs and not as a map between the parts of the tasks. Similarity measures compare two observed representations \citep{kornblith2019similarity,barannikov2022representation}, and they disagree with each other \citep{ding2021grounding}.

Topological methods have also been applied to transfer. The homology of decision boundaries selects among pretrained models \citep{ramamurthy2019decision}, and persistence regularises domain adaptation \citep{weeks2021domain}. More generally, persistent homology measures the complexity of networks and the topology of their activations \citep{rieck2019neural,gebhart2019activation,naitzat2020topology,melodia2021dimension} and enters training through regularisers and layers \citep{chen2019topological,moor2020topological,gabrielsson2020topology}. None of these methods, however, encodes a structural map between tasks or computes the invariant that such a map forces on the target.

\begin{figure}[!htbp]
\centering
\hvpic{}\space\space
\caption{The transfer criterion. A change of task \(J\) maps the source task \(\A\) to the target task \(\B\), and the source invariant \(F\) is extended along \(J\) to \(L=\Lan_JF\) with unit \(\eta:F\Rightarrow LJ\). Every comparison \(\xi:F\Rightarrow GJ\) with an observed invariant \(G\) factors as \(\xi=(\bar\xi J)\circ\eta\) through a unique \(\bar\xi:L\Rightarrow G\), the value \(L(b)\) is the colimit \eqref{eq:pointwise} over the comma category, and the discrepancy \(\Comp_J(F,G)\) of \eqref{eq:discrepancy} compares \(L\) with \(G\) at every object of \(\B\). In all figures \colorkey{srcdark}~marks source data, \colorkey{extdark}~the extension, \colorkey{tgtdark}~the observed target and \colorkey{cmpdark}~the comparison.}
\label{fig:overview}
\end{figure}

In our approach the change of task is part of the input (\Cref{fig:overview}). We model a task as a small category whose objects are its parts, such as data domains, classes or patches of a data set, and whose morphisms record which parts are compared, for instance the inclusion of an overlap into a patch. A change of task is a functor \(J:\A\to\B\). It acts on the parts of the data and on their comparisons, not on the inputs. A trained representation together with an invariant, here the persistent homology of latent point clouds, gives a functor \(F:\A\to\V\) on the source task and an observed invariant \(G:\B\to\V\) on the target task. Both are computed from the same representation without retraining.

As the reference for the target we take the left Kan extension \(\Lan_JF:\B\to\V\) with unit \(\eta:F\Rightarrow(\Lan_JF)J\). The choice is justified by its defining property. A comparison \(\xi:F\Rightarrow GJ\) records how the source invariants sit in the target, and for every such comparison there is a unique \(\bar\xi:\Lan_JF\Rightarrow G\) with \((\bar\xi J)\circ\eta=\xi\). For two arcs of a circle (\Cref{ex:arcs}), for example, a comparison consists of linear maps from the persistent homology of each arc and of their overlap to the persistent homology of the whole circle such that both ways from the overlap agree, and the inclusions induce such a comparison. Every comparison of the source with an observation thus passes through \(\Lan_JF\), the invariant determined by the source and the change of task. In this sense \(G\cong\Lan_JF\) is the counterpart of the equivariance condition \(\varphi(gx)\approx M_g\varphi(x)\) of \citet[Section~2]{lenc2015understanding}.

A deviation of \(G\) from \(\Lan_JF\) is structure that transport along \(J\) does not explain, namely classes that \(G\) adds to the extension or identifies in it, and we count both alike. Whether such a deviation is a defect of the representation or structure wanted on the target depends on the application. The discrepancy locates the deviation, whereas a comparison of \(G\) with \(F\) itself would mix it with the effect of the change of task. The same principle underlies the measurements of equivariance, compositionality and disentanglement. Concurrent work also uses it for transport along a change of schema \citep[p.~7]{wang2026selfrevising}. The reference is canonical relative to the chosen \(\A\), \(\B\), \(J\) and \(F\), just as disentanglement is relative to a chosen decomposition of a group \citep[p.~7]{higgins2018towards}. In particular, the functor \(J\) has to be prescribed by the task, as when fine labels are assigned to coarse ones, and must not be fitted to the data. If a partition is read off the representation, its collars at the freezing scale make the chain-level comparison vacuous (\Cref{cor:refinement-gluing}).

Concretely, the value of \(\Lan_JF\) at \(b\) collects the source parts that land in \(b\) and glues their invariants along the identifications imposed by the source morphisms. Consider a circle sampled as two overlapping arcs whose exclusive parts lie more than a scale \(T\) apart. Merging the arcs as two independent classes predicts two connected components at scale \(T\), gluing them along their overlap predicts one, and the sample has one. In degree zero the discrepancy is \(\infty\) for the first change of task and \(0\) for the second, although both values are computed from the same invariants of the same arcs.

Our contributions are the following:
\begin{enumerate}
\item \textbf{Modelling.} We model changes of task as functors between small task categories and the invariant of a trained representation as a functor. The transfer discrepancy compares an observation with \(\Lan_JF\) (\Cref{def:discrepancy}). Its vanishing is strictly weaker than a natural isomorphism, and for a separated distance it is sensitive to the morphisms of the source and not only to its objects (\Cref{prop:pw-not-nat,prop:structure}, \Cref{thm:main}).
\item \textbf{A metric discrepancy with exact evaluation.} We show that the extension is a cokernel over comma categories (\Cref{prop:cokernel}), which is the standard presentation of colimits in additive categories. The diagram-level discrepancy is \(1\)-Lipschitz under interleavings (\Cref{thm:lan-interleaving-stability}), in line with the preservation of weak interleavings by left Kan extensions \citep[Lemma~4.15]{botnan2020relative}. On modules of finite type the discrepancy is a bottleneck distance, and by the isometry theorem \citep{bubenik2014categorification,chazal2009proximity,lesnick2015theory} it can be computed without approximation (\Cref{prop:correctness}). Truncation commutes with the extension, so for extensions of homology the truncated score used in the experiments is the discrepancy of truncated invariants, and the bound applies to it (\Cref{lem:truncation}).
\item \textbf{The extension as a Mayer--Vietoris cokernel with exact defect.} Along the \v{C}ech category of a two-set \(T\)-Lebesgue cover, the extension of persistent homology is the cokernel term of the Mayer--Vietoris sequence \citep{difabio2011mayer,lipsky2011spectral,yoon2020persistence}. The kernel term measures what is lost when homology instead of chains is extended. The distance is positive iff the kernel term is nonzero, and it is infinite iff the kernel term is alive at the freezing scale (\Cref{thm:mayer-vietoris}). Nothing is lost when chains are extended (\Cref{prop:chain-gluing}). In degree zero a sharp criterion only requires that the two sets cover the data, not the \(T\)-Lebesgue condition (\Cref{prop:cover-gluing}), and we have not found it stated elsewhere. If a partition is read off the representation, its collars of radius \(T\) make the chain-level comparison vacuous, so the cover has to come from the task (\Cref{cor:refinement-gluing}).
\item \textbf{A bridge to classification margins} (\Cref{thm:bridge}). For a merge of labelled clouds the degree-zero discrepancy is \(0\) or \(\infty\), and it vanishes iff some \(1\)-Lipschitz score head separates the sample with a margin above the freezing scale.
\end{enumerate}
In \Cref{sec:experiments} we check the statements about finite samples on every row of sampled manifolds, trained encoders and pretrained image models. Since these statements are theorems, only the implementation can fail them. We also test the predictions that no theorem implies against negative controls and baselines. The discrepancy of truncated invariants recovers a hidden merge of separated parts far above chance and distinguishes intended targets from permuted controls. Across pretrained models, however, it does not track the accuracy~of~a~linear~probe.\keeptail

In machine learning, Kan extensions have so far appeared mostly as models. \citet{shiebler2022kan} derives classification and clustering as Kan extensions along inclusion functors, \citet{pugh2025learning} present error minimisation as a left Kan extension, and \citet{mahadevan2026ket} reads attention as an extension operator. In these works the extended functor is the predictor. Kan extensions have also been used in persistence to interpolate modules \citep{bubenik2017higher}. \citet{arya2025kan} extend persistent homology transforms from finitely many directions to the sphere and compare the result with the true transform. Their comparison also checks an invariant against ground truth, but along a refinement of the sample rather than a change of task, and without a learned representation. Closest to our comparison is the concurrent work of \citet{wang2026selfrevising}, who compare the state of a discovery system after a change of schema with the left Kan transport of its earlier state. They work with set-valued functors and without a metric. Other concurrent preprints use Kan extensions as canonical completions against which the answers of language models are tested \citep[Section~1]{shi2026canonical}, and as the propagation of observed associations in compositional generalisation \citep[Section~2]{maeda2026compositional}. Compared with these works, we add linear invariants, a metric with stability and exact evaluation, and the gluing theorem.

For gluing, \citet{difabio2011mayer} give a Mayer--Vietoris formula for persistent Betti numbers of sublevel sets, and \citet[Lemma~3.1]{yoon2020persistence} split the homology of a Rips complex over a cover with one-dimensional nerve into cosheaf homology at small scales. \Cref{thm:mayer-vietoris} is the persistence-module version of their lemma for two sets, with the range of scales fixed by the \(T\)-Lebesgue condition. In general, spectral sequences distribute the computation \citep{lipsky2011spectral}.

Our tools are classical. We use the pointwise formula for left Kan extensions, the cokernel presentation of colimits, the preservation of weak interleavings by left Kan extensions, the isometry theorem and the Mayer--Vietoris sequence of persistence modules. What is new is their combination into a transfer criterion for linear invariants with a stable metric and exact evaluation, the description of the zero set of the transfer discrepancy, and the identification of the extension along a \v{C}ech category with the Mayer--Vietoris cokernel, together with its distance statements. The sharp criterion in degree zero is also new, as is the translation of the degree-zero discrepancy of a merge into classification margins, which is elementary once the extension has been computed.

\section{The Categorical Transfer Framework}\label{sec:framework}
\textbf{Conventions.} All categories are locally small. We write \(\Fun{\A}{\V}\) for the category of functors \(\A\to\V\) and natural transformations, \(\mathbf 1\) for the terminal category with \(\Ob(\mathbf 1)=\{\bullet\}\) and \(\Mor(\mathbf 1)=\{\operatorname{id}_\bullet\}\), and \(\sqcup\) for a coproduct, written \(\oplus\) in linear categories and as a disjoint union in sets. In weighted sums \(0\cdot\infty\coloneqq0\). Let \(\K\) be a field, \(\Vect\) the category of \(\K\)-vector spaces and linear maps, \(\Ch\) the category of \(\mathbb Z\)-graded chain complexes of \(\K\)-vector spaces and chain maps, and \(\Pers\coloneqq\Fun{(\mathbb R,\leq)}{\Vect}\) the category of \textbf{persistence modules}. A persistence module \(M\) consists of vector spaces \(M_t\) for \(t\in\mathbb R\) and linear \textbf{structure maps} \(M_{s\leq t}:M_s\to M_t\) for \(s\leq t\) with \(M_{t\leq t}=\operatorname{id}\) and \(M_{t\leq u}\circ M_{s\leq t}=M_{s\leq u}\), and a morphism \(f:M\to N\) consists of linear maps \(f_t:M_t\to N_t\) with \(N_{s\leq t}\circ f_s=f_t\circ M_{s\leq t}\). For instance, the degree-zero persistent homology of a point cloud (\Cref{def:frozen-vr}) has at scale \(t\) one basis vector per connected component of its complex at \(t\). Its structure maps send each component to the component that contains it later, so that two components that have merged go to the same vector. A module of finite type is determined up to isomorphism by its barcode, a finite multiset of intervals, and for such modules the interleaving distance of \Cref{def:interleaving} is the bottleneck distance of the barcodes (\Cref{sec:score}). A colimit glues the values of a diagram along its morphisms. For a diagram of vector spaces, for example, it is the direct sum of the values modulo the identification of every \(x\) with its images (\Cref{lem:finite-colimit}).

\textbf{Tasks.} A \textbf{task} is a small category \(\A\) whose objects are its parts, such as a data domain, a class label or a patch of data, and whose morphisms record the admitted comparisons. A \textbf{change of task} is a functor \(J:\A\to\B\) between tasks. For \(b\in\Ob(\B)\) the \textbf{comma category} \(J\downarrow b\) consists of the source parts that land in \(b\), each together with the way it lands. Its objects are the pairs \((a,\beta)\) with \(a\in\Ob(\A)\) and \(\beta:Ja\to b\) in \(\B\), its morphisms \((a,\beta)\to(a',\beta')\) are the morphisms \(\bar\alpha:a\to a'\) of \(\A\) with \(\beta'\circ J\bar\alpha=\beta\), composed as in \(\A\), and its \textbf{projection} is \(\pi_b:J\downarrow b\to\A\), \((a,\beta)\mapsto a\), \(\bar\alpha\mapsto\bar\alpha\). In every change of task of \Cref{sec:experiments} a source part lands in \(b\) in at most one way, and \Cref{prop:pw-not-nat} gives an example of a part that lands in two ways.\tighter

\textbf{The Extension.} A \textbf{left Kan extension} of a functor \(F:\A\to\V\) along \(J\) is a pair \((L,\eta)\) of a functor \(L:\B\to\V\) and a natural transformation \(\eta:F\Rightarrow LJ\), the \textbf{unit}, such that for every \(E:\B\to\V\) and every \(\xi:F\Rightarrow EJ\) there is a unique \(\bar\xi:L\Rightarrow E\) with \((\bar\xi J)\circ\eta=\xi\). Two left Kan extensions \((L,\eta)\) and \((L',\eta')\) of \(F\) along \(J\) are isomorphic in \(\Fun{\B}{\V}\): the universal properties give \(\sigma:L\Rightarrow L'\) with \((\sigma J)\circ\eta=\eta'\) and \(\sigma':L'\Rightarrow L\) with \((\sigma'J)\circ\eta'=\eta\), so \(((\sigma'\circ\sigma)J)\circ\eta=\eta=(\operatorname{id}_LJ)\circ\eta\). The uniqueness clause for \(\xi=\eta\) gives \(\sigma'\circ\sigma=\operatorname{id}_L\), and likewise \(\sigma\circ\sigma'=\operatorname{id}_{L'}\). Whenever \(F\) has a left Kan extension along \(J\) we fix one, written \((\Lan_JF,\eta^F)\). For a natural transformation \(\theta:F\Rightarrow F'\) between two such functors, \(\Lan_J\theta:\Lan_JF\Rightarrow\Lan_JF'\) is the unique natural transformation with \((\Lan_J\theta\,J)\circ\eta^F=\eta^{F'}\circ\theta\), given by the universal property of \((\Lan_JF,\eta^F)\) for \(E=\Lan_JF'\) and \(\xi=\eta^{F'}\circ\theta\). Every statement below depends on \(\Lan_JF\) only through its isomorphism class in \(\Fun{\B}{\V}\).

\begin{example}\label{ex:arcs}
Let \(P\) be a finite sample of a circle covered by two arcs \(U_1,U_2\subseteq P\) that overlap in two short arcs, \(U_{12}\coloneqq U_1\cap U_2\) (\Cref{fig:gluing}). We consider two changes of task that start from the same arcs. The \textbf{merge} treats the arcs as two independent classes and ignores the points they share. Its source task is discrete on \(u_1,u_2\) and its target task is \(\mathbf 1\). The \textbf{gluing} also keeps the overlap. Its source task is the span \(u_1\leftarrow u_{12}\rightarrow u_2\) of the two inclusions, and its target task is again \(\mathbf 1\). Since the target is the same, the discrepancy compares two descriptions of the same arcs. In both cases \(F\) sends each part to the persistent homology of its points, and the observed invariant is that of \(P\). Where a scale \(T\geq0\) enters, we assume that the arcs form a \(T\)-Lebesgue cover (\Cref{def:lebesgue-cover}), that \(\VR_T(U_1)\) and \(\VR_T(U_2)\) have no homology in degree one and that \(\VR_T(P)\) has homology \(\K\) in degree one.
\end{example}

The universal property characterises the extension but does not compute it. A colimit over each comma category computes it, and for three elementary changes of task this colimit has an explicit form (\Cref{fig:pointwise,fig:task-changes}).

\begin{proposition}\label{prop:basic-task-changes}
Let \(\V\) be a category and \(J:\A\to\B\) a functor between small categories.
\begin{enumerate}
\item\label{it:forget} \textbf{Pointwise Formula.} Let \(F:\A\to\V\). If \(\V\) has a colimit of \(F\pi_b\) over \(J\downarrow b\) for every \(b\in\Ob(\B)\), as it does for every \(F\) when \(\V\) is cocomplete, then \(F\) has a left Kan extension along \(J\), and for every \(b\in\Ob(\B)\) there is an isomorphism
\begin{equation}\label{eq:pointwise}
(\Lan_JF)(b)\cong\operatorname*{colim}_{J\downarrow b}F\pi_b
\end{equation}
under which the colimit cocone has the legs \(\lambda^{(a,\beta)}\coloneqq(\Lan_JF)(\beta)\circ\eta^F_a:F(a)\to(\Lan_JF)(b)\). For every left Kan extension \((L,\eta)\) of \(F\) along \(J\) one has \(L(\beta')\circ\eta_{a'}\circ F(\bar\alpha)=L(\beta)\circ\eta_a\) for every morphism \(\bar\alpha:(a,\beta)\to(a',\beta')\) of \(J\downarrow b\), in particular \(\lambda^{(a',\beta')}\circ F(\bar\alpha)=\lambda^{(a,\beta)}\).
\item\label{it:merge} \textbf{Merging Domains.} If \(\A\) is discrete, \(\B=\mathbf 1\) and the values of \(F:\A\to\V\) have a coproduct in \(\V\), then \(F\) has a left Kan extension along \(J\), and \((\Lan_JF)(\bullet)\cong\bigsqcup_{a\in\Ob(\A)}F(a)\) with the leg \(\lambda^{(a,\operatorname{id}_\bullet)}=\eta^F_a\) corresponding to the coproduct inclusion of \(F(a)\).
\item\label{it:refine} \textbf{Class Refinement.} Let \(C\) and \(D\) be sets, let \(R\subseteq C\times D\), let \(\A\) be discrete with \(\Ob(\A)=C\), and let \(\B\) have \(\Ob(\B)=C\sqcup D\) and, besides the identities, only the morphisms \(c\to d\), one for each \((c,d)\in R\). Let \(J:\A\to\B\) be the inclusion and let \(F:\A\to\V\) be such that for every \(d\in D\) the values \(F(c)\) with \((c,d)\in R\) have a coproduct in \(\V\). Then \(F\) has a left Kan extension along \(J\), with \((\Lan_JF)(c)\cong F(c)\) for \(c\in C\) and \((\Lan_JF)(d)\cong\bigsqcup_{c:(c,d)\in R}F(c)\) for \(d\in D\), so that \((\Lan_JF)(d)\cong F(c)\) when \(c\) is the only element of \(C\) with \((c,d)\in R\).
\item\label{it:collapse} \textbf{Layer Collapse.} If \([n]\) is the poset \(0\leq\cdots\leq n\) read as a category and \(J:[n]\to\mathbf 1\) is the unique functor, then every \(F:[n]\to\V\) has a left Kan extension along \(J\), and~\({(\Lan_JF)(\bullet)\cong F(n)}\).\keeptail
\end{enumerate}
\end{proposition}

\begin{figure}[!htbp]
\centering
\hvpic{}\space\space
\caption{The value of \(L=\Lan_JF\) at \(b\) (\itemref{prop:basic-task-changes}{it:forget}). \(J\downarrow b\) has three objects and one morphism \(\bar\alpha:(a_1,\beta_1)\to(a_2,\beta_2)\) besides identities, with \(\beta_2\circ J\bar\alpha=\beta_1\), so the colimit, whose legs satisfy \(\lambda^{(a_2,\beta_2)}\circ F(\bar\alpha)=\lambda^{(a_1,\beta_1)}\), keeps one copy of the bar that \(F(\bar\alpha)\) carries into \(F(a_2)\). Against the observed \(G(b)\) the largest difference is \(0.04\), at the marked pair.}
\label{fig:pointwise}
\end{figure}

\begin{figure}[!htbp]
\centering
\hvpic{}\space\space
\caption{The shapes of \(J\downarrow b\) behind \Cref{prop:basic-task-changes}: discrete for a merge, discrete per fine class for a refinement, with a ringed terminal object for a collapse, and the span \(\A_2\) of \Cref{sec:gluing}, whose colimit is a pushout. Each row shows \(J\), the comma categories and the colimits of example barcodes.}
\label{fig:task-changes}
\end{figure}

\begin{proof}
~\begin{steps}
\item \textbf{The Pointwise Formula.} Under the hypothesis of \labelcref{it:forget}, \(F\) has a left Kan extension whose value at \(b\) is the colimit of \(F\pi_b\), with the legs \(\lambda^{(a,\beta)}=(\Lan_JF)(\beta)\circ\eta^F_a\) \citetext{\citealp[Section~X.3]{maclane1998categories}, \citealp[Theorem~6.2.1]{riehl2016category}}. For a left Kan extension \((L,\eta)\) and a morphism \(\bar\alpha:(a,\beta)\to(a',\beta')\) of \(J\downarrow b\),
\[L(\beta')\circ\eta_{a'}\circ F(\bar\alpha)\eqby{a}L(\beta')\circ L(J\bar\alpha)\circ\eta_a\eqby{b}L(\beta)\circ\eta_a\dc\]
where \rref{a} is the naturality of \(\eta\) and \rref{b} the functoriality of \(L\) together with \(\beta'\circ J\bar\alpha=\beta\), and for \((L,\eta)=(\Lan_JF,\eta^F)\) this is the identity \(\lambda^{(a',\beta')}\circ F(\bar\alpha)=\lambda^{(a,\beta)}\).
Since \(\A\) and \(\B\) are small, every \(J\downarrow b\) is small, so a cocomplete \(\V\) has all these colimits. For the other items it remains to compute the comma categories and their colimits.
\item \textbf{Merging Domains.} The objects of \(J\downarrow\bullet\) are the pairs \((a,\operatorname{id}_\bullet)\), and its morphisms are those of the discrete category \(\A\), identities. A cocone on a diagram over a discrete category is a family of morphisms without any condition, so its colimit is a coproduct of the values with the inclusions as legs, and it exists iff the coproduct does. The leg formula gives \(\lambda^{(a,\operatorname{id}_\bullet)}=(\Lan_JF)(\operatorname{id}_\bullet)\circ\eta^F_a=\eta^F_a\).
\item \textbf{Class Refinement.} For \(d\in D\) the set \(\B(c,d)\) has one element if \((c,d)\in R\) and none otherwise, and \(\A\) is discrete, so \(J\downarrow d\) is discrete on the \(c\) with \((c,d)\in R\), and the previous step gives the coproduct, which exists by hypothesis. For \(c\in C\) the only morphism into \(c\) is \(\operatorname{id}_c\), so \(J\downarrow c\) has one object, and \(F(c)\) with the identity leg is the colimit. A coproduct with one summand is that summand.
\item \textbf{Layer Collapse.} The condition \(\operatorname{id}_\bullet\circ J(i\leq j)=\operatorname{id}_\bullet\) is automatic, so \((i,\operatorname{id}_\bullet)\mapsto i\) is an isomorphism \(J\downarrow\bullet\cong[n]\) under which \(\pi_\bullet\) becomes the identity. The legs \(\lambda^i\coloneqq F(i\leq n)\) form a cocone, since \(F(j\leq n)\circ F(i\leq j)=F(i\leq n)\). Every cocone \(\nu^i:F(i)\to W\) satisfies \(\nu^n\circ\lambda^i=\nu^i\) by its cocone condition, and \(\nu^n\) is the only such morphism because \(\lambda^n=\operatorname{id}_{F(n)}\).\qedhere
\end{steps}
\end{proof}

In \Cref{ex:arcs} the merge is \labelcref{it:merge}, so its extension is \(F(u_1)\oplus F(u_2)\). For the gluing, \(J\downarrow\bullet\) is isomorphic to the span itself, and the colimit is the pushout that identifies the image of every class of the overlap in \(F(u_1)\) with its image in \(F(u_2)\). When \(\V\) is cocomplete, \labelcref{it:forget} gives the extension of every source diagram, with each value a colimit that \Cref{prop:cokernel} presents by a cokernel. For a refinement the source task contains only the coarse classes. By \labelcref{it:refine} the reference of a fine class that lies in a single coarse class is then the invariant of that coarse class, and the discrepancy measures how far the fine class departs from it. A layer collapse treats the layers of a network as the parts of a task, with the maps between consecutive layers as morphisms, and its reference is the invariant of the last layer.

To compare the extension with an observation we need a distance. Colimits may leave the class of objects on which a distance is defined. For example, a direct sum of infinitely many copies of a persistence module \(M\) with \(M_0\neq0\) is infinite-dimensional at \(0\), so it is not pointwise finite-dimensional. An invariant target therefore keeps apart the category in which colimits are formed and the class on which the distance is defined.

\begin{definition}\label{def:invariant}
An \textbf{invariant target} is a triple \((\V,\mathcal S,d_{\V})\) in which \(\V\) admits all small colimits, \(\mathcal S\subseteq\Ob(\V)\) is a class of objects closed under isomorphism whose isomorphism classes form a set \(\mathcal S/{\cong}\), and \(d_{\V}\) is an \textbf{extended pseudometric} on \(\mathcal S/{\cong}\), that is a map \(d_{\V}:(\mathcal S/{\cong})\times(\mathcal S/{\cong})\to[0,\infty]\) with \(d_{\V}(x,x)=0\), \(d_{\V}(x,y)=d_{\V}(y,x)\) and \(d_{\V}(x,z)\leq d_{\V}(x,y)+d_{\V}(y,z)\) for all \(x,y,z\). The target is \textbf{separated} if \(d_{\V}([X],[Y])=0\) implies \([X]=[Y]\). We put \(\Sep(\mathcal S,d_{\V})\coloneqq(\mathcal S/{\cong})/{\sim_0}\), where \([X]\sim_0[Y]\) iff \(d_{\V}([X],[Y])=0\), an equivalence relation by the three properties of \(d_{\V}\), and for \(X,Y\in\mathcal S\) we write \(d_{\V}(X,Y)\coloneqq d_{\V}([X],[Y])\). An \textbf{invariant functor} on a small category \(\A\) with values in \(\V\) is a functor \(F:\A\to\V\).\tighter
\end{definition}

For persistence modules we use the interleaving distance, which requires maps in both directions whose composites are the structure maps up to a shift.

\begin{definition}\label{def:interleaving}
For \(\varepsilon\geq0\) let \(\Sigma_\varepsilon:\Pers\to\Pers\) be precomposition with the translation \((\mathbb R,\leq)\to(\mathbb R,\leq)\), \(t\mapsto t+\varepsilon\), so that \((\Sigma_\varepsilon M)_t=M_{t+\varepsilon}\), \((\Sigma_\varepsilon M)_{s\leq t}=M_{s+\varepsilon\leq t+\varepsilon}\) and \((\Sigma_\varepsilon f)_t=f_{t+\varepsilon}\) for a morphism \(f:M\to N\). Since \((t+\varepsilon)+\varepsilon'=t+(\varepsilon+\varepsilon')\) for every \(t\in\mathbb R\), one has \(\Sigma_\varepsilon\Sigma_{\varepsilon'}=\Sigma_{\varepsilon+\varepsilon'}\) as functors, and \(\Sigma_0\) is the identity. The \textbf{transition morphism} \(\tau^M_\varepsilon:M\to\Sigma_\varepsilon M\) has the components \(M_{t\leq t+\varepsilon}\), which are natural because \(M_{t\leq t+\varepsilon}\circ M_{s\leq t}=M_{s\leq t+\varepsilon}=M_{s+\varepsilon\leq t+\varepsilon}\circ M_{s\leq s+\varepsilon}\) for \(s\leq t\) by functoriality of \(M\). An \(\varepsilon\)-\textbf{interleaving} of \(M,N\in\Pers\) is a pair of morphisms \(\varphi:M\to\Sigma_\varepsilon N\) and \(\psi:N\to\Sigma_\varepsilon M\) with \(\Sigma_\varepsilon\psi\circ\varphi=\tau^M_{2\varepsilon}\) and \(\Sigma_\varepsilon\varphi\circ\psi=\tau^N_{2\varepsilon}\), the codomains \(\Sigma_\varepsilon\Sigma_\varepsilon M=\Sigma_{2\varepsilon}M\) and \(\Sigma_\varepsilon\Sigma_\varepsilon N=\Sigma_{2\varepsilon}N\) of the composites being read through the equality just stated, and the \textbf{interleaving distance} is
\begin{equation}\label{eq:interleaving}
d_I(M,N)\coloneqq\inf\{\varepsilon\geq0\mid M\text{ and }N\text{ are }\varepsilon\text{-interleaved}\}\in[0,\infty]\dc\qquad\inf\varnothing\coloneqq\infty\dd
\end{equation}
\end{definition}

\begin{figure}[!htbp]
\centering
\hvpic{}\space\space
\caption{A \(0.10\)-interleaving of an extension \(M\) and an observation \(N\). Top: \(\varphi\) carries a class of \(M\) at \(t\) to \(N\) at \(t+\varepsilon\), and \(\psi\) carries it back at \(t+2\varepsilon\). Bottom left: the identities of \Cref{def:interleaving}. Bottom right: the diagrams, \colorkey{extdark}~rings for \(M\), \colorkey{tgtdark}~cores for \(N\), matched at cost \(d_B=0.10=d_I(M,N)\).}
\label{fig:interleaving}
\end{figure}

The interleaving distance (\Cref{fig:interleaving}) is an extended pseudometric on persistence modules \citep[Theorem~3.3]{bubenik2014categorification}, invariant under isomorphism and hence a function on isomorphism classes, and \Cref{lem:diagram-pseudometric} contains these statements as the case of diagrams over \(\mathbf 1\). \citet{lesnick2015theory} extends the interleaving distance to multiparameter modules. We use two invariant targets, \((\Ch,\mathcal S_{\mathrm{ch}},d_{\mathrm{ch}})\) of \Cref{def:chain-target} and \((\Pers,\Pers^{\mathrm{pfd}},d_I)\), where \(\Pers^{\mathrm{pfd}}\) is the class of \textbf{pointwise finite-dimensional} modules, those \(M\in\Pers\) with \(\dim_\K M_t<\infty\) for every \(t\in\mathbb R\). In both targets the colimits are formed in the ambient category and the distance is defined on the smaller class. Once a target is fixed, the discrepancy compares the extension with the observation at every target object.

\begin{definition}\label{def:discrepancy}
Let \((\V,\mathcal S,d_{\V})\) be an invariant target, let \(J:\A\to\B\) be a functor between small categories, and let \(F:\A\to\V\) and \(G:\B\to\V\) be functors such that \(\Lan_JF\) exists and \((\Lan_JF)(b),G(b)\in\mathcal S\) for every \(b\in\Ob(\B)\). The invariant functor \(F\) is \textbf{Kan-transferable} to \(G\) along \(J\) if \(\Lan_JF\cong G\) in \(\Fun{\B}{\V}\), and the \textbf{transfer discrepancy} of \((F,G)\) along \(J\) is
\begin{equation}\label{eq:discrepancy}
\Comp_J(F,G)\coloneqq\sup_{b\in\Ob(\B)}d_{\V}\bigl([(\Lan_JF)(b)],[G(b)]\bigr)\in[0,\infty]\dc\qquad\sup\varnothing\coloneqq0\dd
\end{equation}
\end{definition}

In \Cref{ex:arcs}, with persistent homology in degree zero frozen at \(T\) (\Cref{def:frozen-vr}), the merge has discrepancy \(\infty\) and the gluing \(0\), by \itemref{prop:cover-gluing}{it:cg-coprod} and \itemref{thm:mayer-vietoris}{it:mv-iso}. We take the supremum so that a discrepancy of at most \(\varepsilon\) bounds the distance at every target object. For \(0<\varepsilon<\infty\) and two or more target objects, a mean over the objects would not give such a bound. Since the discrepancy is defined objectwise, it does not distinguish all functors. It cannot detect a natural isomorphism (\Cref{prop:pw-not-nat}), but for a separated distance it is sensitive to the morphisms of the source and not only to its objects (\Cref{prop:structure}).

\begin{proposition}\label{prop:pw-not-nat}
For every field \(\K\) there exist small categories \(\A,\B\), a functor \(J:\A\to\B\) and functors \(F:\A\to\Vect\), \(G:\B\to\Vect\) with finite-dimensional values such that \(\Comp_J(F,G)=0\) for every invariant target \((\Vect,\mathcal S,d_{\V})\) whose class \(\mathcal S\) contains the finite-dimensional spaces, whatever its distance, separated or not, while \(\Lan_JF\not\cong G\) in \(\Fun{\B}{\Vect}\).
\end{proposition}

\begin{proof}
~\begin{steps}
\item \textbf{The Data.} Let \(\A\coloneqq\mathbf 1\), and let \(\B\) have \(\Ob(\B)=\{0,1\}\) and \(\Mor(\B)=\{\operatorname{id}_0,\operatorname{id}_1,f,g\}\) with \(f,g:0\to1\). Every composite in \(\B\) involves an identity, so a functor \(\B\to\Vect\) is the same as two vector spaces with two parallel linear maps. Let \(J\bullet=0\), \(F(\bullet)=\K\), \(G(0)=\K\), \(G(1)=\K\oplus\K\) and \(G(f)=G(g):\K\to\K\oplus\K\), \(x\mapsto(x,0)\).
\item \textbf{The Extension.} We use \itemref{prop:basic-task-changes}{it:forget}. Since \(\B(0,0)=\{\operatorname{id}_0\}\), \(J\downarrow0\cong\mathbf 1\) and \((\Lan_JF)(0)\cong\K\) with the leg \(\eta^F_\bullet=\operatorname{id}_\K\). The objects of \(J\downarrow1\) are \((\bullet,f)\) and \((\bullet,g)\), and a morphism \((\bullet,\beta)\to(\bullet,\beta')\) is \(\operatorname{id}_\bullet\) subject to \(\beta'=\beta\), so \(J\downarrow1\) is discrete and its colimit is a coproduct, as in the proof of \Cref{prop:basic-task-changes}: \((\Lan_JF)(1)\cong\K\oplus\K\) with the legs \(\lambda^{(\bullet,f)}(x)=(x,0)\) and \(\lambda^{(\bullet,g)}(x)=(0,x)\). For \(\beta\in\{f,g\}\) the leg formula gives \((\Lan_JF)(\beta)=(\Lan_JF)(\beta)\circ\eta^F_\bullet=\lambda^{(\bullet,\beta)}\).
\item \textbf{The Discrepancy Vanishes.} {\addfontfeatures{LetterSpace=0.7}All four values are finite-dimensional, hence in \(\mathcal S\), and} \((\Lan_JF)(b)\cong G(b)\) for both \(b\). A function on isomorphism classes that vanishes on the diagonal gives \(\Comp_J(F,G)=\max\{d_{\V}([\K],[\K]),d_{\V}([\K^2],[\K^2])\}=0\).
\item \textbf{No Natural Isomorphism.} Let \(\sigma:\Lan_JF\Rightarrow G\) be natural and \(c\coloneqq\sigma_0(1)\). Naturality at \(f\) and at \(g\), evaluated at \(1\in\K\), gives \(\sigma_1(1,0)=G(f)(c)=(c,0)=G(g)(c)=\sigma_1(0,1)\), so \(\sigma_1\) kills \((1,-1)\neq0\), and no natural transformation \(\Lan_JF\Rightarrow G\) is an isomorphism.\qedhere
\end{steps}
\end{proof}

\begin{proposition}\label{prop:structure}
For every field \(\K\) there exist a small category \(\A\), a functor \(J:\A\to\mathbf 1\) and functors \(F,F':\A\to\Vect\) with finite-dimensional values such that \(F(a)\cong F'(a)\) for every \(a\in\Ob(\A)\) and \((\Lan_JF)(\bullet)\not\cong(\Lan_JF')(\bullet)\). For every invariant target \((\Vect,\mathcal S,d_{\V})\) whose class \(\mathcal S\) contains the finite-dimensional spaces and whose distance is separated, and for \(G\coloneqq\Lan_JF\), one has \(\Comp_J(F,G)=0<\Comp_J(F',G)\).
\end{proposition}

\begin{proof}
~\begin{steps}
\item \textbf{The Data.} Let \(\A\) have \(\Ob(\A)=\{0,1\}\) and \(\Mor(\A)=\{\operatorname{id}_0,\operatorname{id}_1,f,g\}\) with \(f,g:0\to1\), let \(J:\A\to\mathbf 1\) be the unique functor, and let \(F,F':\A\to\Vect\) have all values \(\K\), \(F(f)=F(g)=F'(f)=\operatorname{id}_\K\) and \(F'(g)=0\), so that \(F(a)=F'(a)\) for both \(a\).
\item \textbf{The Extensions.} The condition over \(\mathbf 1\) is vacuous, so \((a,\operatorname{id}_\bullet)\mapsto a\), \(\bar\alpha\mapsto\bar\alpha\) is an isomorphism \(J\downarrow\bullet\cong\A\) under which \(\pi_\bullet\) becomes the identity. For \(E\in\{F,F'\}\) a cocone \((\nu^0,\nu^1)\) is determined by \(\nu^1\) with \(\nu^1\circ E(f)=\nu^0=\nu^1\circ E(g)\), so the colimit is the coequalizer of \(E(f)\) and \(E(g)\). In \(\Vect\) the coequalizer of \(u,v:X\to Y\) is \(\coker(u-v)\) with the quotient map, since \(\omega\circ u=\omega\circ v\) iff \(\omega\circ(u-v)=0\). By \eqref{eq:pointwise},
\begin{equation}\label{eq:structure-values}
(\Lan_JF)(\bullet)\cong\coker(\operatorname{id}_\K-\operatorname{id}_\K)\cong\K,\qquad(\Lan_JF')(\bullet)\cong\coker(\operatorname{id}_\K-0)=0\dc
\end{equation}
of dimensions \(1\) and \(0\), so the two values are not isomorphic.
\item \textbf{The Discrepancies.} All values are finite-dimensional by \eqref{eq:structure-values}, so both discrepancies are defined. Since \(\Ob(\mathbf 1)=\{\bullet\}\), \(\Comp_J(F,G)=d_{\V}([(\Lan_JF)(\bullet)],[(\Lan_JF)(\bullet)])=0\) and \(\Comp_J(F',G)=d_{\V}([0],[\K])>0\), because \([0]\neq[\K]\) and \(d_{\V}\) is separated.\qedhere
\end{steps}
\end{proof}

\begin{figure}[!htbp]
\centering
\hvpic{}\space\space
\caption{\Cref{prop:pw-not-nat}. The extension sends \(f,g\) to the inclusions of the two axes of \(\K^2\), whereas the observed \(G\) sends both to the first axis. The dimensions agree, but every natural \(\sigma\) satisfies~\(\sigma_1(e_1)=\sigma_1(e_2)\).}
\label{fig:pw-not-nat}
\medskip
\hvpic{}\space\space
\caption{\Cref{prop:structure}. The functors \(F\) and \(F'\) agree on objects, and the coequalizer keeps the line when \(F(f)=F(g)\) and collapses it when \(F'(g)=0\neq F'(f)\).}
\label{fig:structure}
\end{figure}

Gluing two arcs along their overlap and merging them as independent classes is a geometric instance of \Cref{prop:structure}, since both use the same parts, once with and once without the morphisms between them. The zero set of the discrepancy can be described in general, for separated distances and for discrete target tasks.

\begin{theorem}\label{thm:main}
Let \((\V,\mathcal S,d_{\V})\) be an invariant target, let \(J:\A\to\B\) be a functor between small categories, and let \(F:\A\to\V\) be a functor such that \(\Lan_JF\) exists and \((\Lan_JF)(b)\in\mathcal S\) for every \(b\in\Ob(\B)\). For \(\varepsilon\geq0\) let \(\mathcal T_{\varepsilon}(J,F)\) be the class of functors \(G:\B\to\V\) with \(G(b)\in\mathcal S\) for all \(b\in\Ob(\B)\) and \(\Comp_J(F,G)\leq\varepsilon\). Then the following hold for every functor \(G:\B\to\V\) with \(G(b)\in\mathcal S\) for all \(b\in\Ob(\B)\):
\begin{enumerate}
\item\label{it:main-lan} \textbf{Membership.} The extension \(\Lan_JF\) lies in \(\mathcal T_0(J,F)\).
\item\label{it:main-iso} \textbf{Invariance.} The class \(\mathcal T_{\varepsilon}(J,F)\) is closed under natural isomorphism for every \(\varepsilon\geq0\).
\item\label{it:main-sep} \textbf{Zero Class.} The functor \(G\) lies in \(\mathcal T_0(J,F)\) iff \([(\Lan_JF)(b)]=[G(b)]\) holds in \(\Sep(\mathcal S,d_{\V})\) for every object \(b\in\Ob(\B)\).
\item\label{it:main-sepcase} \textbf{Separation.} For separated \(d_{\V}\), the functor \(G\) lies in \(\mathcal T_0(J,F)\) iff \((\Lan_JF)(b)\cong G(b)\) holds in \(\V\) for every object \(b\in\Ob(\B)\).
\item\label{it:main-strict} \textbf{Strictness.} If \(F\) is Kan-transferable to \(G\) along \(J\), then \(G\in\mathcal T_0(J,F)\). If \(\B\) is discrete and \(d_{\V}\) is separated, the converse holds as well. In general it fails: for the categories and functors that \Cref{prop:pw-not-nat} provides for the field \(\K\), with \(\mathcal S\) the class of finite-dimensional spaces and \(d_{\V}\) any extended pseudometric on \(\mathcal S/{\cong}\), one has \(G\in\mathcal T_0(J,F)\) while \(F\) is not Kan-transferable to \(G\) along \(J\).\tighter
\end{enumerate}
\end{theorem}

\begin{proof}
~\begin{steps}
\item \textbf{Membership.} By hypothesis \((\Lan_JF)(b)\in\mathcal S\) for every \(b\in\Ob(\B)\), and every term of \eqref{eq:discrepancy} for \(G=\Lan_JF\) is \(d_{\V}(x,x)=0\).
\item \textbf{Invariance.} For a natural isomorphism \(\sigma:G\Rightarrow G'\) with \(G\in\mathcal T_\varepsilon(J,F)\), each \(\sigma_b\) is an isomorphism, so \(G'(b)\in\mathcal S\) by closure under isomorphism and \([G(b)]=[G'(b)]\). The two suprema \eqref{eq:discrepancy} agree termwise, so \(\Comp_J(F,G')=\Comp_J(F,G)\leq\varepsilon\).
\item \textbf{Zero Class.} A supremum in \([0,\infty]\) vanishes iff every term does, and \(d_{\V}(x,y)=0\) means \(x\sim_0y\), the equality of the images in \(\Sep(\mathcal S,d_{\V})\).
\item \textbf{Separation.} If \(d_{\V}\) is separated, \(\sim_0\) is the equality of \(\mathcal S/{\cong}\), and \labelcref{it:main-sep} becomes \labelcref{it:main-sepcase}.
\item \textbf{Strictness.} If \(\Lan_JF\cong G\), then \labelcref{it:main-lan} and \labelcref{it:main-iso} for \(\varepsilon=0\) give \(G\in\mathcal T_0(J,F)\). If \(\B\) is discrete, \(d_{\V}\) separated and \(G\in\mathcal T_0(J,F)\), then \labelcref{it:main-sepcase} gives isomorphisms \(\sigma_b:(\Lan_JF)(b)\to G(b)\), and their naturality squares commute because the only morphisms of \(\B\) are identities, which both functors send to identities. For the failure, \Cref{prop:pw-not-nat} gives \(\Comp_J(F,G)=0\) for every extended pseudometric on the isomorphism classes of finite-dimensional spaces, while \(\Lan_JF\) and \(G\) are not isomorphic.\qedhere
\end{steps}
\end{proof}

\section{Instantiation in Chain Complexes and Persistence Modules}\label{sec:instantiation}
To be useful, the criterion has to be computable and stable under small changes of the data, and it is both. Chain complexes and persistence modules are built from \(\Vect\) degree by degree, respectively scale by scale, so a single lemma, which presents a colimit of vector spaces as a cokernel, covers both cases.

\begin{lemma}\label{lem:finite-colimit}
Let \(\mathcal I\) be a small category and let \(D:\mathcal I\to\Vect\) be a functor. For a morphism \(\gamma\) of \(\mathcal I\) write \(\gamma:\operatorname{dom}\gamma\to\operatorname{cod}\gamma\), so that \(D(\gamma):D(\operatorname{dom}\gamma)\to D(\operatorname{cod}\gamma)\), and put
\begin{equation}\label{eq:fc-sums}
P\coloneqq\bigoplus_{i\in\Ob(\mathcal I)}D(i),\qquad R\coloneqq\bigoplus_{\gamma\in\Mor(\mathcal I)}D(\operatorname{dom}\gamma)\dc
\end{equation}
with summand inclusions \(\iota_i:D(i)\to P\) and \(\kappa_\gamma:D(\operatorname{dom}\gamma)\to R\). Let \(r:R\to P\) be the unique linear map with \(r\circ\kappa_\gamma=\iota_{\operatorname{cod}\gamma}\circ D(\gamma)-\iota_{\operatorname{dom}\gamma}\) for every \(\gamma\in\Mor(\mathcal I)\), let \(Q\coloneqq\coker(r)\) and let \(p:P\to Q\) be the quotient map. Then
\begin{equation}\label{eq:finite-colimit}
\operatorname*{colim}_{\mathcal I}D\cong Q=\coker\bigl(R\xrightarrow{r}P\bigr)
\end{equation}
with colimit cocone \(\lambda^i\coloneqq p\circ\iota_i:D(i)\to Q\) for \(i\in\Ob(\mathcal I)\), and the isomorphism is natural~in~the~diagram~\(D\).\keeptail
\end{lemma}

\begin{proof}
~\begin{steps}
\item \textbf{Cocones Are the Linear Maps Vanishing on \(\operatorname{im}r\).} A cocone with vertex \(W\) is a family of linear maps \(f_i:D(i)\to W\) with \(f_{\operatorname{cod}\gamma}\circ D(\gamma)=f_{\operatorname{dom}\gamma}\) for every \(\gamma\). Since \(P\) is the coproduct of the \(D(i)\), families \((f_i)_i\) correspond bijectively to linear maps \(f:P\to W\) by \(f\circ\iota_i=f_i\), and for every \(\gamma\) and \(x\in D(\operatorname{dom}\gamma)\)
\begin{equation}
f\bigl(r(\kappa_\gamma(x))\bigr)\eqby{a}f\bigl(\iota_{\operatorname{cod}\gamma}D(\gamma)(x)\bigr)-f\bigl(\iota_{\operatorname{dom}\gamma}(x)\bigr)\eqby{b}f_{\operatorname{cod}\gamma}\bigl(D(\gamma)(x)\bigr)-f_{\operatorname{dom}\gamma}(x)\dd
\end{equation}
Here \rref{a} is the definition of \(r\) and \rref{b} the identity \(f\circ\iota_i=f_i\). The elements \(\kappa_\gamma(x)\) span \(R\), so \(f\) vanishes on \(\operatorname{im}r\) iff \((f_i)_i\) is a cocone.
\item \textbf{The Universal Property.} The map \(p\) is onto with kernel \(\operatorname{im}r\), so \(\bar h\mapsto\bar h\circ p\) is a bijection from \(\operatorname{Hom}(Q,W)\) onto the linear maps \(P\to W\) that vanish on \(\operatorname{im}r\). Composed with the first step, it is a bijection from \(\operatorname{Hom}(Q,W)\) onto the cocones with vertex \(W\), \(\bar h\mapsto(\bar h\circ p\circ\iota_i)_i=(\bar h\circ\lambda^i)_i\). For \(\bar h=\operatorname{id}_Q\) it shows that \((\lambda^i)_i\) is a cocone, and its bijectivity says that every cocone \((f_i)_i\) has a unique \(\bar f:Q\to W\) with \(\bar f\circ\lambda^i=f_i\) for all \(i\), which is \eqref{eq:finite-colimit} with the colimit cocone \((\lambda^i)_i\).
\item \textbf{Naturality in the Diagram.} Let \(\theta:D\Rightarrow D'\) be natural, let \(P',R',r',Q',p',\iota'_i,\kappa'_\gamma,\lambda'^i\) be the data of \(D'\), and put \(\theta^{\Ob}\coloneqq\bigoplus_i\theta_i:P\to P'\) and \(\theta^{\Mor}\coloneqq\bigoplus_\gamma\theta_{\operatorname{dom}\gamma}:R\to R'\). For \(\gamma\in\Mor(\mathcal I)\)~and~\(x\in D(\operatorname{dom}\gamma)\),
\begin{align}
r'\theta^{\Mor}\kappa_\gamma(x)&\eqby{a}\iota'_{\operatorname{cod}\gamma}D'(\gamma)\theta_{\operatorname{dom}\gamma}(x)-\iota'_{\operatorname{dom}\gamma}\theta_{\operatorname{dom}\gamma}(x)\\
&\eqby{b}\iota'_{\operatorname{cod}\gamma}\theta_{\operatorname{cod}\gamma}D(\gamma)(x)-\iota'_{\operatorname{dom}\gamma}\theta_{\operatorname{dom}\gamma}(x)\eqby{c}\theta^{\Ob}r\kappa_\gamma(x)\dd
\end{align}
Here \rref{a} and \rref{c} are the definitions of \(r'\) and \(r\), and \rref{b} is the naturality of \(\theta\). So \(r'\circ\theta^{\Mor}=\theta^{\Ob}\circ r\), the map \(p'\circ\theta^{\Ob}\) vanishes on \(\ker p=\operatorname{im}r\), and a unique \(\bar\theta:Q\to Q'\) satisfies \(\bar\theta\circ p=p'\circ\theta^{\Ob}\). It is the induced map of colimits, since \(\bar\theta\lambda^i=\bar\theta p\iota_i=p'\theta^{\Ob}\iota_i=p'\iota'_i\theta_i=\lambda'^i\theta_i\) for every \(i\).\qedhere
\end{steps}
\end{proof}

We compare chain complexes up to homotopy type, with a distance that only records whether two complexes are chain homotopy equivalent. When chains rather than homology are extended along a gluing, nothing is lost (\Cref{sec:gluing}).

\begin{definition}\label{def:chain-target}
Let \(\mathcal S_{\mathrm{ch}}\subseteq\Ob(\Ch)\) be closed under isomorphism and such that its isomorphism classes form a set, for instance the class of complexes that are finite-dimensional in every degree, each of which is determined up to isomorphism by a dimension function \(\mathbb Z\to\mathbb N\) and a family of matrices. For \(X,Y\in\mathcal S_{\mathrm{ch}}\) write \(X\simeq_{\mathrm{ch}}Y\) if \(X\) and \(Y\) are chain homotopy equivalent, and put \(d_{\mathrm{ch}}([X],[Y])\coloneqq0\) if \(X\simeq_{\mathrm{ch}}Y\) and \(d_{\mathrm{ch}}([X],[Y])\coloneqq1\) otherwise, which depends only on the classes \([X]\) and \([Y]\), since an isomorphism of complexes is a chain homotopy equivalence and chain homotopy equivalences compose.
\end{definition}

Both of our targets satisfy the axioms of \Cref{def:invariant}, and for chain complexes a vanishing discrepancy amounts to equal homology at every target object.

\begin{proposition}\label{prop:targets}
Let \(\K\) be a field, let \(J:\A\to\B\) be a functor between small categories and let \(F:\A\to\Ch\), \(G:\B\to\Ch\) be functors with \(\Lan_JF\) existing and \((\Lan_JF)(b),G(b)\in\mathcal S_{\mathrm{ch}}\) for every \(b\in\Ob(\B)\). The triple \((\Ch,\mathcal S_{\mathrm{ch}},d_{\mathrm{ch}})\) of \Cref{def:chain-target} is an invariant target, and for it \(\Comp_J(F,G)=0\) iff \((\Lan_JF)(b)\simeq_{\mathrm{ch}}G(b)\) for every \(b\in\Ob(\B)\), equivalently iff \(H_n((\Lan_JF)(b))\cong H_n(G(b))\) for every \(b\in\Ob(\B)\) and every \(n\in\mathbb Z\). The triple \((\Pers,\Pers^{\mathrm{pfd}},d_I)\), with the interleaving distance of \Cref{def:interleaving} restricted to the isomorphism classes of pointwise finite-dimensional modules, is an invariant target as well.\tighter
\end{proposition}

\begin{proof}
~\begin{steps}
\item \textbf{Cocompleteness.} The category \(\Vect\) has all small coproducts and coequalizers, hence all small colimits \citep[Chapter~V]{maclane1998categories}, and colimits in \(\Ch\) are formed degreewise \citep[Section~1.2]{weibel1994introduction}. In \(\Pers\) a diagram \(X\) has the pointwise colimit \((\operatorname*{colim}_iX_i)_t\cong\operatorname*{colim}_i(X_i)_t\): the structure maps of the \(X_i\) induce maps between pointwise colimits, a cocone in \(\Pers\) is a family of pointwise cocones compatible with them, and the pointwise mediators commute with the structure maps by their uniqueness.
\item \textbf{The Chain Distance.} Chain homotopy equivalence is an equivalence relation containing isomorphism \citep[Section~1.4]{weibel1994introduction}, so \(d_{\mathrm{ch}}\) is a symmetric function on \(\mathcal S_{\mathrm{ch}}/{\cong}\), zero on the diagonal, and in the triangle inequality either the right-hand side is at least \(1\), which bounds every value, or both terms vanish and transitivity applies.
\item \textbf{The Zero Set of the Chain Discrepancy.} As \(d_{\mathrm{ch}}\) takes the values \(0\) and \(1\), the supremum \eqref{eq:discrepancy} vanishes iff \((\Lan_JF)(b)\simeq_{\mathrm{ch}}G(b)\) for every \(b\). Homotopic chain maps induce the same map on homology \citep[Section~1.4]{weibel1994introduction}, so a homotopy equivalence induces isomorphisms on every \(H_n\). Conversely, let \(X,Y\in\mathcal S_{\mathrm{ch}}\) have \(H_n(X)\cong H_n(Y)\) for every \(n\). Over \(\K\) choose complements \(X_n=\ker\partial_n\oplus V_n\) and \(\ker\partial_n=\operatorname{im}\partial_{n+1}\oplus W_n\). Then \(\partial\) maps \(V_n\) isomorphically onto \(\operatorname{im}\partial_n\), so \(X\) is the direct sum of \((W_n)_n\cong(H_n(X))_n\) with zero differential and of the complexes \(V_n\xrightarrow{\cong}\operatorname{im}\partial_n\), each contractible with the inverse isomorphism as homotopy. Hence \(X\simeq_{\mathrm{ch}}(H_n(X))_n\), likewise for \(Y\), and isomorphisms \(H_n(X)\cong H_n(Y)\) give an isomorphism of these complexes.\tighter
\item \textbf{The Persistent Target.} The class \(\Pers^{\mathrm{pfd}}\) is closed under isomorphism, and its classes form a set, every \(M\) being isomorphic to a module with \(M_t=\K^{\dim_\K M_t}\), given by a function \(\mathbb R\to\mathbb N\) and matrices. The distance \(d_I\) is an extended pseudometric \citep[Theorem~3.3]{bubenik2014categorification}, and isomorphic modules are \(0\)-interleaved, so by the triangle inequality it descends to an extended pseudometric on \(\Pers^{\mathrm{pfd}}/{\cong}\).\qedhere
\end{steps}
\end{proof}

The left argument of the discrepancy is, compatibly with differentials and structure maps, a quotient of a direct sum, which is finite whenever the comma category has finitely many objects and morphisms (\Cref{fig:cokernel}).

\begin{proposition}\label{prop:cokernel}
Let \(J:\A\to\B\) be a functor between small categories and let \(b\in\Ob(\B)\). For \(\gamma\in\Mor(J\downarrow b)\) write \(\bar\gamma\coloneqq\pi_b(\gamma)\), \(\operatorname{dom}\gamma=(a_\gamma,\beta_\gamma)\) and \(\operatorname{cod}\gamma=(a'_\gamma,\beta'_\gamma)\). Let \(F:\A\to\Ch\) or \(F:\A\to\Pers\) be a functor, let \(v\) be an index of the corresponding grading, an integer in the first case and a real number in the second, and write \(X_v\) for the component of \(X\) at \(v\). Put \(P_v\coloneqq\bigoplus_{(a,\beta)\in\Ob(J\downarrow b)}F(a)_v\) and \(R_v\coloneqq\bigoplus_{\gamma\in\Mor(J\downarrow b)}F(a_\gamma)_v\), with summand inclusions \(\iota_{(a,\beta)}:F(a)_v\to P_v\) and \(\kappa_\gamma:F(a_\gamma)_v\to R_v\), and
\begin{equation}\label{eq:cokernel}
Q_v\coloneqq\coker\bigl(R_v\xrightarrow{r_v}P_v\bigr)\dc
\end{equation}
where \(r_v\) is the unique linear map with \(r_v\circ\kappa_\gamma=\iota_{\operatorname{cod}\gamma}\circ F(\bar\gamma)_v-\iota_{\operatorname{dom}\gamma}\) for all \(\gamma\). Then the following hold, naturally in \(F\) and in \(b\):
\begin{enumerate}
\item\label{it:cok-ch} \textbf{Chain Complexes.} For \(F:\A\to\Ch\) and \(n\in\mathbb Z\) one has \((\Lan_JF)(b)_n\cong Q_n\) as \(\K\)-vector spaces, with the differential \(Q_n\to Q_{n-1}\) induced on the cokernels by \(\bigoplus_{(a,\beta)}\partial_{F(a),n}\), where \(\partial_{F(a),n}:F(a)_n\to F(a)_{n-1}\) is the differential of \(F(a)\).
\item\label{it:cok-pers} \textbf{Persistence Modules.} For \(F:\A\to\Pers\) and \(u\in\mathbb R\) one has \((\Lan_JF)(b)_u\cong Q_u\) as \(\K\)-vector spaces, with the structure map \(Q_s\to Q_t\) for \(s\leq t\) induced on the cokernels by \(\bigoplus_{(a,\beta)}F(a)_{s\leq t}\).\tighter
\end{enumerate}
\end{proposition}

\begin{figure}[!htbp]
\centering
\hvpic{}\par\vspace{3mm}
\hvpic{}\space\space
\caption{\Cref{prop:cokernel} for the comma category of \Cref{fig:pointwise} at \(v=0.40\): one basis vector of \(P_v\) per bar alive at \(v\), one block of columns of \(r_v\) per morphism. Identities give zero columns, and \(\gamma:x\to y\), with \(\bar\gamma=\bar\alpha\), identifies the bar of \(F(a_1)\) with its image in \(F(a_2)\), so \(\dim_\K Q_v=4\). Below, the transition maps descend to the map \(\overline\Phi_{s,t}\) of cokernels, and both squares commute.}
\label{fig:cokernel}
\end{figure}

\begin{proof}
~\begin{steps}
\item \textbf{Reduction to Vector Spaces.} By \itemref{prop:basic-task-changes}{it:forget}, \(L\coloneqq(\Lan_JF)(b)\) is the colimit of \(F\pi_b\), with legs \(\lambda^{(a,\beta)}\). Colimits in \(\Ch\) and in \(\Pers\) are formed degreewise, respectively pointwise, by the proof of \Cref{prop:targets}, so \(L_v\) is the colimit of \(D_v:J\downarrow b\to\Vect\), \((a,\beta)\mapsto F(a)_v\), \(\gamma\mapsto F(\bar\gamma)_v\), whose data \eqref{eq:fc-sums} are \(R_v\), \(P_v\) and \(r_v\). \Cref{lem:finite-colimit} gives \(L_v\cong Q_v\) with \(\lambda^{(a,\beta)}_v=p_v\circ\iota_{(a,\beta)}\), where \(p_v:P_v\to Q_v\) is the quotient map. It remains to identify the maps between two indices.
\item \textbf{Descent of the Transition Maps.} In \labelcref{it:cok-ch} put \(v'\coloneqq n-1\) and \(\varphi_{(a,\beta)}\coloneqq\partial_{F(a),n}\), in \labelcref{it:cok-pers} \(v\coloneqq s\), \(v'\coloneqq t\) and \(\varphi_{(a,\beta)}\coloneqq F(a)_{s\leq t}\), and let \(\Phi_v\coloneqq\bigoplus_{(a,\beta)}\varphi_{(a,\beta)}:P_v\to P_{v'}\) and \(\Psi_v\coloneqq\bigoplus_\gamma\varphi_{\operatorname{dom}\gamma}:R_v\to R_{v'}\). Each \(F(\bar\gamma)\) is a chain map, respectively a morphism of persistence modules, so on the summand of \(\gamma\)\tighter
\begin{equation}
\begin{aligned}
\Phi_vr_v\kappa_\gamma&\eqby{a}\iota_{\operatorname{cod}\gamma}\varphi_{\operatorname{cod}\gamma}F(\bar\gamma)_v-\iota_{\operatorname{dom}\gamma}\varphi_{\operatorname{dom}\gamma}\\
&\eqby{b}\bigl(\iota_{\operatorname{cod}\gamma}F(\bar\gamma)_{v'}-\iota_{\operatorname{dom}\gamma}\bigr)\varphi_{\operatorname{dom}\gamma}\eqby{c}r_{v'}\Psi_v\kappa_\gamma\dd
\end{aligned}
\end{equation}
Here \rref{a} and \rref{c} are the definitions of \(r_v\) and \(r_{v'}\), and \rref{b} says that \(F(\bar\gamma)\) commutes with \(\varphi\). Hence \(\Phi_vr_v=r_{v'}\Psi_v\), the map \(p_{v'}\Phi_v\) vanishes on \(\ker p_v=\operatorname{im}r_v\), and a unique \(\overline\Phi_v:Q_v\to Q_{v'}\) satisfies \(\overline\Phi_vp_v=p_{v'}\Phi_v\). In \labelcref{it:cok-ch}, \(\overline\Phi_{n-1}\overline\Phi_np_n=p_{n-2}\Phi_{n-1}\Phi_n=0\) with \(p_n\) onto, so~\(Q_\bullet\)~is~a~chain~complex.
\item \textbf{Identification of the Descended Map.} The differential, respectively structure map, \(L_{v\to v'}\) is the only linear map with \(L_{v\to v'}\lambda^{(a,\beta)}_v=\lambda^{(a,\beta)}_{v'}\varphi_{(a,\beta)}\) for all \((a,\beta)\): it has this property because each leg is a morphism, and it is unique because the legs \(\lambda^{(a,\beta)}_v=p_v\iota_{(a,\beta)}\) are~jointly~surjective.~Now\keeptail
\begin{equation}
\overline\Phi_v\lambda^{(a,\beta)}_v\eqby{a}\overline\Phi_vp_v\iota_{(a,\beta)}\eqby{b}p_{v'}\Phi_v\iota_{(a,\beta)}\eqby{c}p_{v'}\iota_{(a,\beta)}\varphi_{(a,\beta)}\eqby{a}\lambda^{(a,\beta)}_{v'}\varphi_{(a,\beta)}\dc
\end{equation}
with \rref{a} the leg formula of the first step and \rref{b}, \rref{c} the definitions of \(\overline\Phi_v\) and \(\Phi_v\), so \(\overline\Phi_v=L_{v\to v'}\) under the isomorphisms of the first step.
\item \textbf{Naturality in \(F\) and in \(b\).} For \(\theta:F\Rightarrow F'\) the components \((\theta_a)_v\) form a natural transformation \(D_v\Rightarrow D'_v\), and the last step of the proof of \Cref{lem:finite-colimit} gives \(\overline\theta_v:Q_v\to Q'_v\) with \(\overline\theta_v\lambda^{(a,\beta)}_v=\lambda'^{(a,\beta)}_v(\theta_a)_v\). The component \((\Lan_J\theta)(b)\) satisfies the same identity,
\[
\begin{aligned}
(\Lan_J\theta)(b)\circ\lambda^{(a,\beta)}&\eqby{a}(\Lan_JF')(\beta)\circ(\Lan_J\theta)(Ja)\circ\eta^F_a\\
&\eqby{b}(\Lan_JF')(\beta)\circ\eta^{F'}_a\circ\theta_a=\lambda'^{(a,\beta)}\circ\theta_a\dc
\end{aligned}
\]
by \rref{a} the naturality of \(\Lan_J\theta\) and \rref{b} its definition, and the legs \(\lambda^{(a,\beta)}_v\) are jointly surjective, so \(\overline\theta_v\) is \((\Lan_J\theta)(b)_v\) under the isomorphisms of the first step. For \(\beta_0:b\to b_0\), postcomposition is a functor \((\beta_0)_*:J\downarrow b\to J\downarrow b_0\), \((a,\beta)\mapsto(a,\beta_0\circ\beta)\), \(\bar\alpha\mapsto\bar\alpha\), with \(D^{b_0}_v\circ(\beta_0)_*=D^b_v\).~The~maps
\[\beta_0^{\Ob}:P^b_v\to P^{b_0}_v,\quad\iota_{(a,\beta)}\mapsto\iota_{(a,\beta_0\circ\beta)},\qquad\beta_0^{\Mor}:R^b_v\to R^{b_0}_v,\quad\kappa_\gamma\mapsto\kappa_{(\beta_0)_*\gamma},\]
satisfy \(r^{b_0}_v\beta_0^{\Mor}=\beta_0^{\Ob}r^b_v\), both sending \(\kappa_\gamma(x)\) to
\[\iota_{(a'_\gamma,\beta_0\circ\beta'_\gamma)}F(\bar\gamma)_v(x)-\iota_{(a_\gamma,\beta_0\circ\beta_\gamma)}(x)\dd\]
So \(\beta_0^{\Ob}\) induces a unique map on the cokernels, sending \(\lambda^{(a,\beta)}_v\) to \(\lambda^{(a,\beta_0\circ\beta)}_v\), which by the leg formula \(\lambda^{(a,\beta_0\circ\beta)}=(\Lan_JF)(\beta_0)\circ\lambda^{(a,\beta)}\) of \itemref{prop:basic-task-changes}{it:forget} characterises \((\Lan_JF)(\beta_0)_v\).\qedhere
\end{steps}
\end{proof}

The sums may be infinite, but they are finite when \(J\downarrow b\) is finite. The comma category \(J\downarrow\bullet\) of the gluing in \Cref{ex:arcs}, for example, is the span with three objects and five morphisms. In degree zero, take a scale \(v\) at which every arc and each of the two overlap arcs is connected but the overlap arcs are not yet joined to each other. Then \(P_v=\K\oplus\K^2\oplus\K\), with one basis vector \(\hat u_1\) for the first arc, \(\hat u'_{12}\) and \(\hat u''_{12}\) for the overlap arcs and \(\hat u_2\) for the second arc. The identities give zero columns, and the two inclusions give the relations \(\hat u_1-\hat u'_{12}\), \(\hat u_1-\hat u''_{12}\), \(\hat u_2-\hat u'_{12}\), \(\hat u_2-\hat u''_{12}\) of rank three, so \(Q_v\cong\K\), which is one component. For the merge we have \(r_v=0\) and \(Q_v=\K\oplus\K\), that is, two components.

Stability requires an interleaving distance between diagrams. For a small category \(\mathcal C\), postcomposition with \(\Sigma_\varepsilon\) is a functor \(\Sigma^{\mathcal C}_\varepsilon:\Fun{\mathcal C}{\Pers}\to\Fun{\mathcal C}{\Pers}\), \(X\mapsto\Sigma_\varepsilon\circ X\), \(\theta\mapsto\Sigma_\varepsilon\theta\), so that \((\Sigma^{\mathcal C}_\varepsilon X)(c)=\Sigma_\varepsilon(X(c))\), \((\Sigma^{\mathcal C}_\varepsilon\theta)_c=\Sigma_\varepsilon\theta_c\) and \(\Sigma^{\mathcal C}_\varepsilon\Sigma^{\mathcal C}_{\varepsilon'}=\Sigma^{\mathcal C}_{\varepsilon+\varepsilon'}\). The transition morphism \(\tau^X_\varepsilon:X\Rightarrow\Sigma^{\mathcal C}_\varepsilon X\) has the components \(\tau^{X(c)}_\varepsilon\). They are natural in \(c\) because every \(X(g)\), \(g:c\to c'\), is a morphism of persistence modules, so that
\[\Sigma_\varepsilon X(g)\circ\tau^{X(c)}_\varepsilon=\tau^{X(c')}_\varepsilon\circ X(g)\dc\]
and natural in \(X\) because every component of a natural transformation \(X\Rightarrow Y\) is a morphism of persistence modules as well.

\begin{definition}\label{def:functorial-discrepancy}
Let \(\mathcal C\) be a small category and \(X,Y:\mathcal C\to\Pers\) functors. An \(\varepsilon\)-\hspace{0pt}\textbf{interleaving of diagrams} is a pair of natural transformations \(\varphi:X\Rightarrow\Sigma^{\mathcal C}_\varepsilon Y\) and \(\psi:Y\Rightarrow\Sigma^{\mathcal C}_\varepsilon X\) with \(\Sigma^{\mathcal C}_\varepsilon\psi\circ\varphi=\tau^X_{2\varepsilon}\) and \(\Sigma^{\mathcal C}_\varepsilon\varphi\circ\psi=\tau^Y_{2\varepsilon}\), and \(d_I^{\mathcal C}(X,Y)\in[0,\infty]\) is the infimum of the \(\varepsilon\geq0\) for which one exists, with \(\inf\varnothing\coloneqq\infty\). For \(F:\A\to\Pers\) and \(G:\B\to\Pers\) with \(\Lan_JF\) existing, the \textbf{functorial transfer discrepancy} is \(\Comp^{\mathrm{nat}}_J(F,G)\coloneqq d_I^{\B}(\Lan_JF,G)\).
\end{definition}

{\addfontfeatures{LetterSpace=-1}The distance \(d_I^{\mathcal C}\) is an extended pseudometric and invariant under isomorphism. In particular, the functorial discrepancy does not depend on the chosen extension, and \(d_I^{\mathbf 1}\) agrees with \(d_I\).}\tighter

\begin{lemma}\label{lem:diagram-pseudometric}
Let \(\mathcal C\) be a small category. The distance \(d_I^{\mathcal C}\) of \Cref{def:functorial-discrepancy} is an extended pseudometric on the functors \(\mathcal C\to\Pers\), and
\[d_I^{\mathcal C}(X,Y)=d_I^{\mathcal C}(X',Y)\qquad\text{whenever }X\cong X'\text{ in }\Fun{\mathcal C}{\Pers}\dd\]
For \(M\in\Pers\) let \(\widehat M:\mathbf 1\to\Pers\) be the functor with \(\widehat M(\bullet)=M\) and \(\widehat M(\operatorname{id}_\bullet)=\operatorname{id}_M\). Then
\[d_I^{\mathbf 1}(\widehat M,\widehat N)=d_I(M,N)\qquad(M,N\in\Pers)\dd\]
\end{lemma}

\begin{proof}
Let \(X,X',Y,Z:\mathcal C\to\Pers\) be functors and write \(\Sigma_\varepsilon\) for \(\Sigma^{\mathcal C}_\varepsilon\).
\begin{steps}
\item \textbf{Transition Morphisms Commute with Shifts and Transformations.} For \(\varepsilon,\varepsilon'\geq0\) and every natural transformation \(\nu:X\Rightarrow Y\),
\begin{equation}\label{eq:dp-calc}
\Sigma_\varepsilon\tau^X_{\varepsilon'}=\tau^{\Sigma_\varepsilon X}_{\varepsilon'},\qquad\tau^{\Sigma_\varepsilon X}_{\varepsilon'}\circ\tau^X_\varepsilon=\tau^X_{\varepsilon+\varepsilon'},\qquad\tau^Y_\varepsilon\circ\nu=\Sigma_\varepsilon\nu\circ\tau^X_\varepsilon\dc
\end{equation}
since at an object \(c\) and a scale \(t\) the first identity reads \(X(c)_{t+\varepsilon\leq t+\varepsilon+\varepsilon'}=X(c)_{t+\varepsilon\leq t+\varepsilon+\varepsilon'}\), the second is functoriality of \(X(c)\), and the third is naturality of \(\nu_c\) at \(t\leq t+\varepsilon\).
\item \textbf{Reflexivity and Symmetry.} The pair \((\operatorname{id}_X,\operatorname{id}_X)\) is a \(0\)-interleaving, as \(\Sigma_0=\operatorname{id}\) and \(\tau^X_0=\operatorname{id}_X\), and exchanging \((X,\varphi)\) with \((Y,\psi)\) exchanges the two conditions.
\item \textbf{Triangle Inequality.} {\addfontfeatures{LetterSpace=-1.5}Let \((\varphi,\psi)\) be an \(\varepsilon\)-interleaving of \(X\) and \(Y\) and \((\varphi',\psi')\) an \(\varepsilon'\)-interleaving of \(Y\) and \(Z\), and put \(\Phi\coloneqq\Sigma_\varepsilon\varphi'\circ\varphi:X\Rightarrow\Sigma_{\varepsilon+\varepsilon'}Z\) and \(\Psi\coloneqq\Sigma_{\varepsilon'}\psi\circ\psi':Z\Rightarrow\Sigma_{\varepsilon+\varepsilon'}X\). Then}\tighter
\begin{align}
\Sigma_{\varepsilon+\varepsilon'}\Psi\circ\Phi
&\eqby{a}\Sigma_{\varepsilon+2\varepsilon'}\psi\circ\Sigma_\varepsilon\bigl(\Sigma_{\varepsilon'}\psi'\circ\varphi'\bigr)\circ\varphi
\eqby{b}\Sigma_{\varepsilon+2\varepsilon'}\psi\circ\Sigma_\varepsilon\tau^Y_{2\varepsilon'}\circ\varphi\\
&\eqby{c}\Sigma_{\varepsilon+2\varepsilon'}\psi\circ\tau^{\Sigma_\varepsilon Y}_{2\varepsilon'}\circ\varphi
\eqby{c}\Sigma_{\varepsilon+2\varepsilon'}\psi\circ\Sigma_{2\varepsilon'}\varphi\circ\tau^X_{2\varepsilon'}\\
&\eqby{b}\Sigma_{2\varepsilon'}\tau^X_{2\varepsilon}\circ\tau^X_{2\varepsilon'}
\eqby{c}\tau^X_{2\varepsilon+2\varepsilon'}\dc
\end{align}
where \rref{a} is the functoriality of \(\Sigma\), \rref{b} are the interleaving identities \(\Sigma_{\varepsilon'}\psi'\circ\varphi'=\tau^Y_{2\varepsilon'}\) and \(\Sigma_\varepsilon\psi\circ\varphi=\tau^X_{2\varepsilon}\), and \rref{c} are the identities \eqref{eq:dp-calc}, the third one applied to \(\nu=\varphi:X\Rightarrow\Sigma_\varepsilon Y\). With \((X,\varphi,\psi,\varepsilon)\) and \((Z,\psi',\varphi',\varepsilon')\) exchanged the same computation gives \(\Sigma_{\varepsilon+\varepsilon'}\Phi\circ\Psi=\tau^Z_{2\varepsilon+2\varepsilon'}\), and the infimum over \(\varepsilon\) and \(\varepsilon'\) gives the triangle inequality in \([0,\infty]\).
\item \textbf{Invariance under Isomorphism.} Let \(\nu:X\Rightarrow X'\) be a natural isomorphism and \((\varphi,\psi)\) an \(\varepsilon\)-interleaving of \(X\) and \(Y\). Then\looser
\begin{equation}
\Sigma_\varepsilon\bigl(\Sigma_\varepsilon\nu\circ\psi\bigr)\circ\varphi\circ\nu^{-1}\eqby{a}\Sigma_{2\varepsilon}\nu\circ\tau^X_{2\varepsilon}\circ\nu^{-1}\eqby{b}\tau^{X'}_{2\varepsilon}\circ\nu\circ\nu^{-1}=\tau^{X'}_{2\varepsilon}\dc
\end{equation}
by \rref{a} the interleaving identity for \(X\) and \rref{b} the third identity of \eqref{eq:dp-calc}, and \(\Sigma_\varepsilon(\varphi\circ\nu^{-1})\circ\Sigma_\varepsilon\nu\circ\psi=\Sigma_\varepsilon\varphi\circ\psi=\tau^Y_{2\varepsilon}\), so \((\varphi\circ\nu^{-1},\Sigma_\varepsilon\nu\circ\psi)\) is an \(\varepsilon\)-interleaving of \(X'\) and \(Y\). Hence \(d_I^{\mathcal C}(X',Y)\leq d_I^{\mathcal C}(X,Y)\), and \(\nu^{-1}\) gives the reverse inequality.
\item \textbf{Diagrams over the Terminal Category.} Since \(\Mor(\mathbf 1)=\{\operatorname{id}_\bullet\}\), a natural transformation \(\widehat M\Rightarrow\widehat N\) is a single morphism \(M\to N\), \(\Sigma^{\mathbf 1}_\varepsilon\widehat M=\widehat{\Sigma_\varepsilon M}\) and \(\tau^{\widehat M}_\varepsilon=\tau^M_\varepsilon\), so the \(\varepsilon\)-interleavings of \(\widehat M\) and \(\widehat N\) are those of \(M\) and \(N\).\qedhere
\end{steps}
\end{proof}

The functorial discrepancy dominates the objectwise discrepancy, and the two agree on discrete target tasks (\Cref{fig:discrepancy}).

\begin{proposition}\label{prop:pointwise-natural-comparison}
Let \(J:\A\to\B\) be a functor between small categories and let \(F:\A\to\Pers\), \(G:\B\to\Pers\) be functors. Then
\[\sup_{b\in\Ob(\B)}d_I\bigl((\Lan_JF)(b),G(b)\bigr)\leq\Comp^{\mathrm{nat}}_J(F,G)\dc\]
with equality when \(\B\) is discrete. In the target \((\Pers,\Pers^{\mathrm{pfd}},d_I)\) the left-hand side is \(\Comp_J(F,G)\) whenever all values of \(\Lan_JF\) and \(G\) are pointwise finite-dimensional.
\end{proposition}

\begin{figure}[!htbp]
\centering
\hvpic{}\hspace{10mm}\ignorespaces
\hvpic{}\space\space
\caption{Left: \(\Comp_J(F,G)\) is the largest of four distances \(d_I(L(b),G(b))\), here \(0.06\) at \(b_2\). Right: \(\Comp^{\mathrm{nat}}_J(F,G)\) also requires an interleaving to commute with every \(\beta:b\to b'\), so it bounds the objectwise discrepancy from above.}
\label{fig:discrepancy}
\end{figure}

\begin{proof}
~\begin{steps}
\item \textbf{The Inequality.} Let \((\varphi,\psi)\) be an \(\varepsilon\)-interleaving of \(\Lan_JF\) and \(G\) and fix \(b\). Since \((\Sigma^{\B}_\varepsilon X)(b)=\Sigma_\varepsilon(X(b))\) and \((\Sigma^{\B}_\varepsilon\theta)_b=\Sigma_\varepsilon\theta_b\), the components \(\varphi_b\) and \(\psi_b\) satisfy \(\Sigma_\varepsilon\psi_b\circ\varphi_b=\tau^{(\Lan_JF)(b)}_{2\varepsilon}\) and \(\Sigma_\varepsilon\varphi_b\circ\psi_b=\tau^{G(b)}_{2\varepsilon}\), an \(\varepsilon\)-interleaving in \(\Pers\), so \(d_I((\Lan_JF)(b),G(b))\leq\varepsilon\). Every admitted \(\varepsilon\) bounds the supremum over \(b\), and so does their infimum \(\Comp^{\mathrm{nat}}_J(F,G)\), which is defined because \(\Pers\) is cocomplete (\Cref{prop:targets}).
\item \textbf{Equality for Discrete Targets.} Let \(\B\) be discrete and \(\varepsilon_0\) the left-hand side. If \(\varepsilon_0=\infty\), the inequality gives equality. Otherwise let \(\varepsilon>\varepsilon_0\), fix \(b\), and write \(M\coloneqq(\Lan_JF)(b)\) and \(N\coloneqq G(b)\). As \(d_I(M,N)<\varepsilon\), some \(\varepsilon'\leq\varepsilon\) admits an \(\varepsilon'\)-interleaving \((\varphi',\psi')\), and \(\varphi_{b,t}\coloneqq N_{t+\varepsilon'\leq t+\varepsilon}\circ\varphi'_t\), \(\psi_{b,t}\coloneqq M_{t+\varepsilon'\leq t+\varepsilon}\circ\psi'_t\) define morphisms \(\varphi_b:M\to\Sigma_\varepsilon N\) and \(\psi_b:N\to\Sigma_\varepsilon M\), composites of morphisms with transition morphisms. At every \(t\in\mathbb R\)
\begin{align}
(\Sigma_\varepsilon\psi_b\circ\varphi_b)_t
&\eqby{a}M_{t+\varepsilon+\varepsilon'\leq t+2\varepsilon}\circ\psi'_{t+\varepsilon}\circ N_{t+\varepsilon'\leq t+\varepsilon}\circ\varphi'_t\\
&\eqby{b}M_{t+\varepsilon+\varepsilon'\leq t+2\varepsilon}\circ M_{t+2\varepsilon'\leq t+\varepsilon+\varepsilon'}\circ\psi'_{t+\varepsilon'}\circ\varphi'_t\\
&\eqby{c}M_{t+2\varepsilon'\leq t+2\varepsilon}\circ M_{t\leq t+2\varepsilon'}\eqby{d}(\tau^M_{2\varepsilon})_t\dc
\end{align}
by \rref{a} the definitions of \(\varphi_b\) and \(\psi_b\), \rref{b} the naturality of \(\psi'\), \rref{c} the interleaving identity for \((\varphi',\psi')\) and \rref{d} the functoriality of \(M\), and symmetrically \(\Sigma_\varepsilon\varphi_b\circ\psi_b=\tau^N_{2\varepsilon}\). The only morphisms of \(\B\) are identities, which all four functors send to identities, so the families \((\varphi_b)_b\) and \((\psi_b)_b\) are natural, their composites are computed componentwise, and they form an \(\varepsilon\)-interleaving of diagrams. Hence \(\Comp^{\mathrm{nat}}_J(F,G)\leq\varepsilon\) for every \(\varepsilon>\varepsilon_0\).
\item \textbf{The Objectwise Discrepancy.} If all values are pointwise finite-dimensional, they lie in \(\Pers^{\mathrm{pfd}}\), and \eqref{eq:discrepancy} in \((\Pers,\Pers^{\mathrm{pfd}},d_I)\) is the left-hand side.\qedhere
\end{steps}
\end{proof}

The objectwise discrepancy compares the target objects one at a time. The functorial discrepancy also requires the comparison to respect the maps between them, which the objectwise discrepancy does not detect (\Cref{prop:pw-not-nat}). Every target task in \Cref{sec:experiments} is discrete, so there the two discrepancies agree. The extension itself is non-expansive, and the discrepancy is \(1\)-Lipschitz in the source and in the observation (\Cref{fig:stability}).

\begin{theorem}\label{thm:lan-interleaving-stability}
Let \(J:\A\to\B\) be a functor between small categories, \(F,F':\A\to\Pers\) and \(G,G':\B\to\Pers\) functors. Then \(d_I^{\B}(\Lan_JF,\Lan_JF')\leq d_I^{\A}(F,F')\) and
\begin{equation}\label{eq:lipschitz}
\Comp^{\mathrm{nat}}_J(F,G)\leq\Comp^{\mathrm{nat}}_J(F',G')+d_I^{\A}(F,F')+d_I^{\B}(G,G')
\end{equation}
in \([0,\infty]\), and \(|\Comp^{\mathrm{nat}}_J(F,G)-\Comp^{\mathrm{nat}}_J(F',G')|\leq d_I^{\A}(F,F')+d_I^{\B}(G,G')\) whenever both discrepancies are finite.
\end{theorem}

\begin{figure}[!htbp]
\centering
\hvpic{}\space\space
\caption{\Cref{thm:lan-interleaving-stability} for a merge: moving the source by \(0.03\) moves the extension by \(0.03\), and when the observation is moved by \(0.04\) as well, the discrepancy moves from \(c=0.02\) to \(c'=0.06\),~within~the~bound~\(0.07\).}
\label{fig:stability}
\end{figure}

\begin{proof}
~\begin{steps}
\item \textbf{Transport along the Extensions.} Write \((L,\eta)\coloneqq(\Lan_JF,\eta^F)\) and \((L',\eta')\coloneqq(\Lan_JF',\eta^{F'})\), which exist because \(\Pers\) is cocomplete (\Cref{prop:targets}). Postcomposition with \(\Sigma_\delta\) commutes with precomposition with \(J\),
\[(\Sigma^{\B}_\delta E)J=\Sigma^{\A}_\delta(EJ)\qquad(E:\B\to\Pers)\dd\]
For \(\alpha:F\Rightarrow\Sigma^{\A}_\delta F'\) the universal property of \((L,\eta)\), applied to \(\Sigma^{\B}_\delta L'\) and \(\Sigma^{\A}_\delta\eta'\circ\alpha\), gives a unique \(\hat\alpha:L\Rightarrow\Sigma^{\B}_\delta L'\) with
\begin{equation}\label{eq:li-hat}
(\hat\alpha J)\circ\eta=\Sigma^{\A}_\delta\eta'\circ\alpha\dc
\end{equation}
and likewise every \(\beta:F'\Rightarrow\Sigma^{\A}_{\delta'}F\) has a unique \(\hat\beta:L'\Rightarrow\Sigma^{\B}_{\delta'}L\) with \((\hat\beta J)\circ\eta'=\Sigma^{\A}_{\delta'}\eta\circ\beta\). Whiskering by \(J\) preserves vertical composition and commutes with the shifts, so
\begin{align}
\bigl((\Sigma^{\B}_\delta\hat\beta\circ\hat\alpha)J\bigr)\circ\eta
&\eqby{a}\Sigma^{\A}_\delta(\hat\beta J)\circ\Sigma^{\A}_\delta\eta'\circ\alpha
\eqby{b}\Sigma^{\A}_\delta\bigl((\hat\beta J)\circ\eta'\bigr)\circ\alpha\\
&\eqby{c}\Sigma^{\A}_\delta\bigl(\Sigma^{\A}_{\delta'}\eta\circ\beta\bigr)\circ\alpha
\eqby{b}\Sigma^{\A}_{\delta+\delta'}\eta\circ\bigl(\Sigma^{\A}_\delta\beta\circ\alpha\bigr)\dd\label{eq:li-two-2}
\end{align}
Here \rref{a} is \eqref{eq:li-hat}, \rref{b} the functoriality of \(\Sigma\) and \rref{c} the defining identity of \(\hat\beta\). The transition morphisms satisfy \(\tau^L_{\varepsilon'}J=\tau^{LJ}_{\varepsilon'}\) componentwise, and the third identity of \eqref{eq:dp-calc} for \(\nu=\eta\) gives
\begin{equation}\label{eq:li-tau}
(\tau^L_{\varepsilon'}J)\circ\eta=\Sigma^{\A}_{\varepsilon'}\eta\circ\tau^F_{\varepsilon'}\qquad(\varepsilon'\geq0)\dd
\end{equation}
\item \textbf{Interleavings Are Preserved.} Let \((\varphi,\psi)\) be an \(\varepsilon\)-interleaving of \(F\) and \(F'\), and \(\Phi\coloneqq\hat\varphi:L\Rightarrow\Sigma^{\B}_\varepsilon L'\), \(\Psi\coloneqq\hat\psi:L'\Rightarrow\Sigma^{\B}_\varepsilon L\). With \(\delta=\delta'=\varepsilon\),
\begin{equation}
\bigl((\Sigma^{\B}_\varepsilon\Psi\circ\Phi)J\bigr)\circ\eta\eqby{a}\Sigma^{\A}_{2\varepsilon}\eta\circ(\Sigma^{\A}_\varepsilon\psi\circ\varphi)\eqby{b}\Sigma^{\A}_{2\varepsilon}\eta\circ\tau^F_{2\varepsilon}\eqby{c}(\tau^L_{2\varepsilon}J)\circ\eta\dc
\end{equation}
by \rref{a} the identity \eqref{eq:li-two-2}, \rref{b} the interleaving identity for \((\varphi,\psi)\) and \rref{c} the identity \eqref{eq:li-tau}, so the uniqueness clause of the universal property of \((L,\eta)\) gives \(\Sigma^{\B}_\varepsilon\Psi\circ\Phi=\tau^L_{2\varepsilon}\). Exchanging the primed and unprimed data gives \(\Sigma^{\B}_\varepsilon\Phi\circ\Psi=\tau^{L'}_{2\varepsilon}\), so \((\Phi,\Psi)\) is an \(\varepsilon\)-interleaving, and the infimum gives the first claim.
\item \textbf{Stability of the Discrepancy.} With \(c\coloneqq\Comp^{\mathrm{nat}}_J(F,G)\) and \(c'\coloneqq\Comp^{\mathrm{nat}}_J(F',G')\), the triangle inequality and symmetry of \(d_I^{\B}\) along \(\Lan_JF\), \(\Lan_JF'\), \(G'\), \(G\) give
\begin{equation}
c\leby{a}d_I^{\B}(\Lan_JF,\Lan_JF')+c'+d_I^{\B}(G,G')\leby{b}d_I^{\A}(F,F')+c'+d_I^{\B}(G,G')\dc
\end{equation}
by \rref{a} \Cref{lem:diagram-pseudometric} and \rref{b} the first claim, which is \eqref{eq:lipschitz}. Exchanging the primed and unprimed data gives \(c'\leq d_I^{\A}(F,F')+c+d_I^{\B}(G,G')\), and for finite \(c\) and \(c'\) the two inequalities give \(|c-c'|\leq d_I^{\A}(F,F')+d_I^{\B}(G,G')\).\qedhere
\end{steps}
\end{proof}

For a merge into \(\mathbf 1\) the two discrepancies agree, and the bound extends to weighted sums over the homological degrees.

\begin{corollary}\label{cor:terminal-weighted-stability}
Let \(J:\A\to\mathbf 1\) be the unique functor from a small category \(\A\), let \(q\in\mathbb N\) and \(w_0,\ldots,w_q\geq0\), and let \(F_n,F'_n:\A\to\Pers\) and \(G_n,G'_n:\mathbf 1\to\Pers\) be functors for \(0\leq n\leq q\), where we write \(G_n\) also for the object \(G_n(\bullet)\). Put \(\Omega\coloneqq\sum_{n=0}^{q}w_nd_I((\Lan_JF_n)(\bullet),G_n)\) and \(\Omega'\coloneqq\sum_{n=0}^{q}w_nd_I((\Lan_JF'_n)(\bullet),G'_n)\). If \(\Omega\) and \(\Omega'\) are finite, then
\begin{equation}\label{eq:weighted-stability}
|\Omega-\Omega'|\leq\sum_{n=0}^{q}w_n\bigl(d_I^{\A}(F_n,F'_n)+d_I(G_n,G'_n)\bigr)\dc
\end{equation}
with the convention \(0\cdot\infty\coloneqq0\).
\end{corollary}

\begin{proof}
By \Cref{lem:diagram-pseudometric}, \(d_I^{\mathbf 1}(G_n,G'_n)=d_I(G_n,G'_n)\), and since \(\mathbf 1\) is discrete, \Cref{prop:pointwise-natural-comparison} gives \(c_n\coloneqq d_I((\Lan_JF_n)(\bullet),G_n)=\Comp^{\mathrm{nat}}_J(F_n,G_n)\) and likewise \(c'_n=\Comp^{\mathrm{nat}}_J(F'_n,G'_n)\). For \(n\in N\coloneqq\{n\mid w_n>0\}\) the nonnegative summands give \(w_nc_n\leq\Omega<\infty\) and \(w_nc'_n\leq\Omega'<\infty\), so \(c_n\) and \(c'_n\) are finite. Summands with \(w_n=0\) vanish in \(\Omega\), in \(\Omega'\) and on the right-hand side of \eqref{eq:weighted-stability} by the convention \(0\cdot\infty\coloneqq0\). Then
\begin{equation}\label{eq:tws-chain}
\begin{aligned}
|\Omega-\Omega'|&=\Bigl|\sum_{n\in N}w_n(c_n-c'_n)\Bigr|\leby{a}\sum_{n\in N}w_n|c_n-c'_n|\\
&\leby{b}\sum_{n\in N}w_n\bigl(d_I^{\A}(F_n,F'_n)+d_I(G_n,G'_n)\bigr)\dc
\end{aligned}
\end{equation}
by \rref{a} the triangle inequality in \(\mathbb R\) and \rref{b} the third clause of \Cref{thm:lan-interleaving-stability} for the finite \(c_n\) and \(c'_n\), and the right-hand side of \eqref{eq:tws-chain} is that of \eqref{eq:weighted-stability}.
\end{proof}

For frozen invariants, a perturbation that moves a death across the freezing scale changes the number of essential classes, however small the perturbation is, and the interleaving distance jumps to \(\infty\). Such a jump occurs in the example after \Cref{cor:separated-merge-robustness}, where moving one point by \(\Delta\) turns two essential classes in degree zero into one. In such cases the bound gives no information about the frozen score. The experiments also report a truncated score, to which the bound applies (\Cref{lem:truncation}).

Changes of task can be composed, and the functorial discrepancy is subadditive along the composite. A change of task assembled from elementary ones can thus be bounded one~step~at~a~time~(\Cref{fig:composition}).

\begin{proposition}\label{prop:composition}
Let \(J:\A\to\B\) and \(J':\B\to\mathcal C\) be functors between small categories.
\begin{enumerate}
\item\label{it:comp-lan} \textbf{Composition.} For every cocomplete category \(\V\) and every functor \(F:\A\to\V\) one has \(\Lan_{J'J}F\cong\Lan_{J'}(\Lan_JF)\) in \(\Fun{\mathcal C}{\V}\).
\item\label{it:comp-sub} \textbf{Subadditivity.} For all functors \(F:\A\to\Pers\), \(G:\B\to\Pers\) and \(H:\mathcal C\to\Pers\) one has \(\Comp^{\mathrm{nat}}_{J'J}(F,H)\leq\Comp^{\mathrm{nat}}_J(F,G)+\Comp^{\mathrm{nat}}_{J'}(G,H)\) in \([0,\infty]\).
\end{enumerate}
\end{proposition}

\begin{figure}[!htbp]
\centering
\hvpic{}\hspace{1.6em}\(=\)\hspace{1.6em}\ignorespaces
\hvpic{}
\par\vspace{3mm}
\hvpic{}\hspace{2.5em}\ignorespaces
\hvpic{}\space\space
\caption{\Cref{prop:composition}. Top: the universal property of the unit. Bottom left: the units paste to the unit along \(J'J\). Bottom right: the triangle inequality through \(\Lan_{J'}G\).}
\label{fig:composition}
\end{figure}

\begin{proof}
~\begin{steps}
\item \textbf{The Composite Unit.} Put \((L,\eta)\coloneqq(\Lan_JF,\eta^F)\) and \((L',\eta')\coloneqq(\Lan_{J'}L,\eta^L)\), which exist by \itemref{prop:basic-task-changes}{it:forget}, and \(\theta\coloneqq(\eta'J)\circ\eta:F\Rightarrow L'J'J\), with components \(\theta_a=\eta'_{Ja}\circ\eta_a\).
\item \textbf{The Pair \((L',\theta)\) Is a Left Kan Extension along \(J'J\).} Let \(E:\mathcal C\to\V\) and \(\xi:F\Rightarrow EJ'J\). The universal property of \((L,\eta)\) for \(EJ'\) gives a unique \(\xi_1:L\Rightarrow EJ'\) with \((\xi_1J)\circ\eta=\xi\), and that of \((L',\eta')\) for \(E\) a unique \(\xi_2:L'\Rightarrow E\) with \((\xi_2J')\circ\eta'=\xi_1\). Whiskering preserves vertical~composition,~so
\begin{equation}
(\xi_2J'J)\circ\theta\eqby{a}(\xi_2J'J)\circ(\eta'J)\circ\eta\eqby{b}(\xi_1J)\circ\eta\eqby{c}\xi\dd
\end{equation}
Here \rref{a}, \rref{b} and \rref{c} are the definitions of \(\theta\), \(\xi_2\) and \(\xi_1\). If \(\omega:L'\Rightarrow E\) also satisfies \((\omega J'J)\circ\theta=\xi\), then \(\omega_1\coloneqq(\omega J')\circ\eta'\) satisfies \((\omega_1J)\circ\eta=\xi\), so \(\omega_1=\xi_1\), and then \(\omega=\xi_2\). So \((L',\theta)\) and \((\Lan_{J'J}F,\eta^F)\) are left Kan extensions of \(F\) along \(J'J\), which are isomorphic by \Cref{sec:framework}, and this is \labelcref{it:comp-lan}.\tighter
\item \textbf{Subadditivity.} All extensions exist because \(\Pers\) is cocomplete (\Cref{prop:targets}), and
\begin{align}
\Comp^{\mathrm{nat}}_{J'J}(F,H)
&\eqby{a}d_I^{\mathcal C}(\Lan_{J'J}F,H)
\eqby{b}d_I^{\mathcal C}\bigl(\Lan_{J'}(\Lan_JF),H\bigr)\\
&\leby{c}d_I^{\mathcal C}\bigl(\Lan_{J'}(\Lan_JF),\Lan_{J'}G\bigr)+d_I^{\mathcal C}(\Lan_{J'}G,H)\\
&\leby{d}d_I^{\B}(\Lan_JF,G)+d_I^{\mathcal C}(\Lan_{J'}G,H)\\
&\eqby{a}\Comp^{\mathrm{nat}}_J(F,G)+\Comp^{\mathrm{nat}}_{J'}(G,H)\dc
\end{align}
where \rref{a} is \Cref{def:functorial-discrepancy}, \rref{b} combines \labelcref{it:comp-lan} with the invariance under isomorphism of \Cref{lem:diagram-pseudometric}, \rref{c} is the triangle inequality of \Cref{lem:diagram-pseudometric}, and \rref{d} is the first claim of \Cref{thm:lan-interleaving-stability} for \(J'\).\qedhere
\end{steps}
\end{proof}

Since the intermediate \(G\) is arbitrary, every observation of the intermediate task, if one exists,~gives~such~a~bound.

\section{The Algorithmic Transfer Score}\label{sec:score}
On modules of finite type the interleaving distance \(d_I\) \citep[Definitions~3.1 and~3.2, Theorem~3.3]{bubenik2014categorification} can be computed, since there it equals the bottleneck distance of the barcodes \citetext{\citealp[Theorem~4.16]{bubenik2014categorification}, \citealp[Theorem~4.4]{chazal2009proximity}}. The extension remains of finite type, and an algorithm returns the exact value~of~the~discrepancy.

\begin{lemma}\label{lem:finite-type-closure}
Let \(J:\A\to\B\) be a functor between small categories, let \(F:\A\to\Pers\) be a functor, and let \(b\in\Ob(\B)\) be an object whose comma category \(J\downarrow b\) has finitely many objects. For an interval \(I\subseteq\mathbb R\) let \(\K_I\in\Pers\) be the \textbf{interval module} with \((\K_I)_t=\K\) for \(t\in I\), \((\K_I)_t=0\) for \(t\notin I\), and identities as structure maps between scales in \(I\). A module is of \textbf{finite type} if it is isomorphic to a finite direct sum of interval modules \citep[Definition~4.1]{bubenik2014categorification}. If \(F(a)\) is of finite type for every object \((a,\beta)\) of \(J\downarrow b\), then \((\Lan_JF)(b)\) is of finite type. A pointwise finite-dimensional module \(M\in\Pers\) is of finite type if for some finite set \(S\subseteq\mathbb R\) the structure map \(M_{s\leq t}\) is an isomorphism whenever \([s,t]\cap S=\varnothing\).
\end{lemma}

\begin{proof}
~\begin{steps}
\item \textbf{The Critical Set.} For each object \((a,\beta)\) of \(J\downarrow b\) fix a decomposition of \(F(a)\) into finitely many interval modules, and let \(S\) be the set of all real endpoints of these intervals, finite because \(\Ob(J\downarrow b)\) is. Let \(s\leq t\) with \([s,t]\cap S=\varnothing\). An interval with no endpoint in \([s,t]\) contains both \(s\) and \(t\) or neither, so each summand map \((\K_I)_{s\leq t}\) is the identity of \(\K\) or of \(0\), and \(F(a)_{s\leq t}\) is an isomorphism.
\item \textbf{The Transition Map Between Regular Scales.} For such \(s\leq t\) take the data \(R_u\), \(P_u\), \(r_u\) and \(p_u:P_u\to Q_u\) of \eqref{eq:cokernel} for \(u\in\{s,t\}\), and the isomorphisms \(\Psi_{s,t}\coloneqq\bigoplus_\gamma F(a_\gamma)_{s\leq t}:R_s\to R_t\) and \(\Phi_{s,t}\coloneqq\bigoplus_{(a,\beta)}F(a)_{s\leq t}:P_s\to P_t\). By the second step of the proof of \Cref{prop:cokernel}, \(\Phi_{s,t}r_s=r_t\Psi_{s,t}\) and \(\overline\Phi_{s,t}p_s=p_t\Phi_{s,t}\), and \itemref{prop:cokernel}{it:cok-pers} identifies \(\overline\Phi_{s,t}\) with \((\Lan_JF)(b)_{s\leq t}\). As \(\Psi_{s,t}\) is onto, \(\operatorname{im}r_t=r_t\Psi_{s,t}(R_s)=\Phi_{s,t}(\operatorname{im}r_s)\), so \(p_s\Phi_{s,t}^{-1}\) vanishes on \(\operatorname{im}r_t=\ker p_t\) and factors as \(p_s\Phi_{s,t}^{-1}=\overline\Phi{}'_{s,t}p_t\). Then \(\overline\Phi{}'_{s,t}\overline\Phi_{s,t}p_s=\overline\Phi{}'_{s,t}p_t\Phi_{s,t}=p_s\) and \(\overline\Phi_{s,t}\overline\Phi{}'_{s,t}p_t=\overline\Phi_{s,t}p_s\Phi_{s,t}^{-1}=p_t\), with \(p_s\) and \(p_t\) onto, so \((\Lan_JF)(b)_{s\leq t}\) is an isomorphism.
\item \textbf{Finite Type.} Let \(M\) be pointwise finite-dimensional with \(M_{s\leq t}\) invertible whenever \([s,t]\cap S=\varnothing\). By \citet[Theorem~1.1]{crawleyboevey2015decomposition}, \(M\cong\bigoplus_{j\in\Lambda}\K_{I_j}\), and \(M_{s\leq t}\) is invertible iff every summand map is. Suppose some \(I_j\) had a real endpoint \(u\notin S\), and choose \(c>0\) with \([u-c,u+c]\cap S=\varnothing\). If \(u\in I_j\), one of \((\K_{I_j})_{u-c\leq u}\) and \((\K_{I_j})_{u\leq u+c}\) has one scale in \(I_j\) and one outside, a map between \(\K\) and \(0\). If \(u\notin I_j\), then \(I_j\) contains points strictly between \(u\) and one of \(u\pm c\), and the map between \(u\) and such a point is again between \(0\) and \(\K\). Both contradict the hypothesis, so all endpoints lie in \(S\cup\{\pm\infty\}\). There are finitely many such intervals, each with finite multiplicity, bounded by the finite dimension of \(M\) at a point of the interval, so \(\Lambda\) is finite. For \(M\coloneqq(\Lan_JF)(b)\) the hypotheses hold: each \(Q_u\) is a quotient of the finite direct sum \(P_u\) of the finite-dimensional \(F(a)_u\), and the previous step gives the isomorphisms off \(S\).\qedhere
\end{steps}
\end{proof}

Since the bottleneck distance is defined on persistence diagrams, we first fix these. A module \(M\) of finite type determines a finite multiset \(\Dgm(M)\subseteq\{(x,y)\in[-\infty,\infty]^2\mid x<y\}\), one point per interval of a decomposition into interval modules, with \(x\) and \(y\) the endpoints of the interval, intervals with \(x=y\) contributing no point. The multiset depends only on the isomorphism class of \(M\), since the multiplicity of each interval is unique by the Krull--Remak--Schmidt--Azumaya theorem, the endomorphism ring of an interval module being the field \(\K\) \citep[Section~1]{crawleyboevey2015decomposition}. For two such multisets \(\mathcal D_1,\mathcal D_2\) the \textbf{bottleneck distance} \(d_B(\mathcal D_1,\mathcal D_2)\) is the infimum, over partial matchings \(\sigma\) between \(\mathcal D_1\) and \(\mathcal D_2\), of the largest of the costs \(\lVert p-\sigma(p)\rVert_\infty\) of matched points \(p\) and \(\tfrac12(y-x)\) of unmatched points \((x,y)\), the \(\ell^\infty\)-distance to the diagonal, where a bar of length zero would sit. Here \(\lvert\infty-\infty\rvert=\lvert(-\infty)-(-\infty)\rvert\coloneqq0\) and \(\max\varnothing\coloneqq0\), and we put \(d_B([M],[N])\coloneqq d_B(\Dgm M,\Dgm N)\). The score combines the degrees with weights. Fix \(q\in\mathbb N\), weights \(w_0,\ldots,w_q\geq0\), functors \(F_n:\A\to\Pers\) with \(\Lan_JF_n\) existing and modules \(G_n(b)\in\Pers\) for \(b\in\Ob(\B)\), for \(0\leq n\leq q\). The \textbf{weighted transfer score} in a distance \(d\in\{d_I,d_B\}\), where for \(d=d_B\) all values \((\Lan_JF_n)(b)\) and \(G_n(b)\) are required to be of finite type, is
\begin{equation}\label{eq:weighted-score}
\Comp^{d}_{J,q}\bigl((F_n)_n,(G_n)_n\bigr)\coloneqq\sup_{b\in\Ob(\B)}\sum_{n=0}^{q}w_n\,d\bigl([(\Lan_JF_n)(b)],[G_n(b)]\bigr)\in[0,\infty]\dd
\end{equation}
For \(d=d_I\), \(q=0\), \(w_0=1\), pointwise finite-dimensional values and \(G_0(b)=G(b)\) for a functor \(G:\B\to\Pers\), the score is the discrepancy \(\Comp_J(F_0,G)\) of \eqref{eq:discrepancy} in the target~\((\Pers,\Pers^{\mathrm{pfd}},d_I)\).

We compute the invariants from filtered complexes. Let \(\operatorname{Simp}\) be the category of finite simplicial complexes and simplicial maps. A \textbf{finite filtration} is a functor \(K_\bullet:(\mathbb R,\leq)\to\operatorname{Simp}\) whose values are subcomplexes of one finite simplicial complex \(K_\infty\) and whose structure maps \(K_s\to K_t\) for \(s\leq t\) are the inclusions \(K_s\subseteq K_t\). Once a total order on the vertex set of \(K_\infty\) is fixed to orient the simplices, the simplicial homology \(H_n(K_\bullet;\K)\in\Pers\) with the maps induced by the inclusions is its \textbf{persistent homology} in degree \(n\). For \(\sigma\in K_\infty\) put \(t_\sigma\coloneqq\inf\{t\in\mathbb R\mid\sigma\in K_t\}\in[-\infty,\infty]\), with \(\inf\varnothing\coloneqq\infty\), and let \(S(K_\bullet)\coloneqq\{t_\sigma\mid\sigma\in K_\infty\}\cap\mathbb R\) be the finite set of its \textbf{birth scales}. A \textbf{barcode} of a module \(M\) of finite type is a finite multiset of intervals \(I\) with \(M\cong\bigoplus_I\K_I\), and \(\Dgm\) of a barcode is the multiset of the endpoint pairs \((x,y)\) with \(x<y\) of its intervals, so that \(\Dgm\) of a barcode of \(M\) is \(\Dgm(M)\). \Cref{alg:score} in \Cref{app:algorithm} computes the invariants degreewise, builds each comma category, forms \eqref{eq:cokernel} at finitely many scales and reports \eqref{eq:weighted-score} for \(d=d_B\), using two oracles, \Bcfun, which returns a barcode of a module of finite type, and \Bnfun, which returns the bottleneck distance of two finite multisets. The output of \Cref{alg:score} is the discrepancy itself, without approximation.\tighter

\begin{proposition}\label{prop:correctness}
Let \(\K\) be a field, let \(J:\A\to\B\) be a functor between categories with finitely many objects and morphisms, so that every comma category \(J\downarrow b\) is finite, let \(q\in\mathbb N\) and \(w_0,\ldots,w_q\geq0\), and let the input of \Cref{alg:score} consist of finite filtrations \(\mathcal K^{\A}(a)\) for \(a\in\Ob(\A)\), functorial in \(a\), and \(\mathcal K^{\B}(b)\) for \(b\in\Ob(\B)\), so that \(F_n\coloneqq H_n(\mathcal K^{\A}(-);\K):\A\to\Pers\) and \(G_n(b)\coloneqq H_n(\mathcal K^{\B}(b);\K)\) for \(0\leq n\leq q\). If \Bcfun returns a barcode of every module of finite type and \Bnfun returns the bottleneck distance of two finite multisets, then the value \(d_{n,B}(b)\) computed in line~21 of \Cref{alg:score} satisfies
\begin{equation}\label{eq:correctness}
d_{n,B}(b)=d_B\bigl(\Dgm((\Lan_JF_n)(b)),\Dgm(G_n(b))\bigr)=d_I\bigl((\Lan_JF_n)(b),G_n(b)\bigr)
\end{equation}
for every \(n\) and \(b\), and the returned value is \(\Comp^{d_B}_{J,q}=\Comp^{d_I}_{J,q}\) in the notation of \eqref{eq:weighted-score}.
\end{proposition}

\begin{proof}
Fix \(n\leq q\) and \(b\in\Ob(\B)\).
\begin{steps}
\item \textbf{The Computed Object.} The category \(\Pers\) is cocomplete (\Cref{prop:targets}), so \(\Lan_JF_n\) exists and \((\Lan_JF_n)(b)\cong\operatorname*{colim}_{J\downarrow b}F_n\pi_b\) by \itemref{prop:basic-task-changes}{it:forget}. Direct sums in \(\Pers\) are formed scalewise, so lines~16 to~18 of \Cref{alg:score} at a scale \(u\) give the data \(P_u\), \(R_u\) and \(r_u\) of \eqref{eq:cokernel} for \(F_n\) at \(b\). Let \(M\coloneqq\coker r\in\Pers\). Cokernels in \(\Pers\) are formed scalewise, by the first step of the proof of \Cref{prop:targets}, so \(M_u=\coker r_u=Q_u\) at every scale \(u\), with the transition maps induced by \(\bigoplus_{(a,\beta)}F_n(a)_{s\leq t}\), of which line~19 forms the values at the scales named there. \itemref{prop:cokernel}{it:cok-pers} gives isomorphisms \(Q_u\cong(\Lan_JF_n)(b)_u\) that carry these transition maps to the structure maps, so \(M\cong(\Lan_JF_n)(b)\) in \(\Pers\).
\item \textbf{Finite Type.} Let \(K_\bullet\) be a finite filtration with ambient complex \(K_\infty\) and birth scales \(S(K_\bullet)\). If \(s\leq t\) and \([s,t]\cap S(K_\bullet)=\varnothing\), then \(K_s=K_t\), since a simplex \(\sigma\in K_t\setminus K_s\) would have \(s\leq t_\sigma\leq t\), the first inequality because \(K_u\subseteq K_s\) excludes \(\sigma\) for every \(u\leq s\). Every \(H_n(K_t;\K)\) is finite-dimensional, \(K_\infty\) being finite, and \(H_n(K_s;\K)\to H_n(K_t;\K)\) is the identity when \(K_s=K_t\). By the last sentence of \Cref{lem:finite-type-closure} every \(F_n(a)\) and every \(G_n(b)\) is of finite type, and by its first part so is \((\Lan_JF_n)(b)\). The last step of that proof puts every real endpoint of an interval of \(F_n(a)\) into \(S(\mathcal K^{\A}(a))\), so its critical set lies in the finite set \(S_b\coloneqq\bigcup_{(a,\beta)\in\Ob(J\downarrow b)}S(\mathcal K^{\A}(a))\), and its second step makes \(M_{s\leq t}\) an isomorphism whenever \([s,t]\cap S_b=\varnothing\). Choose one scale \(u_C\) in each of the finitely many components \(C\) of \(\mathbb R\setminus S_b\). For \(t\in C\) the structure map of \(M\) between \(t\) and \(u_C\) is an isomorphism, since the closed interval between them lies in \(C\) and so misses \(S_b\). Hence, by functoriality of \(M\), the isomorphism class of \(M\) is determined by the spaces \(Q_u\) at the finitely many scales \(u\in S_b\) and \(u=u_C\), with the transition maps between consecutive ones, which is what line~19 forms.
\item \textbf{Evaluation of the Distance.} By the previous steps and the hypothesis on \Bcfun, line~20 computes \(\mathcal D_1=\Dgm(M)=\Dgm((\Lan_JF_n)(b))\) and \(\mathcal D_2=\Dgm(G_n(b))\), and line~21 returns \(d_B(\mathcal D_1,\mathcal D_2)\). On modules of finite type this is the interleaving distance \citetext{\citealp[Theorem~4.16]{bubenik2014categorification}, \citealp[Theorem~4.4]{chazal2009proximity}}, so \(d_{n,B}(b)\) satisfies \eqref{eq:correctness}.
\item \textbf{Aggregation.} Line~22 accumulates \(\Omega_b=\sum_{n=0}^{q}w_nd_{n,B}(b)\) and line~24 takes the maximum over the finitely many objects of \(\B\), the supremum in \eqref{eq:weighted-score}, so the output is \(\Comp^{d_B}_{J,q}\), which equals \(\Comp^{d_I}_{J,q}\) by \eqref{eq:correctness}.\qedhere
\end{steps}
\end{proof}

A matched pair of an essential and a finite point has cost \(\infty\), and so has an unmatched essential point. The score is \(\infty\) as soon as the numbers of essential points of the two diagrams differ in some degree of positive weight. Conversely, if every point of both diagrams has a finite birth coordinate, as is the case for persistent homology, then \(d_B=\infty\) iff the numbers of essential points differ. Cutting every bar at a fixed scale removes this dichotomy, and the cut commutes with the extension.

\begin{lemma}\label{lem:truncation}
Let \(T\in\mathbb R\). The \textbf{truncation} \(\tr_T:\Pers\to\Pers\) sends a module \(M\) to the module with \((\tr_TM)_t\coloneqq M_t\) for \(t<T\) and \((\tr_TM)_t\coloneqq0\) for \(t\geq T\), whose structure map \((\tr_TM)_{s\leq t}\) is \(M_{s\leq t}\) for \(t<T\) and \(0\) otherwise, and a morphism \(f:M\to N\) to the morphism \(\tr_Tf\) with the components \(f_t\) for \(t<T\) and \(0\) otherwise.
\begin{enumerate}
\item\label{it:tr-colim} \textbf{Colimits.} The truncation \(\tr_T\) is a functor and preserves small colimits.
\item\label{it:tr-lan} \textbf{Extension.} Let \(J:\A\to\B\) be a functor between small categories and \(F:\A\to\Pers\) a functor. The natural transformation \(\sigma:\Lan_J(\tr_T\circ F)\Rightarrow\tr_T\circ\Lan_JF\) with \((\sigma J)\circ\eta^{\tr_T\circ F}=\tr_T\eta^F\) is an isomorphism.
\item\label{it:tr-lip} \textbf{Lipschitz Bound.} For every small category \(\mathcal C\) and all functors \(X,Y:\mathcal C\to\Pers\) one has \(d_I^{\mathcal C}(\tr_T\circ X,\tr_T\circ Y)\leq d_I^{\mathcal C}(X,Y)\), and \(d_I(\tr_TM,\tr_TN)\leq d_I(M,N)\) for all \(M,N\in\Pers\).
\item\label{it:tr-bars} \textbf{Barcodes.} Every interval \(I\subseteq\mathbb R\) satisfies \(\tr_T\K_I\cong\K_{I\cap(-\infty,T)}\). So if \(M\) is of finite type with a barcode \((I_j)_j\), then \(\tr_TM\) is of finite type with the barcode \((I_j\cap(-\infty,T))_j\), in which a bar \([x,y)\) becomes \([x,\min\{y,T\})\), empty for \(x\geq T\).
\end{enumerate}
\end{lemma}

\begin{proof}
~\begin{steps}
\item \textbf{Functoriality.} For \(s\leq t\leq u\) the composite \((\tr_TM)_{t\leq u}\circ(\tr_TM)_{s\leq t}\) is \(M_{t\leq u}\circ M_{s\leq t}=M_{s\leq u}\) if \(u<T\), and otherwise it and \((\tr_TM)_{s\leq u}\) are maps into \((\tr_TM)_u=0\). The components of \(\tr_Tf\) commute with the structure maps, by naturality of \(f\) below \(T\) and as maps into \(0\) from \(T\) on, and \(\tr_T\) preserves identities and composites componentwise.
\item \textbf{Colimits.} Let \(D:\mathcal I\to\Pers\) be a small diagram with colimit cocone \((\lambda^i:D(i)\to L)_i\). Colimits in \(\Pers\) are formed scalewise (\Cref{prop:targets}), so \((\lambda^i_t)_i\) is a colimit cocone in \(\Vect\) for every \(t\), and the \(\tr_T\lambda^i\) form a cocone. Let \((\nu^i:\tr_TD(i)\to W)_i\) be a cocone. For \(t<T\) let \(g_t:L_t\to W_t\) be the unique linear map with \(g_t\lambda^i_t=\nu^i_t\) for all \(i\), and for \(t\geq T\) let \(g_t\coloneqq0\). For \(s\leq t<T\) both \(W_{s\leq t}g_s\) and \(g_tL_{s\leq t}\) composed with \(\lambda^i_s\) give \(\nu^i_tD(i)_{s\leq t}\), and for \(s<T\leq t\) the map \(W_{s\leq t}g_s\) composed with \(\lambda^i_s\) gives \(\nu^i_t(\tr_TD(i))_{s\leq t}=0\). The \(\lambda^i_s\) are jointly surjective, so \(W_{s\leq t}g_s=g_t(\tr_TL)_{s\leq t}\) in both cases, and for \(T\leq s\) both sides start at \(0\). So \(g:\tr_TL\to W\) is a morphism with \(g\circ\tr_T\lambda^i=\nu^i\), unique because its components below \(T\) are and \((\tr_TL)_t=0\) for \(t\geq T\).
\item \textbf{Extension.} Write \(L\coloneqq\Lan_JF\) and \((L',\eta')\coloneqq(\Lan_J(\tr_T\circ F),\eta^{\tr_T\circ F})\), which exist because \(\Pers\) is cocomplete (\Cref{prop:targets}), and let \(\sigma\) be given by the universal property of \((L',\eta')\) for \(E=\tr_T\circ L\) and \(\xi=\tr_T\eta^F\). Fix \(b\in\Ob(\B)\). By \itemref{prop:basic-task-changes}{it:forget} the legs \(\lambda'^{(a,\beta)}=L'(\beta)\circ\eta'_a\) form a colimit cocone of \(\tr_T\circ F\pi_b\), and the legs \(\lambda^{(a,\beta)}=L(\beta)\circ\eta^F_a\) one of \(F\pi_b\), which the previous step carries to the colimit cocone \((\tr_T\lambda^{(a,\beta)})\) of \(\tr_T\circ F\pi_b\). For every~object~\((a,\beta)\)~of~\(J\downarrow b\),
\[\sigma_b\circ\lambda'^{(a,\beta)}\eqby{a}\tr_TL(\beta)\circ\sigma_{Ja}\circ\eta'_a\eqby{b}\tr_TL(\beta)\circ\tr_T\eta^F_a=\tr_T\lambda^{(a,\beta)}\dc\]
where \rref{a} is the naturality of \(\sigma\) at \(\beta\) and \rref{b} the defining identity of \(\sigma\). So \(\sigma_b\) mediates between two colimit cocones of one diagram and is an isomorphism, for every \(b\).
\item \textbf{Lipschitz Bound.} Let \((\varphi,\psi)\) be an \(\varepsilon\)-interleaving of \(X\) and \(Y\), and let \((\varphi'_c)_t\coloneqq(\varphi_c)_t\) if \(t+\varepsilon<T\) and \((\varphi'_c)_t\coloneqq0\) otherwise, and likewise \(\psi'\). Where \(t+\varepsilon<T\), the naturality of \(\varphi'\) in \(t\) and in \(c\) at \(t\) is that of \(\varphi\), and otherwise both sides are maps into \((\Sigma_\varepsilon\tr_TY(c))_t=0\). In the interleaving identity at \(t\) both sides are maps into \((\tr_TX(c))_{t+2\varepsilon}=0\) if \(t+2\varepsilon\geq T\), and otherwise all scales involved lie below \(T\), where \(\varphi'\) and \(\psi'\) agree with \(\varphi\) and \(\psi\). So \((\varphi',\psi')\) is an \(\varepsilon\)-interleaving of \(\tr_T\circ X\) and \(\tr_T\circ Y\), the infimum gives the first bound, and \Cref{lem:diagram-pseudometric} gives the second as the case \(\mathcal C=\mathbf 1\).
\item \textbf{Barcodes.} By the definition, \(\tr_T\K_I\) is \(\K\) at the scales of \(I\cap(-\infty,T)\) and \(0\) elsewhere, with identities between scales of \(I\cap(-\infty,T)\), so it is \(\K_{I\cap(-\infty,T)}\), the zero module when the intersection is empty. As a functor \(\tr_T\) preserves isomorphisms, and it preserves finite direct sums by the second step, so \(M\cong\bigoplus_j\K_{I_j}\) gives \(\tr_TM\cong\bigoplus_j\K_{I_j\cap(-\infty,T)}\), and \([x,y)\cap(-\infty,T)=[x,\min\{y,T\})\).\qedhere
\end{steps}
\end{proof}

The truncated score used in \Cref{sec:experiments} is applied to the extension of persistent homology. It cuts every bar \([x,y)\) of the frozen barcodes to \([x,\min\{y,T\})\), drops the empty bars and aggregates as in \eqref{eq:weighted-score}. By \labelcref{it:tr-bars} the cut barcodes are barcodes of \(\tr_T((\Lan_JF_n)(b))\) and \(\tr_TG_n(b)\), by \labelcref{it:tr-lan} the first module is \((\Lan_J(\tr_T\circ F_n))(b)\), and on modules of finite type the bottleneck distance equals the interleaving distance, as in the proof of \Cref{prop:correctness}. The truncated score is thus \(\Comp^{d_I}_{J,q}\bigl((\tr_T\circ F_n)_n,(\tr_TG_n)_n\bigr)\), the weighted discrepancy of the truncated invariants. In particular, \Cref{thm:lan-interleaving-stability} and, for a merge into \(\mathbf 1\), \Cref{cor:terminal-weighted-stability} apply to the truncated invariants, and by \labelcref{it:tr-lip} their distances are at most those before truncation. The truncated modules vanish from \(T\) on and, like frozen modules and their extensions, below \(0\). For \(T\geq0\) any two of them are \(T/2\)-interleaved by zero morphisms, so the truncated score is at most \(T/2\) per degree.

In \Cref{ex:arcs} the degree-one diagram of the gluing has no essential point, while that of \(P\) has one, so the degree-one score of the gluing is \(\infty\). Here the deviation is caused by a loss in the reference and not in the data. Such losses can be computed from the source~alone~(\Cref{sec:gluing}).\keeptail

\section{Invariants of Glued Data}\label{sec:gluing}
Homology does not preserve colimits. Two arcs glued along two overlapping arcs form a circle whose loop is carried by neither arc, so the extension of persistent homology along a gluing can miss classes, as in \Cref{ex:arcs}. Extending chains before taking homology removes this loss. We freeze the filtration at a scale \(T\), so that a class that is alive at \(T\) counts as essential. The choice of scale matters. If the filtration is frozen at a scale that is at least the diameter of the union, the extension of a merge of \(k\geq2\) nonempty clouds has \(k\) essential components, whereas the union has one, so the merge has discrepancy \(\infty\) however far apart the clouds lie. If it is frozen at a scale \(T\) below the distance between any two of the clouds, the merge has discrepancy \(0\) (\Cref{prop:separated-vr-merge}).

\begin{definition}\label{def:frozen-vr}
Let \((P,d)\) be a finite metric space with a fixed total order on \(P\), and let \(T\in\mathbb R\). For \(t\in\mathbb R\) and \(X\subseteq P\) the \textbf{Vietoris--Rips complex} \(\VR_t(X)\) is the simplicial complex whose simplices are the nonempty subsets \(\sigma\subseteq X\) with \(d(x,y)\leq t\) for all \(x,y\in\sigma\), so that \(\VR_t(X)=\varnothing\) for \(t<0\) because \(d(x,x)=0\). The inclusions \(\VR_t(X)\subseteq\VR_{t'}(X)\) for \(t\leq t'\) make \(\VR_\bullet(X)\) a functor \((\mathbb R,\leq)\to\operatorname{Simp}\), the \textbf{Vietoris--Rips filtration} of \(X\). The \textbf{frozen} filtration of \(X\) at \(T\) is \(\VR^T_\bullet(X)\) with \(\VR^T_t(X)\coloneqq\VR_{\min\{t,T\}}(X)\), which agrees with \(\VR_\bullet(X)\) below \(T\) and is constant from \(T\) on. For a simplicial complex \(K\) with vertices in \(P\) let \(C_\bullet(K;\K)\in\Ch\) be its \textbf{oriented chain complex}: \(C_k(K;\K)\) has the \(k\)-simplices of \(K\) as a basis, the differential sends a \(k\)-simplex \(\{x_0<\cdots<x_k\}\) to \(\sum_{i=0}^{k}(-1)^i\{x_0<\cdots<\widehat{x_i}<\cdots<x_k\}\), where the hat omits the vertex, and the degrees \(k<0\) are zero. An inclusion of complexes induces the inclusion of bases, a chain map. The \textbf{frozen persistent homology} of \(X\) in degree \(n\in\mathbb N\) is the persistence module \(\PH^T_n(X)\coloneqq H_n(C_\bullet(\VR^T_\bullet(X);\K))\in\Pers\), with the structure maps induced by the inclusions, and we put \(\PH^T_n(X)\coloneqq0\) for \(n<0\). Every class that is still alive at the scale \(T\) yields an essential interval, that is one whose death coordinate equals \(+\infty\), since the filtration is constant from \(T\) on, and there are \(\dim_\K\PH^T_n(X)_T\) of them.
\end{definition}

\begin{figure}[!htbp]
\centering
\hvpic{}\space\space
\caption{{\addfontfeatures{LetterSpace=-1.5}The frozen filtration on nine points at \(T=1.25\), with the unfrozen complex \(\VR_{t_3}(P)\) dashed. The loop born at \(0.77\) becomes essential, although it dies at \(1.72\) without freezing. Against \(P'=0.7\,P\), whose loop dies at \(1.20<T\), the frozen distance in degree one is \(\infty\), the truncated one \(0.23\).}}
\label{fig:frozen-vr}
\end{figure}

Gluing chains along a cover can only recover the chains of the whole if every simplex of the whole lies in one of the sets. For Vietoris--Rips complexes this requirement becomes a condition on the scale, and the associated \v{C}ech category is the source task of a gluing.

\begin{definition}\label{def:lebesgue-cover}
For nonempty finite subsets \(X,Y\) of a metric space \((Z,d)\) put \(\dist(X,Y)\coloneqq\min\{d(x,y)\mid x\in X,\ y\in Y\}\), so that \(\dist(\{x\},Y)=\min_{y\in Y}d(x,y)\) is the distance from a point to a set. Let \((P,d)\) be a finite metric space, let \(T\in\mathbb R\) and \(m\in\mathbb N\), and let \(U_1,\ldots,U_m\subseteq P\), with \(U_{ij}\coloneqq U_i\cap U_j\). The family is a \textbf{\(T\)-Lebesgue cover} of \(P\), or satisfies the \textbf{\(T\)-Lebesgue condition}, if every nonempty \(\sigma\subseteq P\) with \(d(x,y)\leq T\) for all \(x,y\in\sigma\) is contained in some \(U_i\). The \textbf{\v{C}ech category} \(\A_m\) has \(\Ob(\A_m)=\{u_i\mid1\leq i\leq m\}\cup\{u_{ij}\mid1\leq i<j\leq m\}\) and, besides the identities, only the morphisms \(\ell^{ij}_i:u_{ij}\to u_i\) and \(\ell^{ij}_j:u_{ij}\to u_j\), and \(J_m:\A_m\to\mathbf 1\) is the unique functor.
\end{definition}

Since singletons have diameter zero, a \(T\)-Lebesgue cover covers \(P\) for \(T\geq0\), while for \(T<0\) every family is a \(T\)-Lebesgue cover. For \(m=2\) the category \(\A_2\) is the span
\[u_1\xleftarrow{\ell_1}u_{12}\xrightarrow{\ell_2}u_2\]
of \Cref{ex:arcs}, with \(\ell_1\coloneqq\ell^{12}_1\) and \(\ell_2\coloneqq\ell^{12}_2\). Under the \(T\)-Lebesgue condition the complexes of the sets cover the complex of the whole. For two sets it is a statement about distances, and collars produce it (\Cref{fig:lebesgue}).

\begin{lemma}\label{lem:lebesgue}
Let \((P,d)\) be a finite metric space, let \(T\in\mathbb R\) and let \(U_1,\ldots,U_m\subseteq P\).
\begin{enumerate}
\item\label{it:leb-cap} \textbf{Intersection.} For all \(t\in\mathbb R\) and all \(i,j\) one has \(\VR^T_t(U_i)\cap\VR^T_t(U_j)=\VR^T_t(U_{ij})\).
\item\label{it:leb-union} \textbf{Union.} {\scriptspace=0pt {\addfontfeatures{LetterSpace=-1}If \(U_1,\ldots,U_m\) form a \(T\)-Lebesgue cover of \(P\), then \(\VR^T_t(P)=\bigcup_{i=1}^{m}\VR^T_t(U_i)\) for all \(t\in\mathbb R\).}}\tighter
\item\label{it:leb-two} \textbf{Two Sets.} If \(m=2\) and \(U_1\cup U_2=P\), then \(U_1,U_2\) form a \(T\)-Lebesgue cover of \(P\) iff \(d(x,y)>T\) for all \(x\in U_1\setminus U_2\) and \(y\in U_2\setminus U_1\).
\item\label{it:leb-collar} \textbf{Collars.} If \(P_1,\ldots,P_m\subseteq P\) are nonempty with \(P=P_1\cup\cdots\cup P_m\) and \(U_i\coloneqq\{x\in P\mid\dist(\{x\},P_i)\leq T\}\), then \(U_1,\ldots,U_m\) form a \(T\)-Lebesgue cover of \(P\).
\end{enumerate}
\end{lemma}

\begin{figure}[!htbp]
\centering
\hvpic{}\space\space
\caption{\Cref{lem:lebesgue}. Top: with the wide overlap the complements lie \(1.07>T=0.7\) apart and every edge of \(\VR_T(P)\) lies in one of the sets. With the narrow one, two complement points lie \(0.63\leq T\) apart, and the dashed edges lie in neither. Bottom: the \v{C}ech category \(\A_3\), and the collars of a partition into thirds, a \(T\)-Lebesgue cover by three sets.}
\label{fig:lebesgue}
\end{figure}

\begin{proof}
Let \(t\in\mathbb R\) and \(u\coloneqq\min\{t,T\}\), so that \(\VR^T_t(X)=\VR_u(X)\) consists of the nonempty \(\sigma\subseteq X\) of diameter at most \(u\), computed in \(P\).
\begin{steps}
\item \textbf{Intersection.} A simplex lies in \(\VR_u(U_i)\cap\VR_u(U_j)\) iff it is contained in \(U_i\) and in \(U_j\) and has diameter at most \(u\), that is iff it lies in \(\VR_u(U_{ij})\).
\item \textbf{Union.} Each \(\VR_u(U_i)\) lies in \(\VR_u(P)\). Conversely a simplex of \(\VR_u(P)\) has diameter at most \(u\leq T\), so it lies in some \(U_i\) by the Lebesgue property, hence in \(\VR_u(U_i)\).
\item \textbf{Two Sets.} Let \(U_1\cup U_2=P\). Suppose \(d(x,y)>T\) for all \(x\in U_1\setminus U_2\), \(y\in U_2\setminus U_1\), and let \(\sigma\neq\varnothing\) have diameter at most \(T\) and \(\sigma\not\subseteq U_1\). A vertex \(y\notin U_1\) of \(\sigma\) lies in \(U_2\setminus U_1\), so no vertex of \(\sigma\) lies in \(U_1\setminus U_2\), and \(\sigma\subseteq U_2\). Conversely, for a \(T\)-Lebesgue cover, \(x\in U_1\setminus U_2\) and \(y\in U_2\setminus U_1\), the set \(\{x,y\}\) lies in neither set, so its diameter \(d(x,y)\) exceeds \(T\).
\item \textbf{Collars.} Let \(\sigma\neq\varnothing\) have diameter at most \(T\), choose \(x\in\sigma\) and a part \(P_i\ni x\). Every \(y\in\sigma\) has \(\dist(\{y\},P_i)\leq d(y,x)\leq T\), so \(\sigma\subseteq U_i\).\qedhere
\end{steps}
\end{proof}

Subsets \(U_1,\ldots,U_m\subseteq P\) define the two source diagrams of a gluing, one of filtered chain complexes and one of persistence modules:
\begin{equation}\label{eq:chain-diagram}
F^T_\bullet:\A_m\to\Fun{(\mathbb R,\leq)}{\Ch},\qquad
\begin{aligned}
u_i&\mapsto C_\bullet(\VR^T_\bullet(U_i);\K), & \ell^{ij}_i&\mapsto C_\bullet(U_{ij}\subseteq U_i),\\
u_{ij}&\mapsto C_\bullet(\VR^T_\bullet(U_{ij});\K), & \ell^{ij}_j&\mapsto C_\bullet(U_{ij}\subseteq U_j),
\end{aligned}
\end{equation}
where \(C_\bullet(U\subseteq U')\) is the natural transformation whose component at \(t\) is induced by the inclusion \(\VR^T_t(U)\subseteq\VR^T_t(U')\), and \(F^T_n\coloneqq H_n\circ F^T_\bullet:\A_m\to\Pers\) for \(n\in\mathbb Z\), so that \(F^T_n(u_i)=\PH^T_n(U_i)\) and \(F^T_n(u_{ij})=\PH^T_n(U_{ij})\), both zero for \(n<0\). Both are functors, since \(\A_m\) has no composable pair of non-identity morphisms. Colimits in \(\Ch\) are formed degreewise and the \(C_\bullet(K;\K)\) vanish in negative degrees, so the extensions of \eqref{eq:chain-diagram} among \(\mathbb N\)-graded complexes and in \(\Ch\) agree in degrees \(k\geq0\), vanish below, and have the same homology in degrees \(n\in\mathbb N\). Gluing chains along a \(T\)-Lebesgue cover loses nothing.

\begin{proposition}\label{prop:chain-gluing}
Let \(U_1,\ldots,U_m\) be a \(T\)-Lebesgue cover of a finite metric space \((P,d)\) with a fixed total order. Then \((\Lan_{J_m}F^T_\bullet)(\bullet)\cong C_\bullet(\VR^T_\bullet(P);\K)\) in \(\Fun{(\mathbb R,\leq)}{\Ch}\), by the morphism induced by the inclusions \(U_i\subseteq P\) and \(U_{ij}\subseteq P\), and so \(H_n\bigl((\Lan_{J_m}F^T_\bullet)(\bullet)\bigr)\cong\PH^T_n(P)\) in~\(\Pers\)~for~every~\(n\in\mathbb N\).
\end{proposition}

\begin{proof}
For \(a\in\Ob(\A_m)\) write \(U_a\) for \(U_i\) if \(a=u_i\) and for \(U_{ij}\) if \(a=u_{ij}\), put \(K^a_t\coloneqq\VR^T_t(U_a)\) and \(K_t\coloneqq\VR^T_t(P)\), and let \(\mathrm S_k(K)\) be the set of \(k\)-simplices of a simplicial complex \(K\), so that \(C_k(K;\K)=\K[\mathrm S_k(K)]\) has basis \(\mathrm S_k(K)\).
\begin{steps}
\item \textbf{Reduction to a Colimit of Sets.} The condition over \(\mathbf 1\) is vacuous, so \(J_m\downarrow\bullet\cong\A_m\) with \(\pi_\bullet\) the identity, and colimits in \(\Fun{(\mathbb R,\leq)}{\Ch}\) are formed at each scale and in each degree, by the argument of \Cref{prop:targets} for \(\Pers\). By \itemref{prop:basic-task-changes}{it:forget} the value of \((\Lan_{J_m}F^T_\bullet)(\bullet)\) at \(t\) in degree \(k\) is the colimit over \(\A_m\) of \(a\mapsto\K[\mathrm S_k(K^a_t)]\), that is \(\K[-]\) applied to the diagram \(a\mapsto\mathrm S_k(K^a_t)\) of inclusions of sets, and \(\K[-]:\mathbf{Set}\to\Vect\), \(S\mapsto\K[S]\), \((f:S\to S')\mapsto\) the linear map sending the basis vector \(s\) to the basis vector \(f(s)\), is left adjoint to the forgetful functor, so it preserves colimits \citep[Chapter~V]{maclane1998categories}.
\item\lineskiplimit=0pt\relax \textbf{The Colimit of Sets Is the Union.} Fix \(t\) and \(k\). The inclusions \(\mathrm S_k(K^a_t)\subseteq\mathrm S_k(K_t)\) form a cocone. Let \((f_a:\mathrm S_k(K^a_t)\to X)_a\) be a cocone, so \(f_{u_i}(\sigma)=f_{u_{ij}}(\sigma)=f_{u_j}(\sigma)\) for \(\sigma\in\mathrm S_k(K^{u_{ij}}_t)\). By \itemref{lem:lebesgue}{it:leb-union} every \(\sigma\in\mathrm S_k(K_t)\) lies in some \(K^{u_i}_t\), and we put \(f(\sigma)\coloneqq f_{u_i}(\sigma)\). If also \(\sigma\in K^{u_j}_t\) with \(i<j\), the case \(j<i\) being symmetric, then \(\sigma\in K^{u_{ij}}_t\) by \itemref{lem:lebesgue}{it:leb-cap}, so \(f_{u_i}(\sigma)=f_{u_j}(\sigma)\) and \(f\) is well defined. It restricts to \(f_{u_i}\) by construction and to \(f_{u_{ij}}\) by the cocone equation, and it is unique because the \(\mathrm S_k(K^{u_i}_t)\) cover \(\mathrm S_k(K_t)\).
\item\lineskiplimit=-0.4\baselineskip\relax \textbf{Assembly.} The chain maps induced by the inclusions \(U_a\subseteq P\) form a cocone on \(F^T_\bullet\), inclusions commuting with differentials and with \(\VR^T_s\subseteq\VR^T_t\). The induced \(c:(\Lan_{J_m}F^T_\bullet)(\bullet)\to C_\bullet(\VR^T_\bullet(P);\K)\) is at each scale and degree \(\K[-]\) of the bijection of the previous step, hence an isomorphism, and so is \(H_n(c)\), homology being formed at each scale.\qedhere
\end{steps}
\end{proof}

The proof uses neither the finiteness of \(P\) nor the fact that \(\K\) is a field. The \(T\)-Lebesgue condition is also needed only degreewise. If every simplex of \(\VR_T(P)\) of dimension at most \(q+1\) lies in some \(U_i\), the extension agrees with \(C_\bullet(\VR^T_\bullet(P);\K)\) in degrees at most \(q+1\), so its homology is \(\PH^T_n(P)\) for \(n\leq q\).

The chain-level source invariant is almost the data itself. Its advantage lies in exactness rather than in compression. Along a gluing it loses nothing, and it costs a small multiple of the persistence of the union (\ref{hyp:cost}). In \Cref{ex:arcs} the extension of chains along the span is the chain complex of the whole circle, including its loop (\Cref{fig:gluing}). If homology is extended instead, the loss is measured by a single kernel. For \(m=2\) and \(n\in\mathbb Z\) put
\begin{equation}\label{eq:mu}
\mu_n\coloneqq\bigl(F^T_n(\ell_1),-F^T_n(\ell_2)\bigr):\PH^T_n(U_{12})\to\PH^T_n(U_1)\oplus\PH^T_n(U_2)\dc
\end{equation}
so that \(\ker\mu_{-1}=0\) by the convention \(\PH^T_{-1}\coloneqq0\), and let \(\chi_n:(\Lan_{J_2}F^T_n)(\bullet)\to\PH^T_n(P)\) be the morphism induced by the inclusions of \(U_1\), \(U_{12}\) and \(U_2\) into \(P\), the unique one with \(\chi_n\circ\lambda^{(u_i,\operatorname{id}_\bullet)}=\PH^T_n(U_i\subseteq P)\) for \(i=1,2\). The distance statements rest on a comparison of dimensions for modules whose structure maps are isomorphisms just above every scale, and frozen modules have this property.

\begin{lemma}\label{lem:regular-modules}
Let \(M,N\in\Pers\), and let \(\dim_\K\) denote the dimension of a vector space, a cardinal number, possibly infinite.
\begin{enumerate}
\item\label{it:reg-dim} \textbf{Equal Dimensions.} If every \(t\in\mathbb R\) has some \(t'>t\) such that \(M_{t\leq s}\) and \(N_{t\leq s}\) are isomorphisms for all \(s\in[t,t')\), and if \(d_I(M,N)=0\), then \(\dim_\K M_t=\dim_\K N_t\) for every \(t\in\mathbb R\).
\item\label{it:reg-inf} \textbf{Infinite Distance.} If for some \(T\in\mathbb R\) all structure maps \(M_{s\leq t}\) and \(N_{s\leq t}\) with \(T\leq s\leq t\) are isomorphisms and \(\dim_\K M_T\neq\dim_\K N_T\), then \(d_I(M,N)=\infty\).
\item\label{it:reg-frozen} \textbf{Frozen Modules.} Let \((P,d)\) be a finite metric space with a fixed total order, \(T\in\mathbb R\), \(X\subseteq P\) and \(n\in\mathbb N\). Then \(\PH^T_n(X)_t\) is finite-dimensional for every \(t\) and zero for \(t<0\), every \(t\in\mathbb R\) has some \(t'>t\) with \(\VR^T_s(X)=\VR^T_t(X)\) for all \(s\in[t,t')\), and \(\VR^T_s(X)=\VR^T_t(X)\) for all \(T\leq s\leq t\), so that the structure maps \(\PH^T_n(X)_{t\leq s}\) and \(\PH^T_n(X)_{s\leq t}\) at these scales are identities.\tighter
\end{enumerate}
\end{lemma}

\begin{proof}
~\begin{steps}
\item \textbf{An Interleaving Injects.} Let \((\varphi,\psi)\) be an \(\varepsilon\)-interleaving and \(M_{t\leq t+2\varepsilon}\) an isomorphism. At \(t\), \(\Sigma_\varepsilon\psi\circ\varphi=\tau^M_{2\varepsilon}\) reads \(\psi_{t+\varepsilon}\circ\varphi_t=M_{t\leq t+2\varepsilon}\), so \(\varphi_t:M_t\to N_{t+\varepsilon}\) is injective and \(\dim_\K M_t\leq\dim_\K N_{t+\varepsilon}\). The same holds with \(M\) and \(N\) exchanged.
\item \textbf{Vanishing Distance.} For \(d_I(M,N)=0\) and \(t,t'\) as in \labelcref{it:reg-dim}, the infimum \eqref{eq:interleaving} gives an \(\varepsilon\)-interleaving with \(t+2\varepsilon<t'\), so \(M_{t\leq t+2\varepsilon}\) and \(N_{t\leq t+\varepsilon}\) are isomorphisms, and the previous step gives \(\dim_\K M_t\leq\dim_\K N_{t+\varepsilon}=\dim_\K N_t\), and symmetrically.
\item \textbf{Different Dimensions at the Freezing Scale.} If an \(\varepsilon\)-interleaving existed, the first step at \(t=T\) would give \(\dim_\K M_T\leq\dim_\K N_{T+\varepsilon}=\dim_\K N_T\) and symmetrically, against the hypothesis. So no \(\varepsilon\) is admitted, and \(d_I(M,N)=\infty\).
\item \textbf{Frozen Modules.} For \(t<0\) the complex \(\VR^T_t(X)\) is empty, and \(X\) is finite, so every \(\PH^T_n(X)_t\) is finite-dimensional. Let \(D_X\) be the finite set of the distances \(d(x,y)\), \(x,y\in X\). For \(0\leq u\leq u'\) with no element of \(D_X\) in \((u,u']\), \(\VR_u(X)=\VR_{u'}(X)\), since membership of a simplex depends only on its pairwise distances. Put \(t'\coloneqq0\) if \(t<0\), \(t'\coloneqq\min(\{v\in D_X\mid v>t\}\cup\{T\})\) if \(0\leq t<T\), and \(t'\coloneqq t+1\) if \(t\geq\max\{T,0\}\). Then \(t'>t\), and for \(s\in[t,t')\) the complexes \(\VR^T_s(X)\) and \(\VR^T_t(X)\) coincide: both are empty in the first case, in the second \(\min\{s,T\}=s\) and no element of \(D_X\) lies in \((t,s]\), and in the third both equal \(\VR_T(X)\). The structure maps are induced by inclusions of equal complexes, hence identities, and likewise for \(T\leq s\leq t\).\qedhere
\end{steps}
\end{proof}

Along the span the extension of homology is \(\coker\mu_n\), and its loss is \(\ker\mu_{n-1}\), the classes of degree \(n-1\) of the overlap that die in both parts. In \Cref{ex:arcs}, let each arc be connected at \(T\) and let the overlap have two components there, as in \Cref{fig:gluing}. The extension only glues homology classes. In each arc it identifies the two components of the overlap, but it does not record any path between them through the other arc. Hence \((\ker\mu_0)_T\) is spanned by the difference of the classes of the two components, and \(\delta_1\) maps \(\PH^T_1(P)_T\cong\K\), the loop, isomorphically onto it. The exact sequence itself is the Mayer--Vietoris sequence of persistence modules \citep{difabio2011mayer,lipsky2011spectral,yoon2020persistence}. Its interpretation through the extension is new, and so are the distance statements.

\begin{theorem}\label{thm:mayer-vietoris}
Let \(U_1,U_2\) be a \(T\)-Lebesgue cover of a finite metric space \((P,d)\) with a fixed total order, and let \(n\in\mathbb N\).
\begin{enumerate}
\item\label{it:mv-exact} \textbf{Exact Sequence.} There is a morphism \(\delta_n:\PH^T_n(P)\to\ker\mu_{n-1}\) of persistence modules such that the following sequence is exact in \(\Pers\):
\begin{equation}\label{eq:mv-ses}
0\to(\Lan_{J_2}F^T_n)(\bullet)\xrightarrow{\chi_n}\PH^T_n(P)\xrightarrow{\delta_n}\ker\mu_{n-1}\to0\dd
\end{equation}
\item\label{it:mv-iso} \textbf{Criterion.} The morphism \(\chi_n\) is an isomorphism iff \(\ker\mu_{n-1}=0\), and \((\Lan_{J_2}F^T_n)(\bullet)\cong\PH^T_n(P)\) iff \(\ker\mu_{n-1}=0\). For \(n=0\) both hold.
\item\label{it:mv-dist} \textbf{Distance.} The distance \(d_I\bigl((\Lan_{J_2}F^T_n)(\bullet),\PH^T_n(P)\bigr)\) vanishes if \(\ker\mu_{n-1}=0\), is positive otherwise, and is infinite iff \((\ker\mu_{n-1})_T\neq0\).
\item\label{it:mv-betti} \textbf{Essential Signature.} With \(h_k(X)\coloneqq\dim_\K H_k(C_\bullet(\VR_T(X);\K))\) for \(X\subseteq P\) and \(e_k\coloneqq\dim_\K\bigl((\Lan_{J_2}F^T_k)(\bullet)\bigr)_T\), one has \(e_0=h_0(P)\) and \(e_{k+1}=h_{k+1}(P)-h_k(U_{12})+h_k(U_1)+h_k(U_2)-e_k\) in \(\mathbb Z\) for every \(k\in\mathbb N\).
\end{enumerate}
\end{theorem}

\begin{figure}[!htbp]
\centering
\hvpic{}\space\space
\caption{\Cref{ex:arcs} computed on \(36\) points at \(T=0.65\), with filled, hollow and square markers for \(U_1\setminus U_2\), \(U_2\setminus U_1\) and \(U_{12}\). Top left: the complements lie \(1.41>T\) apart, so the cover is \(T\)-Lebesgue. Top right: gluing chains gives the circle with \(H_1=\K\), while gluing homology gives \(0\) in degree one. Bottom: \eqref{eq:mv-ses} for \(n=1\), the lost loop reappearing in \(\ker\mu_0\).}
\label{fig:gluing}
\end{figure}

\begin{figure}[!htbp]
\centering
\hvpic{}\space\space
\caption{The proof of \Cref{thm:mayer-vietoris} at \(t=0.25\). Left: every simplex lies in \(K^1_t\) (solid) or \(K^2_t\) (dashed), which gives the short exact sequence at the bottom. Middle: the loop \(c=y+z\) with \(y\) in \(K^1_t\), \(z\) in \(K^2_t\) and \(\partial y=q-p\), so \(\delta_1[c]=[q]-[p]\in\ker\mu_0\). Right: the recursion of \itemref{thm:mayer-vietoris}{it:mv-betti}.}
\label{fig:mv-proof}
\end{figure}

\begin{proof}
Write \(K_t\coloneqq\VR^T_t(P)\), \(K^i_t\coloneqq\VR^T_t(U_i)\), \(K^{12}_t\coloneqq\VR^T_t(U_{12})\) and \(C_\bullet(-)\) for \(C_\bullet(-;\K)\). By \itemref{lem:lebesgue}{it:leb-cap} and \itemref{lem:lebesgue}{it:leb-union},
\begin{equation}\label{eq:mv-union}
K_t=K^1_t\cup K^2_t,\qquad K^1_t\cap K^2_t=K^{12}_t\qquad(t\in\mathbb R)\dc
\end{equation}
and the first five steps use the cover only through \eqref{eq:mv-union}.
\begin{steps}
\item \textbf{The Extension Is a Cokernel.} As in the proof of \Cref{prop:chain-gluing}, \(J_2\downarrow\bullet\cong\A_2\), so by \eqref{eq:pointwise} the module \((\Lan_{J_2}F^T_n)(\bullet)\) is the pushout of \(F^T_n(\ell_1)\) and \(F^T_n(\ell_2)\), with legs \(\lambda^{(u_i,\operatorname{id}_\bullet)}\). Let \(\iota_1,\iota_2\) be the summand inclusions into \(\PH^T_n(U_1)\oplus\PH^T_n(U_2)\) and \(p\) the quotient onto \(\coker\mu_n\). Morphisms \(f_i:\PH^T_n(U_i)\to W\) determine a unique \(f\) with \(f\circ\iota_i=f_i\), and
\begin{equation}\label{eq:mv-pushout-chain}
\begin{aligned}
f_1\circ F^T_n(\ell_1)=f_2\circ F^T_n(\ell_2)&\iff f\circ\mu_n=0\\
&\iff f=\bar f\circ p\text{ for a unique }\bar f:\coker\mu_n\to W\dc
\end{aligned}
\end{equation}
the first equivalence by \eqref{eq:mu}, the second because \(p\) is an epimorphism with kernel \(\operatorname{im}\mu_n\). With \(f=p\) the pair \((p\iota_1,p\iota_2)\) is a cocone, and for every cocone \((f_1,f_2)\) the mediator \(\bar f\) is unique because \(\bar f\circ p\circ\iota_i=f_i\) determines \(\bar f\circ p\). So \((\Lan_{J_2}F^T_n)(\bullet)\cong\coker\mu_n\) with \(\lambda^{(u_i,\operatorname{id}_\bullet)}\) corresponding to \(p\circ\iota_i\). The inclusions of \(U_1\) and \(U_2\) into \(P\) restrict to the same inclusion of \(U_{12}\), so the morphism \(\nabla_n\) with \(\nabla_n\circ\iota_i=\PH^T_n(U_i\subseteq P)\) satisfies \(\nabla_n\circ\mu_n=0\), and under the isomorphism, \(\chi_n\) is the unique \(\bar\nabla_n\) with \(\bar\nabla_n\circ p=\nabla_n\).
\item \textbf{A Short Exact Sequence of Filtered Chain Complexes.} The chain maps \(a_t:C_\bullet(K^{12}_t)\to C_\bullet(K^1_t)\oplus C_\bullet(K^2_t)\), \(x\mapsto(x,-x)\), and \(b_t:C_\bullet(K^1_t)\oplus C_\bullet(K^2_t)\to C_\bullet(K_t)\), \((y,z)\mapsto y+z\), chains of subcomplexes read in the larger complex, commute with the inclusions for \(s\leq t\) and form an exact sequence \(0\to C_\bullet(K^{12}_t)\to C_\bullet(K^1_t)\oplus C_\bullet(K^2_t)\to C_\bullet(K_t)\to0\): \(a_t\) is injective, \(b_ta_t=0\), \(y+z=0\) gives \(y=-z\in C_\bullet(K^1_t)\cap C_\bullet(K^2_t)\), the span of the common basis elements, which is \(C_\bullet(K^{12}_t)\) by \eqref{eq:mv-union}, so \((y,z)=a_t(y)\), and \(b_t\) is onto because every simplex of \(K_t\) lies in \(K^1_t\) or \(K^2_t\) by \eqref{eq:mv-union}.
\item \textbf{The Sequence at One Scale.} Fix \(t\). The long exact homology sequence of this short exact sequence \citep[Theorem~1.3.1]{weibel1994introduction} contains the exact segment
\[
\hvpic
\]
where the second term is \(H_n(C_\bullet(K^1_t)\oplus C_\bullet(K^2_t))\), homology commuting with finite direct sums, \(H_n(a_t)=\mu_{n,t}\) and \(H_n(b_t)=\nabla_{n,t}\) by the formulas for \(a_t\) and \(b_t\), and the connecting map is \(\delta_{n,t}[c]=[\partial y]\) for every cycle \(c=y+z\) with \(y\in C_n(K^1_t)\), \(z\in C_n(K^2_t)\): the pair \((y,z)\) lifts \(c\), and \(\partial c=0\) gives \(\partial y=-\partial z\in C_{n-1}(K^{12}_t)\) by the kernel computation above, so \((\partial y,\partial z)=a_t(\partial y)\). For \(n=0\) the fourth term is \(0\).
\item \textbf{Naturality in the Scale.} The morphisms \(\mu_n\) and \(\nabla_n\) commute with the structure maps, being induced by chain maps that commute with the inclusions, and a decomposition \(c=y+z\) at \(s\) also decomposes the image of \(c\) at \(t\geq s\), so \(\delta_{n,t}(\PH^T_n(P)_{s\leq t}[c])=[\partial y]=\PH^T_{n-1}(U_{12})_{s\leq t}(\delta_{n,s}[c])\). So \(\delta_n\coloneqq(\delta_{n,t})_t\) is a morphism, and the segment is exact in \(\Pers\), exactness being checked at each scale.
\item \textbf{The Exact Sequence.} Exactness gives \(\ker\nabla_n=\operatorname{im}\mu_n\), so \(\bar\nabla_n\) is a monomorphism, \(\ker\delta_n=\operatorname{im}\nabla_n=\operatorname{im}\bar\nabla_n\) and \(\operatorname{im}\delta_n=\ker\mu_{n-1}\). Corestricting \(\delta_n\) to \(\ker\mu_{n-1}\) gives \(0\to\coker\mu_n\to\PH^T_n(P)\to\ker\mu_{n-1}\to0\), which is \eqref{eq:mv-ses} by the first step, proving \labelcref{it:mv-exact}.
\item \textbf{The Criterion.} By exactness \(\coker\chi_n\cong\ker\mu_{n-1}\), and \(\Pers\) is abelian with kernels and cokernels formed at each scale, so the monomorphism \(\chi_n\) is an isomorphism iff \(\ker\mu_{n-1}=0\), and then \(\chi_n\) gives \((\Lan_{J_2}F^T_n)(\bullet)\cong\PH^T_n(P)\). For \(n=0\), \(\ker\mu_{-1}\subseteq\PH^T_{-1}(U_{12})=0\). If \(\ker\mu_{n-1}\neq0\), pick \(t\) with \((\ker\mu_{n-1})_t\neq0\). All dimensions at \(t\) are finite by \itemref{lem:regular-modules}{it:reg-frozen}, and exactness of \eqref{eq:mv-ses} at \(t\) gives
\begin{equation}\label{eq:mv-dim}
\dim_\K\bigl((\Lan_{J_2}F^T_n)(\bullet)\bigr)_t=\dim_\K\PH^T_n(P)_t-\dim_\K(\ker\mu_{n-1})_t<\dim_\K\PH^T_n(P)_t\dc
\end{equation}
so the two modules are not isomorphic, an isomorphism of persistence modules being one at every scale. This is \labelcref{it:mv-iso}.
\item \textbf{Regularity of the Extension.} By \itemref{lem:regular-modules}{it:reg-frozen} the three modules of the span are pointwise finite-dimensional, vanish at scales \(t<0\), and have identities as structure maps on some \([t,t')\) for every \(t\) and at all scales \(\geq T\). The module \(\coker\mu_n\cong(\Lan_{J_2}F^T_n)(\bullet)\) inherits these properties: its value at \(t\) is a quotient of \(\PH^T_n(U_1)_t\oplus\PH^T_n(U_2)_t\), its structure maps are induced by those of the span, and a morphism of spans whose three components are isomorphisms induces an isomorphism of the cokernels, by \eqref{eq:mv-pushout-chain}. So \itemref{lem:regular-modules}{it:reg-dim} and \itemref{lem:regular-modules}{it:reg-inf} apply to \((\Lan_{J_2}F^T_n)(\bullet)\) and \(\PH^T_n(P)\), with \(t'\) the smaller of the two scales provided for the two modules.\tighter
\item \textbf{The Distance.} If \(\ker\mu_{n-1}=0\), then \((\chi_n,\chi_n^{-1})\) is a \(0\)-interleaving. If \(\ker\mu_{n-1}\neq0\), the dimensions differ at some scale by \eqref{eq:mv-dim}, and \itemref{lem:regular-modules}{it:reg-dim} makes the distance positive. If \((\ker\mu_{n-1})_T\neq0\), they differ at \(T\), and \itemref{lem:regular-modules}{it:reg-inf} makes it infinite. Conversely, let \((\ker\mu_{n-1})_T=0\), and write \(L\coloneqq(\Lan_{J_2}F^T_n)(\bullet)\) and \(N\coloneqq\PH^T_n(P)\). The structure maps of \(\PH^T_{n-1}(U_{12})\) and of \(\PH^T_{n-1}(U_1)\oplus\PH^T_{n-1}(U_2)\) at scales \(\geq T\) are identities by \itemref{lem:regular-modules}{it:reg-frozen}, so \(\mu_{n-1,s}=\mu_{n-1,T}\) and \((\ker\mu_{n-1})_s=0\) for \(s\geq T\), and exactness makes \(\chi_{n,s}\) an isomorphism for every \(s\geq T\). Both modules vanish at scales \(t<0\). Put \(\varepsilon\coloneqq\max\{T,0\}\), \(\varphi\coloneqq\tau^N_\varepsilon\circ\chi_n:L\to\Sigma_\varepsilon N\), and \(\psi_t\coloneqq\chi_{n,t+\varepsilon}^{-1}\circ N_{t\leq t+\varepsilon}\) for \(t\geq T-\varepsilon\), \(\psi_t\coloneqq0\) for \(t<T-\varepsilon\), the second case occurring only where \(N_t=0\), since \(T-\varepsilon\leq0\). The family \(\psi\) is natural: for \(s\leq t\) with \(s<T-\varepsilon\) both sides of \(L_{s+\varepsilon\leq t+\varepsilon}\psi_s=\psi_tN_{s\leq t}\) vanish as \(N_s=0\), and for \(s\geq T-\varepsilon\) the naturality of \(\chi_n\) at \(s+\varepsilon\leq t+\varepsilon\), where it is invertible, gives \(L_{s+\varepsilon\leq t+\varepsilon}\chi_{n,s+\varepsilon}^{-1}=\chi_{n,t+\varepsilon}^{-1}N_{s+\varepsilon\leq t+\varepsilon}\), which composed with \(N_{s\leq s+\varepsilon}\) is the claim by functoriality of \(N\). For \(t<0\) both interleaving identities hold because \(L_t=N_t=0\). For \(t\geq0\) every scale from \(t+\varepsilon\) on is \(\geq T\), and the naturality of \(\chi_n\) gives\tighter
\begin{equation}
\begin{aligned}
\psi_{t+\varepsilon}\circ\varphi_t&=\chi_{n,t+2\varepsilon}^{-1}\circ N_{t\leq t+2\varepsilon}\circ\chi_{n,t}=L_{t\leq t+2\varepsilon},\\
\varphi_{t+\varepsilon}\circ\psi_t&=N_{t+\varepsilon\leq t+2\varepsilon}\circ N_{t\leq t+\varepsilon}=N_{t\leq t+2\varepsilon}\dc
\end{aligned}
\end{equation}
so \((\varphi,\psi)\) is an \(\varepsilon\)-interleaving and the distance is at most \(\max\{T,0\}<\infty\). This is \labelcref{it:mv-dist}.
\item \textbf{The Essential Signature.} At \(T\) every space is finite-dimensional and \(\dim_\K\PH^T_k(X)_T=h_k(X)\). By the first step and rank-nullity,
\begin{equation}
\begin{gathered}
e_k=\dim_\K\coker\mu_{k,T}=h_k(U_1)+h_k(U_2)-\operatorname{rank}\mu_{k,T},\\
\dim_\K\ker\mu_{k,T}=h_k(U_{12})-\operatorname{rank}\mu_{k,T}\dd
\end{gathered}
\end{equation}
Exactness of \eqref{eq:mv-ses} at \(T\) in degree \(k+1\) gives \(e_{k+1}=h_{k+1}(P)-\dim_\K\ker\mu_{k,T}\), and eliminating the rank gives the recursion, while \(e_0=h_0(P)\) by \labelcref{it:mv-iso} for \(n=0\). This is \labelcref{it:mv-betti}.\qedhere
\end{steps}
\end{proof}

By \eqref{eq:mv-ses} the comparison \(\chi_n\) is a monomorphism, so the extension of homology can miss classes of the whole but never contains classes that the whole lacks. The loss disappears once chains instead of homology are extended (\Cref{prop:chain-gluing}). The recursion reduces the loss to arithmetic on essential Betti numbers (\Cref{fig:gluing,fig:mv-proof,fig:mv-examples}).

\begin{example}\label{ex:mayer-vietoris}
Let \(U_1,U_2\) be a \(T\)-Lebesgue cover of a finite metric space \(P\) with a fixed total order. By \itemref{thm:mayer-vietoris}{it:mv-betti} the essential signature \((e_0,e_1,e_2)\) of the extension is determined by the essential signatures \((h_0,h_1,h_2)\) of \(U_1\), \(U_2\), \(U_{12}\) and \(P\). The signatures listed below are assumptions about samples that realise them, and on a given sample they are read off the~frozen~barcodes~at~\(T\).
\begin{enumerate}
\item \textbf{Circle.} Signatures \((1,0,0)\), \((1,0,0)\), \((2,0,0)\) and \((1,1,0)\), of two arcs, their two overlapping arcs and a circle, give \(e_1=h_1(P)-h_0(U_{12})+h_0(U_1)+h_0(U_2)-e_0=1-2+1+1-1=0\neq h_1(P)\). The loop is lost, and in degree one the distance is infinite by \itemref{thm:mayer-vietoris}{it:mv-dist}, since exactness of \eqref{eq:mv-ses} at the scale \(T\) in degree one gives \(\dim_\K(\ker\mu_0)_T=h_1(P)-e_1=1\).
\item \textbf{Two Halves of a Torus.} {\addfontfeatures{LetterSpace=-1}Signatures \((1,1,0)\), \((1,1,0)\), \((2,2,0)\) and \((1,2,1)\) give \((e_0,e_1,e_2)=(1,1,0)\)}, so one loop and the void are lost.
\item \textbf{Chart and Complement of a Torus.} {\addfontfeatures{LetterSpace=-1}Signatures \((1,0,0)\), \((1,2,0)\), \((1,1,0)\) and \((1,2,1)\) give} \((e_0,e_1,e_2)=(1,2,0)\), so only the void is lost.
\end{enumerate}
In all three cases \Cref{prop:chain-gluing} recovers the signature of \(P\).\looser
\end{example}

\begin{figure}[!htbp]
\centering
\setlength{\tabcolsep}{4.5pt}
\renewcommand{\arraystretch}{1.25}
\begin{tabular}{@{}l*{3}{>{\centering\arraybackslash}p{3.65cm}}@{}}
& \paneltitle{Circle}
& \parbox[b]{\linewidth}{\centering\paneltitle{Two Halves\\of a Torus}}
& \parbox[b]{\linewidth}{\centering\paneltitle{Chart and\\Complement}} \\[3pt]
\hvpic{}
&
\hvpic{}
&
\hvpic{}
&
\hvpic{}
\\[4pt]
\figrule
\(\src{U_1}\) & \(\src{(1,0,0)}\) & \(\src{(1,1,0)}\) & \(\src{(1,0,0)}\) \\
\(\src{U_2}\) & \(\src{(1,0,0)}\) & \(\src{(1,1,0)}\) & \(\src{(1,2,0)}\) \\
\(\src{U_{12}}\) & \(\src{(2,0,0)}\) & \(\src{(2,2,0)}\) & \(\src{(1,1,0)}\) \\
\(\tgt{P}\) & \(\tgt{(1,1,0)}\) & \(\tgt{(1,2,1)}\) & \(\tgt{(1,2,1)}\) \\
\figrule
\ext{Gluing chains} & \(\ext{(1,1,0)}\) & \(\ext{(1,2,1)}\) & \(\ext{(1,2,1)}\) \\
\ext{Gluing homology} & \(\ext{(1,\cmp{\underline{0}},0)}\) & \(\ext{(1,\cmp{\underline{1}},\cmp{\underline{0}})}\) & \(\ext{(1,2,\cmp{\underline{0}})}\) \\
\cmp{Loss} \(\cmp{\dim_\K\ker\mu_{k,T}}\) & \(\cmp{(1,0)}\) & \(\cmp{(1,1)}\) & \(\cmp{(0,1)}\) \\
\end{tabular}\space\space
\caption{\Cref{ex:mayer-vietoris} computed: the circle of \Cref{fig:gluing} and tori of \(1{,}200\) points at \(T=0.5\), cut into halves or into a chart and its complement. Gluing chains returns the signature of \(P\), gluing homology the signature of \itemref{thm:mayer-vietoris}{it:mv-betti}, its losses underlined in~\colorkey{cmpdark} and counted by \(\dim_\K(\ker\mu_{n-1})_T\)~in~the~last~row.}
\label{fig:mv-examples}
\end{figure}

In degree zero the criterion only requires \(U_1\cup U_2=P\), not the \(T\)-Lebesgue condition, and it is sharp (\Cref{fig:h0-criterion}). The proof is a count of components. For a simplicial complex \(K\) with vertices in \(P\), \(C_0(K;\K)\) has the vertices as a basis and \(\operatorname{im}\partial_1\) is spanned by the differences \(y-x\) over the edges \(\{x<y\}\). With \(\pi_0(K)\) the set of components of \(K\), the linear map \(\varepsilon_K:C_0(K;\K)\to\K[\pi_0(K)]\), \(v\mapsto[v]_K\), where \([v]_K\in\pi_0(K)\) is the component of the vertex \(v\), is surjective and vanishes on \(\operatorname{im}\partial_1\). Its kernel consists of the chains \(\sum_va_vv\) with \(\sum_{v\in c}a_v=0\) for every component \(c\), so it is spanned by the differences \(y-x\) of two vertices of one component, since \(\sum_{v\in c}a_vv=\sum_{v\in c}a_v(v-v_c)\) for any vertex \(v_c\) of \(c\). An edge path \(x=z_0,\ldots,z_r=y\) gives \(y-x=\sum_{i=1}^{r}(z_i-z_{i-1})\in\operatorname{im}\partial_1\). Hence \(\ker\varepsilon_K=\operatorname{im}\partial_1\), and \(\varepsilon_K\) induces the isomorphism
\begin{equation}\label{eq:h0-components}
H_0(C_\bullet(K;\K))\cong\K[\pi_0(K)]\dc
\end{equation}
naturally in inclusions of complexes, which send a component into a component.

\begin{proposition}\label{prop:cover-gluing}
Let \((P,d)\) be a finite metric space with a fixed total order, let \(T\in\mathbb R\), and let \(U_1,U_2\subseteq P\) be arbitrary subsets.
\begin{enumerate}
\item\label{it:cg-union-complex} \textbf{Union Complex.} The morphism induced by the inclusions \(\VR^T_t(U_i)\subseteq\VR^T_t(U_1)\cup\VR^T_t(U_2)\) is an isomorphism \((\Lan_{J_2}F^T_0)(\bullet)\cong H_0\bigl(C_\bullet(\VR^T_\bullet(U_1)\cup\VR^T_\bullet(U_2);\K)\bigr)\).
\item\label{it:cg-sharp} \textbf{Sharp Criterion.} Let \(U_1\cup U_2=P\). Then \((\Lan_{J_2}F^T_0)(\bullet)\cong\PH^T_0(P)\) iff all \(x\in U_1\setminus U_2\) and \(y\in U_2\setminus U_1\) with \(d(x,y)\leq T\) are joined by an edge path in \(\VR_{d(x,y)}(U_1)\cup\VR_{d(x,y)}(U_2)\), and then \(\chi_0\) is an isomorphism. Otherwise \(d_I\bigl((\Lan_{J_2}F^T_0)(\bullet),\PH^T_0(P)\bigr)>0\).
\item\label{it:cg-coprod} \textbf{Coproduct.} For the full subcategory \(\A_2^{\mathrm{disc}}\subseteq\A_2\) on \(u_1,u_2\) and the unique functor \(J_2^{\mathrm{disc}}:\A_2^{\mathrm{disc}}\to\mathbf 1\) one has \((\Lan_{J_2^{\mathrm{disc}}}F^T_0|_{\A_2^{\mathrm{disc}}})(\bullet)\cong\PH^T_0(U_1)\oplus\PH^T_0(U_2)\). If \(U_1\cup U_2=P\), \(T\geq0\) and \(U_{12}\neq\varnothing\), this module is at interleaving distance \(\infty\) from \(\PH^T_0(P)\) and from \((\Lan_{J_2}F^T_0)(\bullet)\).
\end{enumerate}
\end{proposition}

\begin{figure}[!htbp]
\centering
\hvpic{}\space\space
\caption{\itemref{prop:cover-gluing}{it:cg-sharp} on an arc, \(x\) and \(y\) at distance \(0.40\leq T=0.62\). Top left: a path in \(\VR_{0.40}(U_1)\cup\VR_{0.40}(U_2)\) joins them, and \(\chi_0\) is an isomorphism. Top right: a gap breaks every such path, and the distance is \(0.20\). Bottom: the two longest bars of the failing case with their distances to \(\PH^T_0(P)\), the coproduct of \labelcref{it:cg-coprod} keeping two essential bars where \(\PH^T_0(P)\) has one.}
\label{fig:h0-criterion}
\end{figure}

\begin{proof}
Keep the notation of the proof of \Cref{thm:mayer-vietoris}, put \(K^{\cup}_t\coloneqq K^1_t\cup K^2_t\), a subcomplex of \(K_t\) that grows with \(t\), and let \(j_t:K^{\cup}_t\subseteq K_t\).
\begin{steps}
\item \textbf{The Union Complex.} The first five steps of that proof use the cover only through \eqref{eq:mv-union}. With \(K^{\cup}_t\) in place of \(K_t\) and \(H_n(C_\bullet(K^{\cup}_\bullet))\) in place of \(\PH^T_n(P)\), both identities of \eqref{eq:mv-union} hold for arbitrary \(U_1,U_2\), the first by definition and the second by \itemref{lem:lebesgue}{it:leb-cap}, so those steps apply verbatim. In degree zero, where \(\ker\mu_{-1}=0\), they give the isomorphism of \labelcref{it:cg-union-complex}.
\item \textbf{Components.} Assume from now on \(U_1\cup U_2=P\). By \eqref{eq:h0-components}, \(H_0(j_t)=\K[\pi_0(j_t)]\), and \(\pi_0(j_t)\) is surjective: for \(\min\{t,T\}\geq0\) both complexes have the vertex set \(P=U_1\cup U_2\), and otherwise~both~are~empty.
\item \textbf{The Edges Outside the Union.} Let \(t\geq0\) and \(u\coloneqq\min\{t,T\}\). An edge \(\{x,y\}\) of \(K_t\) lies outside \(K^{\cup}_t\) iff its vertices lie together neither in \(U_1\) nor in \(U_2\), that is iff one lies in \(U_1\setminus U_2\) and the other in \(U_2\setminus U_1\), since a point of \(U_{12}\) with any point, and two points of one complement, lie together in \(U_1\) or \(U_2\).
\item \textbf{Injectivity at One Scale.} The map \(\pi_0(j_t)\) is injective iff the two vertices of every edge of \(K_t\) lie in one component of \(K^{\cup}_t\): necessary because they lie in one component of \(K_t\), sufficient because each edge of an edge path in \(K_t\) can then be replaced by an edge path in \(K^{\cup}_t\). By the previous step only the edges with \(x\in U_1\setminus U_2\), \(y\in U_2\setminus U_1\) and \(d(x,y)\leq u\) need to be checked.\tighter
\item \textbf{The Sharp Criterion.} Suppose the criterion holds, let \(t\geq0\) and \(\{x,y\}\) such an edge of \(K_t\). Then \(d(x,y)\leq\min\{t,T\}\), and the edge path joining \(x\) and \(y\) in \(K^{\cup}_{d(x,y)}\) lies in \(K^{\cup}_t\). So \(\pi_0(j_t)\) is bijective for every \(t\), and \(H_0(j)\coloneqq(H_0(j_t))_t\), natural in \(t\) because the \(j_t\) are inclusions, is an isomorphism. Its composite with the isomorphism of the first step is induced by the inclusions \(U_i\subseteq P\), so it is \(\chi_0\). Suppose conversely that some \(x\in U_1\setminus U_2\) and \(y\in U_2\setminus U_1\) with \(t\coloneqq d(x,y)\leq T\) are not joined in \(K^{\cup}_t\). Then \(\{x,y\}\) is an edge of \(K_t\) between two components of \(K^{\cup}_t\), so \(\pi_0(j_t)\) is not injective, and with surjectivity \(\dim_\K H_0(C_\bullet(K^{\cup}_t))>\dim_\K H_0(C_\bullet(K_t))\). By the first step the two modules of \labelcref{it:cg-sharp} have different dimensions at \(t\), so they are not isomorphic. Both satisfy the hypothesis of \itemref{lem:regular-modules}{it:reg-dim}, the second by \itemref{lem:regular-modules}{it:reg-frozen}, the first because its structure maps are identities whenever \(K^1_s=K^1_t\) and \(K^2_s=K^2_t\), which \itemref{lem:regular-modules}{it:reg-frozen} for \(U_1\) and \(U_2\) gives on some \([t,t')\), so their interleaving distance is positive.
\item \textbf{The Coproduct Has Infinite Distance.} The category \(\A_2^{\mathrm{disc}}\) is discrete, so \itemref{prop:basic-task-changes}{it:merge} gives the first claim of \labelcref{it:cg-coprod}. Let \(U_1\cup U_2=P\), \(T\geq0\) and \(x\in U_{12}\), a vertex of \(K^1_T\) and of \(K^2_T\). The map \(\pi_0(K^1_T)\sqcup\pi_0(K^2_T)\to\pi_0(K^{\cup}_T)\) induced by the inclusions is surjective, every vertex of \(K^{\cup}_T\) lying in \(K^1_T\) or \(K^2_T\), and not injective, sending \([x]_{K^1_T}\neq[x]_{K^2_T}\) to \([x]_{K^{\cup}_T}\). With the second~step~one~obtains
\begin{align}
\dim_\K\PH^T_0(P)_T&\eqby{a}|\pi_0(K_T)|\leby{b}|\pi_0(K^{\cup}_T)|\\
&\by{<}{c}|\pi_0(K^1_T)|+|\pi_0(K^2_T)|\eqby{a}\dim_\K\bigl(\PH^T_0(U_1)\oplus\PH^T_0(U_2)\bigr)_T\dd
\end{align}
Here \rref{a} is \eqref{eq:h0-components}, \rref{b} holds because \(\pi_0(j_T)\) is onto, and \rref{c} because the map from \(\pi_0(K^1_T)\sqcup\pi_0(K^2_T)\) is onto and not injective. All structure maps of both modules at scales \(\geq T\) are identities by \itemref{lem:regular-modules}{it:reg-frozen}, so \itemref{lem:regular-modules}{it:reg-inf} gives the distance \(\infty\). By the first step \(\dim_\K((\Lan_{J_2}F^T_0)(\bullet))_T=|\pi_0(K^{\cup}_T)|\) is smaller as well, its structure maps at scales \(\geq T\) are identities because \(K^{\cup}_s=K^{\cup}_T\) for \(s\geq T\), and by \itemref{lem:regular-modules}{it:reg-inf} the second distance is~\(\infty\).\qedhere
\end{steps}
\end{proof}

By \itemref{lem:lebesgue}{it:leb-two} a \(T\)-Lebesgue cover satisfies the sharp criterion vacuously. Along the span the extension is a pushout and along the discrete subcategory a coproduct. With a nonempty overlap and \(T\geq0\) the two are \(\infty\) apart, which is the difference between the gluing and the merge of \Cref{ex:arcs}. Reassembling a set from the parts of a partition is also a gluing, and the collars of \itemref{lem:lebesgue}{it:leb-collar} turn every partition into a cover.

\begin{corollary}\label{cor:refinement-gluing}
Let \((P,d)\) be a finite metric space with a fixed total order, let \(T\in\mathbb R\), let \(P_1,\ldots,P_m\subseteq P\) be nonempty with \(P=P_1\cup\cdots\cup P_m\), and let \(U_i\coloneqq\{x\in P\mid\dist(\{x\},P_i)\leq T\}\). Then \(H_n\bigl((\Lan_{J_m}F^T_\bullet)(\bullet)\bigr)\cong\PH^T_n(P)\) for every \(n\in\mathbb N\).
\end{corollary}

\begin{proof}
By \itemref{lem:lebesgue}{it:leb-collar} the collars form a \(T\)-Lebesgue cover, and \Cref{prop:chain-gluing} applies.
\end{proof}

\begin{figure}[!htbp]
\centering
\hvpic{}\hspace{8mm}\ignorespaces
\begin{tabular}[t]{@{}lcc@{}}
\multicolumn{3}{@{}l@{}}{\paneltitle{Gluing Chains}} \\[6pt]
& \paneltitle{\begin{tabular}[b]{@{}c@{}}Essential\\Signature\end{tabular}} & \paneltitle{\boldmath\(d_I\)} \\
\figrule
\(\tgt{\PH^T_n(P)}\) & \(\tgt{(1,1)}\) & \\
Chains glued along \(\src{U_i}\) & \(\ext{(1,1)}\) & \(\cmp{0}\) \\
Chains glued along \(\src{P_i}\) & \(\ext{(3,0)}\) & \(\cmp{\infty}\) \\[6pt]
\multicolumn{3}{@{}p{6.6cm}@{}}{\raggedright\textcolor{greydark}{The collars form a \(T\)-Lebesgue cover, but the parts do not.}} \\
\end{tabular}\space\space
\caption{\Cref{cor:refinement-gluing} on a circle cut into three arcs at \(T=0.6\). Left: the parts \(P_i\) inside their collars \(U_i\). Right: gluing chains along the collars recovers \(\PH^T_n(P)\), while gluing along the bare parts gives three components and no loop, at distance \(\infty\) in degrees \(0\) and \(1\).}
\label{fig:refinement}
\end{figure}

Since the corollary holds for every family of parts, a collar cover computed from the representation itself makes the comparison vacuous. The chain-level discrepancy then vanishes for every representation and carries no information about it. For this reason the cover has to be prescribed by the task, and only the invariants are read from the representation. This is how we proceed in \Cref{sec:experiments}, where covers are fixed in the coordinates of the data and evaluated on the latent cloud.

\section{Discrepancy and Classification Margins}\label{sec:bridge}
For a merge of well-separated samples the extension of frozen persistent homology is exact in every degree, and the separation that exactness requires is the quantity measured by classifiers. The exactness is stable under small perturbations, and a vanishing degree-zero discrepancy is equivalent to a classification margin above \(T\).

\begin{proposition}\label{prop:separated-vr-merge}
Let \(T\geq0\), let \(k\in\mathbb N\), and let \(P_1,\ldots,P_k\) be nonempty finite subsets of a metric space \((X,d)\) with \(\dist(P_i,P_j)>T\) for all \(i\neq j\). Let \(P\coloneqq P_1\cup\cdots\cup P_k\) carry the induced metric and a fixed total order, which the parts inherit. Then the parts are pairwise disjoint, and for every \(n\in\mathbb N\) the morphism \(\sigma_n:\bigoplus_{i=1}^{k}\PH^T_n(P_i)\to\PH^T_n(P)\) with \(\sigma_n\circ\iota_i=\PH^T_n(P_i\subseteq P)\) for all \(i\), where \(\iota_i\) is the inclusion of the summand \(i\), is an isomorphism in \(\Pers\). So for \(\A\) discrete with \(\Ob(\A)=\{a_1,\ldots,a_k\}\), \(J:\A\to\mathbf 1\) the unique functor, \(F^T_n:\A\to\Pers\), \(a_i\mapsto\PH^T_n(P_i)\), \(\operatorname{id}_{a_i}\mapsto\operatorname{id}\), and \(G^T_n:\mathbf 1\to\Pers\), \(\bullet\mapsto\PH^T_n(P)\), \(\operatorname{id}_\bullet\mapsto\operatorname{id}\), the functor \(F^T_n\) is Kan-transferable to \(G^T_n\) along \(J\) for every \(n\in\mathbb N\).
\end{proposition}

\begin{proof}
~\begin{steps}
\item \textbf{Disjointness.} A point \(x\in P_i\cap P_j\) with \(i\neq j\) would give \(\dist(P_i,P_j)\leq d(x,x)=0\leq T\).
\item \textbf{Splitting of the Complex.} Let \(t\in\mathbb R\) and \(u\coloneqq\min\{t,T\}\). A simplex \(\sigma\in\VR_u(P)\) with \(x\in\sigma\cap P_i\) and \(y\in\sigma\cap P_j\), \(i\neq j\), would have \(d(x,y)\geq\dist(P_i,P_j)>T\geq u\). So \(\sigma\subseteq P_i\) for a unique \(i\), and \(\sigma\in\VR_u(P_i)\), the metric being induced. Conversely every simplex of \(\VR_u(P_i)\)~lies~in~\(\VR_u(P)\).~Hence
\begin{equation}\label{eq:sm-split}
\VR^T_t(P)=\bigsqcup_{i=1}^{k}\VR^T_t(P_i)\qquad(t\in\mathbb R)\dc
\end{equation}
a union of subcomplexes on pairwise disjoint vertex sets, and for \(s\leq t\) the inclusion \(\VR^T_s(P)\subseteq\VR^T_t(P)\) carries \(\VR^T_s(P_i)\) into \(\VR^T_t(P_i)\).
\item \textbf{Chains and Homology.} The simplices of \eqref{eq:sm-split} are partitioned by the parts, so
\[C_\bullet(\VR^T_t(P);\K)=\bigoplus_iC_\bullet(\VR^T_t(P_i);\K)\dc\]
the summand \(i\) embedded by the chain map of \(P_i\subseteq P\), compatibly with the differentials, which preserve each summand, and with the inclusions for \(s\leq t\). Homology commutes with finite direct sums, so \(\sigma_{n,t}\) is an isomorphism at every \(t\), natural in \(t\) because the splitting is natural in inclusions of complexes, and its componentwise inverse is natural as well.
\item \textbf{Identification with the Extension.} \itemref{prop:basic-task-changes}{it:merge} gives an isomorphism \(\omega\) from \((\Lan_JF^T_n)(\bullet)\) to \(\bigoplus_i\PH^T_n(P_i)\) with \(\omega\circ\eta^{F^T_n}_{a_i}=\iota_i\). Since \(\Mor(\mathbf 1)=\{\operatorname{id}_\bullet\}\), which both functors send to the identity, the isomorphism \(\sigma_n\circ\omega\) is the only component of a natural isomorphism \(\Lan_JF^T_n\Rightarrow G^T_n\).\qedhere
\end{steps}
\end{proof}

Since a representation is only observed up to noise, it matters that exactness is preserved under perturbations by less than half the excess of the separation over \(T\).

\begin{corollary}\label{cor:separated-merge-robustness}
In the setting of \Cref{prop:separated-vr-merge} let \(\Delta\geq0\), sharpen the hypothesis to \(\dist(P_i,P_j)>T+2\Delta\) for all \(i\neq j\), and let \(P'_1,\ldots,P'_k\) be nonempty finite subsets of the same space with \(d_H(P_i,P'_i)\leq\Delta\) for all \(i\), where \(d_H\) is the Hausdorff distance, and let \(P'_1\cup\cdots\cup P'_k\) carry the induced metric and a fixed total order, which the parts inherit. Then \(\dist(P'_i,P'_j)>T\) for all \(i\neq j\), and with \(F'^T_n(a_i)\coloneqq\PH^T_n(P'_i)\) and \(G'^T_n(\bullet)\coloneqq\PH^T_n(P'_1\cup\cdots\cup P'_k)\) one has \((\Lan_JF'^T_n)(\bullet)\cong G'^T_n(\bullet)\) for every \(n\in\mathbb N\), and so
\[\sum_{n=0}^{q}w_nd_I\bigl((\Lan_JF'^T_n)(\bullet),G'^T_n(\bullet)\bigr)=0\qquad(q\in\mathbb N,\ w_0,\ldots,w_q\geq0)\dd\]
\end{corollary}

\begin{proof}
Fix \(i\neq j\), \(x'\in P'_i\) and \(y'\in P'_j\). Since \(d_H(P_i,P'_i)\leq\Delta\) and \(P_i\) is finite, some \(x\in P_i\) has \(d(x,x')\leq\Delta\), and likewise some \(y\in P_j\) has \(d(y,y')\leq\Delta\), so
\begin{equation}
d(x',y')\geby{a}d(x,y)-d(x,x')-d(y,y')>(T+2\Delta)-\Delta-\Delta=T\dd
\end{equation}
Here \rref{a} is the triangle inequality. The minimum defining \(\dist(P'_i,P'_j)\) is attained at such a pair, so \(\dist(P'_i,P'_j)>T\), and \Cref{prop:separated-vr-merge} gives \((\Lan_JF'^T_n)(\bullet)\cong G'^T_n(\bullet)\). The interleaving distance vanishes on isomorphic modules (\Cref{lem:diagram-pseudometric}), so every summand is zero.
\end{proof}

\begin{figure}[!htbp]
\centering
\hvpic{}\space\space
\caption{\Cref{prop:separated-vr-merge,cor:separated-merge-robustness} at \(T=0.6\). Top left: the gaps \(1.44\) and \(1.31\) exceed \(T+2\Delta=0.90\), so no simplex meets two parts. Top right: moving every point by at most \(\Delta=0.15\) keeps the parts more than \(T\) apart. Bottom: without the margin the union becomes connected at \(T\),~and~the~distance~is~\(\infty\).}
\label{fig:separated-merge}
\end{figure}

The margin \(2\Delta\) cannot be dropped (\Cref{fig:separated-merge}). For example, for \(T,\Delta>0\) the sets \(P_1=P'_1=\{0\}\), \(P_2=\{T+\Delta\}\) and \(P'_2=\{T\}\) on the real line satisfy \(d_H(P_i,P'_i)\leq\Delta\) and \(\dist(P_1,P_2)>T\), but \(\VR_T(P'_1\cup P'_2)\) is the edge \(\{0,T\}\). At \(T\) the observed module is one-dimensional, whereas the extension \(\PH^T_0(P'_1)_T\oplus\PH^T_0(P'_2)_T\) is two-dimensional. The structure maps at scales \(\geq T\) are identities by \itemref{lem:regular-modules}{it:reg-frozen}, so \itemref{lem:regular-modules}{it:reg-inf} gives the distance \(\infty\).

A transfer score is useful for learning if it predicts how a classifier can behave on the representation from which the score was computed. For a merge of labelled clouds we obtain such a prediction in degree zero, in the form of an equivalence. The extension keeps the clouds apart, with one essential class for each component of every cloud, whereas the observed invariant fuses two components as soon as points of different labels come within distance \(T\). The same close pairs bound the margin of every Lipschitz score head, and the best margin of a \(1\)-Lipschitz head is the smallest distance between differently labelled points. For two classes this equality follows from the hard-margin case of {\setcitestyle{notesep={,~}}\citet[Lemmas~11 and~12]{luxburg2004distance}}.\keeptail

\begin{theorem}\label{thm:bridge}
Let \((P,d)\) be a finite metric space with a fixed total order, let \(T\geq0\), let \(C\) be a finite set and \(y:P\to C\) a map, the \textbf{labelling}, and let the \textbf{class clouds} \(P_c\coloneqq y^{-1}(c)\) for \(c\in C\) carry the induced metric and order. Let \(\A_C\) be the discrete category with \(\Ob(\A_C)=C\) and \(\Mor(\A_C)=\{\operatorname{id}_c\mid c\in C\}\), let \(J:\A_C\to\mathbf 1\) be the unique functor, and for \(n\in\mathbb N\) let
\begin{equation}
\begin{aligned}
F^T_n&:\A_C\to\Pers,&c&\mapsto\PH^T_n(P_c),&\operatorname{id}_c&\mapsto\operatorname{id},\\
G^T_n&:\mathbf 1\to\Pers,&\bullet&\mapsto\PH^T_n(P),&\operatorname{id}_\bullet&\mapsto\operatorname{id}\dc
\end{aligned}
\end{equation}
and let \(\chi_n:(\Lan_JF^T_n)(\bullet)\to\PH^T_n(P)\) be the unique morphism with \(\chi_n\circ\eta^{F^T_n}_c=\PH^T_n(P_c\subseteq P)\) for all \(c\in C\), where \(\PH^T_n(P_c\subseteq P)\) is induced by the inclusions \(\VR^T_t(P_c)\subseteq\VR^T_t(P)\). A \textbf{score head} is a family \(s=(s_c)_{c\in C}\) of maps \(s_c:P\to\mathbb R\), and its \textbf{margin} is
\begin{equation}\label{eq:margin}
m_s:P\to(-\infty,\infty],\qquad x\mapsto\min_{c\in C\setminus\{y(x)\}}\bigl(s_{y(x)}(x)-s_c(x)\bigr)\dc
\end{equation}
with \(\min\varnothing\coloneqq\infty\), so that \(m_s\) is constantly \(\infty\) if \(|C|=1\). The \textbf{class gap} and the \mbox{\textbf{optimal margin}} of \((P,y)\) are
\begin{equation}
\delta(P,y)\coloneqq\min_{\substack{x,x'\in P\\ y(x)\neq y(x')}}d(x,x')\in[0,\infty],\qquad\gamma(P,y)\coloneqq\sup_{s}\,\min_{x\in P}m_s(x)\in(-\infty,\infty]\dc
\end{equation}
with \(\min\varnothing\coloneqq\infty\) in both minima, the supremum running over the score heads whose maps \(s_c\) are all \(1\)-Lipschitz, and the \textbf{bichromatic graph} \(B_T\) has the vertex set \(P\) and an edge \(\{x,x'\}\) whenever \(y(x)\neq y(x')\) and \(d(x,x')\leq T\).
\begin{enumerate}
\item\label{it:bridge-margin} \textbf{Margin Bound.} Let \(L\geq0\) and let \(s\) be a score head whose maps \(s_c\) are all \(L\)-Lipschitz. Then \(m_s(x)+m_s(x')\leq2L\,d(x,x')\), and so \(\min\{m_s(x),m_s(x')\}\leq L\,d(x,x')\), for all \({x,x'\in P}\)~with~\({y(x)\neq y(x')}\).\keeptail
\item\label{it:bridge-optimal} \textbf{Optimal Margin.} The optimal margin equals the class gap, \(\gamma(P,y)=\delta(P,y)\), and the supremum is attained if \(\delta(P,y)<\infty\).
\item\label{it:bridge-equiv} \textbf{Transfer Criterion.} The following are equivalent:
\begin{enumerate}
\item\label{it:bridge-chi} the morphism \(\chi_0\) is an isomorphism in \(\Pers\),
\item\label{it:bridge-exact} the functor \(F^T_0\) is Kan-transferable to \(G^T_0\) along \(J\),
\item\label{it:bridge-zero} the discrepancy \(\Comp_J(F^T_0,G^T_0)\) vanishes in the invariant target \((\Pers,\Pers^{\mathrm{pfd}},d_I)\),
\item\label{it:bridge-finite} the discrepancy \(\Comp_J(F^T_0,G^T_0)\) is finite,
\item\label{it:bridge-sep} all \(x,x'\in P\) with \(y(x)\neq y(x')\) satisfy \(d(x,x')>T\),
\item\label{it:bridge-gamma} the optimal margin satisfies \(\gamma(P,y)>T\),
\item\label{it:bridge-head} some score head whose maps \(s_c\) are all \(1\)-Lipschitz has \(m_s(x)>T\) for every \(x\in P\).
\end{enumerate}
If they hold, then \(\chi_n\) is an isomorphism for every \(n\in\mathbb N\). In all cases
\begin{equation}\label{eq:bridge-dichotomy}
\Comp_J(F^T_0,G^T_0)=d_I\bigl((\Lan_JF^T_0)(\bullet),\PH^T_0(P)\bigr)=\begin{cases}0&\text{if }\gamma(P,y)>T,\\ \infty&\text{if }\gamma(P,y)\leq T\dd\end{cases}
\end{equation}
\item\label{it:bridge-violations} \textbf{Margin Violations.} Let \(L\geq0\) and let \(s\) be a score head whose maps \(s_c\) are all \(L\)-Lipschitz. Then \(V_s\coloneqq\{x\in P\mid m_s(x)\leq LT\}\) is a vertex cover of \(B_T\), so \(|V_s|\geq\nu(B_T)\), the matching number of \(B_T\).\footnote{A \textbf{vertex cover} of a graph is a set of vertices that meets every edge, and the \textbf{matching number} \(\nu\) of the graph is the largest number of pairwise disjoint edges.} If \(\Comp_J(F^T_0,G^T_0)\neq0\), then \(\nu(B_T)\geq1\).
\item\label{it:bridge-linear} \textbf{Linear Heads.} Let \(P\subseteq\mathbb R^D\) carry the Euclidean metric, and let \(s_c:P\to\mathbb R\), \(x\mapsto\langle w_c,x\rangle+b_c\), with \(w_c\in\mathbb R^D\) and \(b_c\in\mathbb R\) for \(c\in C\). Then \(s_c\) is \(\lVert w_c\rVert\)-Lipschitz, and \(m_s(x)+m_s(x')\leq\lVert w_{y(x)}-w_{y(x')}\rVert\,\lVert x-x'\rVert\) for all \(x,x'\in P\) with \(y(x)\neq y(x')\). If \(\Comp_J(F^T_0,G^T_0)\neq0\), then some \(x,x'\in P\) with \(y(x)\neq y(x')\) satisfy \(\min\{m_s(x),m_s(x')\}\leq{\tfrac12\lVert w_{y(x)}-w_{y(x')}\rVert\,T\leq T\max_{c\in C}\lVert w_c\rVert}\).\keeptail
\end{enumerate}
\end{theorem}

\begin{remark}\label{rem:bridge}
\Cref{thm:bridge} is a statement about the training sample \((P,y)\) and about the margins of Lipschitz score heads on it, not about test accuracy. A head is only defined on the sample, but nothing is lost by this restriction. Every \(1\)-Lipschitz map on \(P\) extends to a \(1\)-Lipschitz map on the ambient metric space \citep{mcshane1934extension}, so the margins are those of classifiers on the whole latent space, measured in the units of the freezing scale. The discrepancy of the merge vanishes if and only if some \(1\)-Lipschitz head separates the sample with a margin above \(T\). Otherwise every \(L\)-Lipschitz head, however it was trained, has margin at most \(LT\) at one end of every close pair of differently labelled points. The class gap alone does not give this bound on every trained classifier. The translation into margins adds it. Points outside \(P\) play no role. Accuracy on unseen data requires a sampling model and a generalisation bound for Lipschitz classifiers \citep{luxburg2004distance}, and the theorem provides neither.

The deviation has a direction, since \(\chi_0\) is an epimorphism for every labelling. The observation identifies components of the extension, but it never has components that the extension lacks. The criterion is also sharp. When \(T<\delta(P,y)<\infty\) the distance head \(s^{\delta(P,y)}\) of the proof has margin at least \(\delta(P,y)\) at every point, and when \(\delta(P,y)=\infty\) the heads \(s^r\) with \(r\in\mathbb R\) have margin at least \(r\) everywhere. When \(\delta(P,y)\leq T\), no \(1\)-Lipschitz head exceeds \(\delta(P,y)\) at both ends of a closest pair. The discrepancy jumps from \(\infty\) to \(0\) exactly where the optimal margin passes \(T\). The bound of \labelcref{it:bridge-margin} is attained. For \(P=\{0,\rho\}\subseteq\mathbb R\) with \(\rho>0\), two labels \(y(0)\neq y(\rho)\) and the linear heads \(s_{y(0)}:P\to\mathbb R\), \(x\mapsto\tfrac\rho2-x\), and \(s_{y(\rho)}:P\to\mathbb R\), \(x\mapsto x-\tfrac\rho2\), both margins equal \(\rho\), so that \(m_s(0)+m_s(\rho)=2\,d(0,\rho)\).

The equivalence is special to degree zero. Separation makes every \(\chi_n\) an isomorphism, but the converse fails in higher degrees. For a sample of two differently labelled points \(x,x'\) at distance at most \(T\), every complex involved has at most two vertices, hence no simplex of dimension two and at most one edge, whose boundary \(\pm(x'-x)\) is nonzero. So \(H_1\) vanishes, and \(\chi_1\) is an isomorphism of zero modules, while the degree-zero discrepancy is \(\infty\).
\end{remark}

\begin{proof}[Proof of Theorem~\ref{thm:bridge}]
~\begin{steps}
\item \textbf{Margins of a Differently Labelled Pair.} Let \(s\) have \(L\)-Lipschitz maps and \(y(x)\neq y(x')\). By \eqref{eq:margin}, \(m_s(x)\leq s_{y(x)}(x)-s_{y(x')}(x)\) and \(m_s(x')\leq s_{y(x')}(x')-s_{y(x)}(x')\), so both margins are real, and adding the two bounds gives
\begin{equation}
m_s(x)+m_s(x')\leq\bigl(s_{y(x)}(x)-s_{y(x)}(x')\bigr)+\bigl(s_{y(x')}(x')-s_{y(x')}(x)\bigr)\leby{a}2L\,d(x,x')\dd
\end{equation}
Here \rref{a} holds because both \(s_{y(x)}\) and \(s_{y(x')}\) are \(L\)-Lipschitz. The smaller of two real numbers is at most their mean, which gives \labelcref{it:bridge-margin}.
\item \textbf{Distance Heads.} Let \(r\in\mathbb R\) satisfy \(r\leq d(x,x')\) whenever \(y(x)\neq y(x')\), and let the \textbf{distance head} \(s^r\) consist of the maps
\begin{equation}
s^r_c:P\to\mathbb R,\qquad x\mapsto\begin{cases}-\dist(\{x\},P_c)&\text{if }P_c\neq\varnothing,\\ -r&\text{if }P_c=\varnothing\dd\end{cases}
\end{equation}
Constant maps are \(1\)-Lipschitz, and so is \(x\mapsto\dist(\{x\},P_c)\): \(\dist(\{x\},P_c)\leq d(x,z)\leq d(x,x')+d(x',z)\) for \(z\in P_c\), and minimising over \(z\) gives \(\dist(\{x\},P_c)-\dist(\{x'\},P_c)\leq d(x,x')\), and symmetrically. For \(x\in P\), \(s^r_{y(x)}(x)=0\). For \(c\neq y(x)\) with \(P_c\neq\varnothing\) the minimum \(\dist(\{x\},P_c)\) is attained at some \(z\) with \(y(z)=c\neq y(x)\), so \(s^r_{y(x)}(x)-s^r_c(x)=d(x,z)\geq r\), and for \(P_c=\varnothing\) the difference is \(r\). Hence \(m_{s^r}\geq r\) on \(P\).
\item \textbf{The Optimal Margin Is the Class Gap.} For a \(1\)-Lipschitz head \(s\) and \(y(x)\neq y(x')\), the first step gives \(\min_zm_s(z)\leq\min\{m_s(x),m_s(x')\}\leq d(x,x')\), so \(\gamma(P,y)\leq\delta(P,y)\). Conversely every real \(r\leq\delta(P,y)\) admits \(s^r\), so \(\gamma(P,y)\geq r\). For \(\delta(P,y)<\infty\), \(r=\delta(P,y)\) attains the supremum, and otherwise \(\gamma(P,y)\geq r\) for all real \(r\). This is \labelcref{it:bridge-optimal}. As \(P\) is finite, \labelcref{it:bridge-sep} says \(\delta(P,y)>T\), which is \labelcref{it:bridge-gamma}. Under \labelcref{it:bridge-gamma} put \(r\coloneqq\delta(P,y)\) if \(\delta(P,y)<\infty\) and \(r\coloneqq T+1\) otherwise, so that \(T<r\leq\delta(P,y)\), and \(s^r\) witnesses \labelcref{it:bridge-head}. Under \labelcref{it:bridge-head} the minimum of the margins over the finite set \(P\) exceeds \(T\), so \(\gamma(P,y)>T\).
\item \textbf{The Extension Is the Direct Sum over the Classes.} Let \(n\in\mathbb N\). \itemref{prop:basic-task-changes}{it:merge} gives an isomorphism \(\omega_n:(\Lan_JF^T_n)(\bullet)\to\bigoplus_{c\in C}\PH^T_n(P_c)\) with \(\omega_n\circ\eta^{F^T_n}_c=\iota_c\). The morphism \(\sigma_n\) with \(\sigma_n\circ\iota_c=\PH^T_n(P_c\subseteq P)\) satisfies
\[\sigma_n\omega_n\eta^{F^T_n}_c=\sigma_n\iota_c=\PH^T_n(P_c\subseteq P)\dc\]
so \(\chi_n=\sigma_n\circ\omega_n\) by the uniqueness defining \(\chi_n\), and \(\chi_n\) is an isomorphism iff \(\sigma_n\) is. All modules are pointwise finite-dimensional by \itemref{lem:regular-modules}{it:reg-frozen}, and \(\Ob(\mathbf 1)=\{\bullet\}\) gives \(\Comp_J(F^T_n,G^T_n)=d_I((\Lan_JF^T_n)(\bullet),\PH^T_n(P))\).
\item \textbf{Separation Splits the Complex.} Assume \labelcref{it:bridge-sep} and let \(C'\coloneqq\{c\mid P_c\neq\varnothing\}\). For \(c\neq c'\) in \(C'\) every \(x\in P_c\), \(x'\in P_{c'}\) have different labels, so \(\dist(P_c,P_{c'})>T\), and \Cref{prop:separated-vr-merge} for the nonempty finite clouds \(P_c\), \(c\in C'\), which cover \(P\), makes \(\sigma'_n:\bigoplus_{c\in C'}\PH^T_n(P_c)\to\PH^T_n(P)\) an isomorphism. For \(c\notin C'\), \(\PH^T_n(P_c)=0\), so \(\sigma_n\) agrees with \(\sigma'_n\), and \(\chi_n\) is an isomorphism for every \(n\), in particular \labelcref{it:bridge-chi}.
\item \textbf{An Isomorphism Makes the Discrepancy Vanish.} If \labelcref{it:bridge-chi} holds, \(\chi_0\) is the only component of a natural isomorphism \(\Lan_JF^T_0\Rightarrow G^T_0\), since \(\Mor(\mathbf 1)=\{\operatorname{id}_\bullet\}\), which is \labelcref{it:bridge-exact}. If \labelcref{it:bridge-exact} holds, \(d_I\) vanishes on the isomorphic modules (\Cref{lem:diagram-pseudometric}), which is \labelcref{it:bridge-zero}, and \labelcref{it:bridge-zero} implies \labelcref{it:bridge-finite}.
\item \textbf{A Close Pair Makes the Discrepancy Infinite.} Assume \(y(x)\neq y(x')\) and \(d(x,x')\leq T\). At the scale \(T\), \eqref{eq:h0-components} and the fourth step give
\begin{equation}\label{eq:brp-dims}
\dim_\K\PH^T_0(P)_T=|\pi_0(\VR_T(P))|,\qquad\dim_\K\bigl((\Lan_JF^T_0)(\bullet)\bigr)_T=\sum_{c\in C}|\pi_0(\VR_T(P_c))|\dd
\end{equation}
The map \(\bigsqcup_c\pi_0(\VR_T(P_c))\to\pi_0(\VR_T(P))\) induced by the inclusions is surjective, every \(z\) being a vertex of \(\VR_T(P_{y(z)})\) as \(T\geq0\), and not injective, since the components of \(x\) in \(\VR_T(P_{y(x)})\) and of \(x'\) in \(\VR_T(P_{y(x')})\) are distinct elements of the disjoint union that both go to the component of the edge \(\{x,x'\}\). So the dimensions in \eqref{eq:brp-dims} differ, all structure maps at scales \(\geq T\) are identities by \itemref{lem:regular-modules}{it:reg-frozen}, and \itemref{lem:regular-modules}{it:reg-inf} with \(\omega_0\) gives \(\Comp_J(F^T_0,G^T_0)=\infty\). So \labelcref{it:bridge-finite} implies \labelcref{it:bridge-sep}.
\item \textbf{The Cycle of Implications.} The last four steps give \labelcref{it:bridge-sep}\(\Rightarrow\)\labelcref{it:bridge-chi}\(\Rightarrow\)\labelcref{it:bridge-exact}\(\Rightarrow\)\labelcref{it:bridge-zero}\(\Rightarrow\)\labelcref{it:bridge-finite}\(\Rightarrow\)\labelcref{it:bridge-sep}, the third step joins \labelcref{it:bridge-gamma} and \labelcref{it:bridge-head}, and the fifth makes every \(\chi_n\) an isomorphism under \labelcref{it:bridge-sep}. If \(\gamma(P,y)>T\), \labelcref{it:bridge-zero} holds, and if \(\gamma(P,y)\leq T\), \labelcref{it:bridge-sep} fails and the discrepancy is \(\infty\), which is \eqref{eq:bridge-dichotomy} and completes \labelcref{it:bridge-equiv}.\looser
\item \textbf{Violations Form a Vertex Cover.} For an edge \(\{x,x'\}\) of \(B_T\) and an \(L\)-Lipschitz head \(s\), the first step gives \(\min\{m_s(x),m_s(x')\}\leq L\,d(x,x')\leq LT\), so \(V_s\) meets the edge. Choosing an endpoint in \(V_s\) for each edge of a set of pairwise disjoint edges is injective, so \(|V_s|\geq\nu(B_T)\). If \(\Comp_J(F^T_0,G^T_0)\neq0\), \labelcref{it:bridge-zero} fails, so \labelcref{it:bridge-sep} fails and \(B_T\) has an edge. This is \labelcref{it:bridge-violations}.
\item \textbf{Margins of Linear Heads.} By Cauchy--Schwarz, \(|s_c(x)-s_c(x')|=|\langle w_c,x-x'\rangle|\leq\lVert w_c\rVert\,\lVert x-x'\rVert\). For \(y(x)\neq y(x')\), the two margins satisfy
\begin{align}
m_s(x)+m_s(x')&\leby{a}\bigl(s_{y(x)}(x)-s_{y(x')}(x)\bigr)+\bigl(s_{y(x')}(x')-s_{y(x)}(x')\bigr)\\
&\eqby{b}\langle w_{y(x)}-w_{y(x')},x-x'\rangle\leby{c}\lVert w_{y(x)}-w_{y(x')}\rVert\,\lVert x-x'\rVert\dd
\end{align}
Here \rref{a} is \eqref{eq:margin}, \rref{b} the definition of the linear head and \rref{c} the Cauchy--Schwarz inequality. If \(\Comp_J(F^T_0,G^T_0)\neq0\), the previous step gives an edge \(\{x,x'\}\) of \(B_T\), so \(\lVert x-x'\rVert\leq T\), and the smaller margin is at most half the sum, at most \(\tfrac12\lVert w_{y(x)}-w_{y(x')}\rVert\,T\leq\tfrac12(\lVert w_{y(x)}\rVert+\lVert w_{y(x')}\rVert)T\leq T\max_c\lVert w_c\rVert\). This is \labelcref{it:bridge-linear}.\qedhere
\end{steps}
\end{proof}

In \Cref{ex:arcs} with \(T\geq0\), label the points of \(U_1\setminus U_2\) and of \(U_2\setminus U_1\) by two classes. By \itemref{lem:lebesgue}{it:leb-two} the cover is \(T\)-Lebesgue iff the class gap of this labelled sample exceeds \(T\), so by \itemref{thm:bridge}{it:bridge-equiv} the gluing hypothesis of \Cref{sec:gluing} holds iff some \(1\)-Lipschitz head separates the two exclusive parts with margin above \(T\). The experiments in \Cref{sec:experiments} use this equivalence to interpret the geometry of a learned cover as a margin.

\section{Experiments}\label{sec:experiments}
Since the theorems hold for every finite metric space, they cannot fail on data. What can fail is the implementation, the hypotheses on learned representations, the value of the discrepancy compared with simpler scores and its relation to downstream performance. The theory does not use that a point cloud comes from a trained network, and the network enters only through the clouds it produces. We test the parts that can fail with five questions. (a)~Do the exact predictions hold on samples of known topology, and how sharp are their hypotheses? (b)~Does the discrepancy identify the change of task that produced an observation? (c)~How often do encoders that are never given a cover produce latent covers that satisfy the gluing hypothesis? (d)~Does the discrepancy of a prescribed change of task on pretrained image models track the accuracy and margin of a linear probe, compared with transferability and similarity scores? (e)~What does the computation cost? A single run of the code in the repository (\Cref{app:details}) produces every table and plot below, and the code also computes every number shown in the figures.

\subsection{Protocol}
\textbf{Scores.} Unless stated otherwise, point clouds carry Gaussian noise of standard deviation \(0.012\), and invariants are the frozen persistent homology of \Cref{def:frozen-vr} at \(T=0.65\) over \(\mathbb F_2\) up to the dimension of the sampled space. They are computed with GUDHI \citep{maria2014gudhi} after an edge collapse that keeps the barcode \citep{boissonnat2020edge}, and bottleneck distances are exact, by bipartite matching \citep{kerber2017geometry}. The \textbf{frozen score} is \eqref{eq:weighted-score} with \(d=d_B\) and unit weights, the sum over the degrees of the distances at a target object, maximised over the target objects. It is \(\infty\) whenever an essential Betti number differs. Extension and observation are computed from the same pairwise distances, so a predicted isomorphism means bit-identical barcodes, and an \textbf{exact match} is a frozen score of \(0\) in floating point. The \textbf{truncated score} replaces each bar \([x,y)\) with \(x\leq T\) by \([x,\min\{y,T\})\), drops the bars born after \(T\) and aggregates like the frozen score. It is finite, at most \(T/2\) per degree, independent of the freezing scale \(T'\geq T\) of the barcodes it truncates, and not divided by \(T\). For an extension of homology it is the weighted discrepancy of the truncated invariants (\Cref{lem:truncation}), to which \Cref{thm:lan-interleaving-stability} applies.

\textbf{Controls and Baselines.} A \textbf{negative control} is a target that the stated change of task does not produce and whose defect is known: a \emph{resample} of the same law, which sets the noise floor, a \emph{twin} with the same essential Betti numbers over \(\mathbb F_2\) and another geometry (such as a Klein bottle for a torus or unlinked circles for a Hopf link), and an \emph{other} space with another signature. A control is detected if it scores above its resample. The baselines are the \emph{task-aware Betti score}, the summed differences of the essential Betti numbers of extension and observation, and the \emph{objectwise score}, the mean truncated distance of the source invariants to the target, without transport along \(J\). Task recovery adds the point count and a \emph{local-scale statistic}, the Wasserstein distance between the distributions of fifth-nearest-neighbour distances. For (d) they are LogME, LEEP, NCE, H-score, \(20\)-nearest-neighbour accuracy \citep{cover1967nearest}, silhouette \citep{rousseeuw1987silhouettes}, linear CKA and RTD \citep{tran2019transferability,bao2019information,nguyen2020leep,you2021logme,kornblith2019similarity,barannikov2022representation}. The Betti score only takes essential classes into account. It agrees with the frozen score on whether a score is \(\infty\), and it is \(0\) on every twin whose sample keeps the essential signature of its truth, whereas the truncated score can still separate such twins.

\textbf{Checks, Statistics, Machine.} Every run evaluates each statement it instantiates on every row and stops with an error listing violating rows. Since the statements are proved, such an error would reveal an implementation defect. Rates carry exact Clopper--Pearson intervals at \(95\,\%\) \citep{clopper1934use}, and mean accuracies carry percentile bootstrap intervals of \(2{,}000\) rounds. Associations are Kendall's \(\tau_b\) \citep{kendall1938new,kendall1945treatment} and Spearman's \(\rho\), with percentile bootstrap intervals and permutation tests of \(2{,}000\) rounds each in (a) and (c) and of \(10{,}000\) in (d). Intervals are rounded outward. The three pre-registered correlation tests of (d) are one-sided and adjusted by Holm's procedure \citep{holm1979simple} at level \(0.05\), and the fourth is a one-sided binomial sign test for each merge. The generators of (a), (b) and (e) are seeded by a stable hash of experiment, space and seed, those of (c) by the seed, and (d) uses fixed seeds. All runs use one laptop with an AMD Ryzen~5~230 processor of six cores, \(13.3\) GB of usable memory and no GPU. Six worker processes share the machine with other load. Every result file records machine, library versions and wall time.

\subsection{Known Topology}
\textbf{Data.} Seven spaces of known essential signature are sampled with \(20\) seeds: a circle, a torus, a sphere, a Klein bottle in \(\mathbb R^4\), a genus-two surface, a figure eight and a Hopf link (\Cref{tab:samplers}). Merges place two spaces (a circle and a sphere, a torus and a Klein bottle, or a figure eight and a Hopf link) at a realised gap \(\delta\in\{0.5,0.9,0.99,1.01,1.1,2,9.23\}\,T\), or perturb both by at most \(\Delta\in\{0.05,0.10\}\) at a gap just above \(T+2\Delta\). The merges at the gaps \(0.9\,T\), \(1.1\,T\) and \(9.23\,T\) are also lifted to \(\mathbb R^D\), \(D\in\{16,64,256,512\}\), by a random orthonormal frame with ambient noise \(0.005\). Covers split a space into two sets at five overlap widths, a circle into \(m\in\{3,4,6\}\) arcs and a torus into \(m\in\{3,4\}\) sectors, or a space into the collars of radius \(\rho\in\{T,T/2,0\}\) of a partition into two to six parts. Each prediction below would be falsified by a single violating row:
\begin{enumerate}[hypothesis]
\item\label{hyp:merge} \textbf{Merge} (\Cref{prop:separated-vr-merge,cor:separated-merge-robustness}). At \(\delta>T\), and at \(\delta>T+2\Delta\) after a perturbation by at most \(\Delta\), the extension and the observed union agree in every degree.
\item\label{hyp:sharp} \textbf{Sharpness} (\itemref{thm:bridge}{it:bridge-equiv}, the pieces as classes). At \(\delta\leq T\) the degree-zero frozen score is \(\infty\). If \(\VR_\delta(P_i)\) is connected for each piece \(P_i\), the two degree-zero barcodes differ in one bar, \([0,T)\) against \([0,\delta)\) after truncation, so the truncated degree-zero score is \(\min\{T-\delta,T/2\}\).
\item\label{hyp:chains} \textbf{Chains} (\Cref{prop:chain-gluing}). Along every \(T\)-Lebesgue cover by \(2\), \(3\), \(4\) or \(6\) sets the chain-level extension and the whole agree in every degree. For \(m\geq3\) sets the condition is checked on the simplices of dimension at most \(q+1\), which suffices by the remark after \Cref{prop:chain-gluing}.
\item\label{hyp:defect} \textbf{Defect} (\itemref{thm:mayer-vietoris}{it:mv-betti}). Along every \(T\)-Lebesgue two-set cover the homology-level extension loses \(h_n(P)-e_n\) essential classes in degree \(n\), no more and no fewer.
\item\label{hyp:degree-zero} \textbf{Degree Zero} (\itemref{prop:cover-gluing}{it:cg-sharp} and \labelcref{it:cg-coprod}). On Voronoi two-partitions with collars \(\rho/T\in\{0,0.1,0.25,0.5,0.75,0.9,1\}\) and on narrow two-set covers, the degree-zero frozen score vanishes iff the sharp criterion holds, also off the \(T\)-Lebesgue condition, and the coproduct is at \(\infty\) for nonempty overlaps.
\item\label{hyp:collars} \textbf{Collars} (\Cref{cor:refinement-gluing}). Collars of radius \(T\) glue without loss in every degree. For \(\rho<T\) no statement applies, and the rate is measured.
\item\label{hyp:controls} \textbf{Controls} (\itemref{lem:regular-modules}{it:reg-inf} and the remark after \Cref{prop:correctness}). A control has frozen score \(\infty\) iff its essential signature differs from that of the extension. No statement says that persistent homology separates a twin from the truth, so the detection rate of twins is measured.
\item\label{hyp:sweeps} \textbf{Robustness.} \ref{hyp:chains} and \ref{hyp:defect} concern the finite sample, so they hold on two-set covers of the circle, the torus and the sphere for \(T\in\{0.4,0.5,0.8\}\), sample-size factors \(\{0.5,0.75,1.5\}\), noise levels \(\{0,0.03,0.06\}\) and lifts to \(\mathbb R^D\), \(D\in\{16,64,256,512\}\), with and without ambient noise \(0.005\), also where the sample no longer realises its manifold, and \ref{hyp:merge} and \ref{hyp:sharp} hold~on~the~lifted~merges.
\end{enumerate}
On every merge row we also check the bridge, that is, that the degree-zero frozen score vanishes iff \(\delta>T\) iff the distance head has margin above \(T\), and that a trained linear probe has smallest margin at most \(\tfrac\delta2\lVert w_{y(x)}-w_{y(x')}\rVert\leq\delta\max_c\lVert w_c\rVert\), where \(x,x'\) is a closest pair of differently labelled points (\itemref{thm:bridge}{it:bridge-linear}).

\textbf{Results.} Every prediction holds on every row (\Cref{tab:known-checks}). Below \(T\) the truncated degree-zero score equals \(\min\{T-\delta,T/2\}\) bit for bit on the \(199\) of \(300\) rows on which every piece is connected at a scale below \(\delta\), and above \(T\) it vanishes (\Cref{fig:known-topology}, left). Along every \(T\)-Lebesgue two-set cover the homology-level extension loses exactly the classes predicted by the recursion (\Cref{tab:known-defects}). These are the loop of the circle, a loop and the void of a torus or a Klein bottle cut into halves, and the void of a torus cut into a chart and its complement, of a sphere cut into caps and of a genus-two surface cut at its neck. The degree-zero criterion classifies all \(2{,}800\) Voronoi and narrow covers correctly, also where \Cref{thm:mayer-vietoris} does not apply. Of these, \(1{,}374\) have score \(0\) and \(1{,}426\) a positive score. Collars of radius \(T/2\) are \(T\)-Lebesgue in only \(13\) of \(400\) rows but glue without loss in \(266\), whereas bare parts never do.\tighter

The negative controls behave as stated in \ref{hyp:controls} (\Cref{tab:known-controls}). A control scores \(\infty\) if and only if its essential signature differs. Six of the \(600\) resample rows and \(16\) of the \(600\) twin rows miss the signature of their truth at \(T\), and one of the \(600\) other rows shares it. The \(600\) rows of each kind compare \(200\) distinct clouds, since the covers and the collars of a space reuse its controls. The Betti score cannot detect the twins that share the signature. The truncated score, however, places the twin above its resample in \(396\) rows, about two thirds, through finite bars. A Hopf link, for instance, differs from unlinked circles in the distance between its two circles and not in its homology. On the merges (\Cref{tab:known-bridge}) the frozen score vanishes if and only if \(\delta>T\), and the probe margin stays below its bound, far below at large gaps, since logistic regression does not maximise the margin (\Cref{fig:known-topology}, right). Across the merge rows the truncated score and the normalised probe margin are associated with \(\tau=-0.448\), and the class gap and the margin with \(\tau=0.592\). \Cref{thm:bridge} suggests such associations but does not imply them.

\begin{table}[!htb]
\centering
\caption{Experiment (a), the predictions \ref{hyp:merge} to \ref{hyp:collars}, \ref{hyp:sweeps} and the bridge: rows on which the prediction holds among those it applies to, with exact \(95\,\%\) intervals.}
\label{tab:known-checks}
\tablestyle
\begin{tabular}{@{}llcc@{}}
\toprule
 & \colhead{Prediction} & \colhead{Rows} & \colhead{\(95\,\%\) interval} \\
\midrule
\ref{hyp:merge} & Gap \(\delta>T\): frozen score \(0\) in every degree & \pill & \([0.993,1.000]\) \\
 & Points moved by at most \(\Delta\): gap above \(T\), score \(0\) & \pill & \([0.969,1.000]\) \\
\addlinespace[3pt]
\ref{hyp:sharp} & Gap \(\delta\leq T\): frozen degree-0 score \(\infty\) & \pill & \([0.987,1.000]\) \\
 & Truncated degree-0 score \(\min\{T-\delta,T/2\}\), \(0\) above \(T\) & \pill & \([0.995,1.000]\) \\
\addlinespace[3pt]
\ref{hyp:chains} & Two sets, \(T\)-Lebesgue: chain-level score \(0\) & \pill & \([0.978,1.000]\) \\
 & \(m\) arcs, \(T\)-Lebesgue: chain-level score \(0\) & \pill & \([0.956,1.000]\) \\
\addlinespace[3pt]
\ref{hyp:defect} & Homology-level signature equals the recursion & \pill & \([0.978,1.000]\) \\
\addlinespace[3pt]
\ref{hyp:degree-zero} & Two-set covers: degree-0 score \(0\) iff criterion & \pill & \([0.994,1.000]\) \\
 & Criterion holds: degree-0 score \(0\) & \pill & \([0.997,1.000]\) \\
 & Criterion fails: degree-0 score positive & \pill & \([0.997,1.000]\) \\
 & Coproduct at distance \(\infty\) & \pill & \([0.994,1.000]\) \\
\addlinespace[3pt]
\ref{hyp:collars} & Collars of radius \(T\): chain-level score \(0\) & \pill & \([0.990,1.000]\) \\
\addlinespace[3pt]
\ref{hyp:sweeps} & Sweeps: chain-level score \(0\) & \pill & \([0.992,1.000]\) \\
 & Sweeps: signature equals the recursion & \pill & \([0.992,1.000]\) \\
\addlinespace[3pt]
Bridge & Score \(0\) iff \(\delta>T\) iff distance head above \(T\) & \pill & \([0.995,1.000]\) \\
 & Smallest probe margin at most \(\delta\max_c\lVert w_c\rVert\) & \pill & \([0.995,1.000]\) \\
\bottomrule
\end{tabular}\space\space
\end{table}

\begin{table}[!htb]
\centering
\caption{Experiment (a), negative controls: mean truncated score, and the rows with infinite frozen score, with positive Betti score and with a truncated score above the resample.}
\label{tab:known-controls}
\tablestyle
\begin{tabular}{@{}llcccc@{}}
\toprule
\colhead{Experiment} & \colhead{Control} & \colhead{Truncated score} & \colhead{Frozen \(\infty\)} & \colhead{Betti positive} & \colhead{Above resample} \\
\midrule
Merges & Resample & \(0.107\) & \(1/60\) & \(1/60\) & \textendash \\
 & Twin & \(0.128\) & \(0/60\) & \(0/60\) & \(40/60\) \\
 & Other & \(0.392\) & \(60/60\) & \(60/60\) & \(60/60\) \\
\grouprule
Covers & Resample & \(0.113\) & \(2/140\) & \(2/140\) & \textendash \\
 & Twin & \(0.130\) & \(6/140\) & \(6/140\) & \(85/140\) \\
 & Other & \(0.358\) & \(139/140\) & \(139/140\) & \(139/140\) \\
\grouprule
Collars & Resample & \(0.107\) & \(3/400\) & \(3/400\) & \textendash \\
 & Twin & \(0.130\) & \(10/400\) & \(10/400\) & \(271/400\) \\
 & Other & \(0.382\) & \(400/400\) & \(400/400\) & \(398/400\) \\
\bottomrule
\end{tabular}\space\space
\end{table}

\begin{table}[!htb]
\centering
\caption{Experiment (a), \ref{hyp:defect}: the essential signatures of \(U_1\), \(U_2\), \(U_{12}\) and \(P\) at \(T\), of the homology-level extension computed directly and by the recursion of \itemref{thm:mayer-vietoris}{it:mv-betti}, and the classes lost.}
\label{tab:known-defects}
\tablestyle
\begin{tabular}{@{}llllllll@{}}
\toprule
\colhead{Cover} & \colhead{\(U_1\)} & \colhead{\(U_2\)} & \colhead{\(U_{12}\)} & \colhead{\(P\)} & \colhead{Extension} & \colhead{Recursion} & \colhead{Lost} \\
\midrule
Circle, arcs & \((1,0)\) & \((1,0)\) & \((2,0)\) & \((1,1)\) & \pill[l] & \((1,0)\) & \((0,1)\) \\
Figure eight, lobes & \((1,0)\) & \((1,1)\) & \((1,0)\) & \((1,1)\) & \pill[l] & \((1,1)\) & \((0,0)\) \\
Figure eight, lobes & \((1,1)\) & \((1,1)\) & \((1,0)\) & \((1,2)\) & \pill[l] & \((1,2)\) & \((0,0)\) \\
Genus two, neck & \((1,2,0)\) & \((1,2,0)\) & \((1,1,0)\) & \((1,4,1)\) & \pill[l] & \((1,4,0)\) & \((0,0,1)\) \\
Klein bottle, halves & \((1,1,0)\) & \((1,1,0)\) & \((2,2,0)\) & \((1,2,1)\) & \pill[l] & \((1,1,0)\) & \((0,1,1)\) \\
Sphere, caps & \((1,0,0)\) & \((1,0,0)\) & \((1,1,0)\) & \((1,0,1)\) & \pill[l] & \((1,0,0)\) & \((0,0,1)\) \\
Torus, chart & \((1,0,0)\) & \((1,2,0)\) & \((1,1,0)\) & \((1,2,1)\) & \pill[l] & \((1,2,0)\) & \((0,0,1)\) \\
Torus, chart & \((1,0,0)\) & \((1,3,0)\) & \((1,2,0)\) & \((1,2,1)\) & \pill[l] & \((1,2,0)\) & \((0,0,1)\) \\
Torus, chart & \((1,0,0)\) & \((1,4,0)\) & \((1,3,0)\) & \((1,2,1)\) & \pill[l] & \((1,2,0)\) & \((0,0,1)\) \\
Torus, halves & \((1,1,0)\) & \((1,1,0)\) & \((2,2,0)\) & \((1,2,1)\) & \pill[l] & \((1,1,0)\) & \((0,1,1)\) \\
\bottomrule
\end{tabular}\space\space
\end{table}

\begin{table}[!htb]
\centering
\caption{Experiment (a), the bridge on the merges per regime: mean class gap over \(T\), rows with frozen degree-0 score \(0\), mean truncated degree-0 score, smallest normalised probe margin over \(T\), mean probe accuracy and rows on which all bounds of \Cref{thm:bridge} hold. Below, Kendall's \(\tau\) over all \(900\) rows with its interval.}
\label{tab:known-bridge}
\tablestyle
\begin{tabular}{@{}lcccccc@{}}
\toprule
\colhead{Regime} & \colhead{\(\delta/T\)} & \colhead{Score \(0\)} & \colhead{Truncated} & \colhead{Margin} & \colhead{Accuracy} & \colhead{Bounds hold} \\
\midrule
Gap \(0.50\,T\) & \(0.50\) & \(0/60\) & \(0.314\) & \(-1.756\) & \(0.992\) & \pill \\
Gap \(0.90\,T\) & \(0.91\) & \(0/180\) & \(0.061\) & \(0.053\) & \(1.000\) & \pill \\
Gap \(0.99\,T\) & \(0.99\) & \(0/60\) & \(0.006\) & \(0.225\) & \(1.000\) & \pill \\
Gap \(1.01\,T\) & \(1.01\) & \(60/60\) & \(0.000\) & \(0.233\) & \(1.000\) & \pill \\
Gap \(1.10\,T\) & \(1.11\) & \(180/180\) & \(0.000\) & \(0.080\) & \(1.000\) & \pill \\
Gap \(2\,T\) & \(2.00\) & \(60/60\) & \(0.000\) & \(0.722\) & \(1.000\) & \pill \\
Far & \(9.23\) & \(180/180\) & \(0.000\) & \(1.044\) & \(1.000\) & \pill \\
Perturbed, \(\Delta=0.05\) & \(1.16\) & \(60/60\) & \(0.000\) & \(0.444\) & \(1.000\) & \pill \\
Perturbed, \(\Delta=0.10\) & \(1.23\) & \(60/60\) & \(0.000\) & \(0.581\) & \(1.000\) & \pill \\
\grouprule
\multicolumn{4}{l}{Truncated score against probe margin} & \multicolumn{1}{c}{\(\tau=-0.448\)} & \multicolumn{2}{c}{\([-0.479,-0.417]\)} \\
\multicolumn{4}{l}{Truncated score against probe accuracy} & \multicolumn{1}{c}{\(\tau=-0.155\)} & \multicolumn{2}{c}{\([-0.208,-0.085]\)} \\
\multicolumn{4}{l}{Class gap against probe margin} & \multicolumn{1}{c}{\(\tau=0.592\)} & \multicolumn{2}{c}{\([0.558,0.623]\)} \\
\bottomrule
\end{tabular}\space\space
\end{table}

\begin{figure}[!htb]
\centering
\resultrow
\caption{Experiment (a) against the gap \(\delta/T\) of a merge. Left: the truncated degree-0 score over \(T\), with the curve \(\min\{1-\delta/T,1/2\}\) of \ref{hyp:sharp} below \(T\) and \(0\) above by \ref{hyp:merge}. Right: the smallest margin of the trained linear probe, normalised by \(\max_c\lVert w_c\rVert\) and divided by \(T\), below the bound~\(\delta/T\)~of~\itemref{thm:bridge}{it:bridge-linear}.\keeptail}
\label{fig:known-topology}
\end{figure}

\subsection{Recovering a Change of Task}
Five or six source classes of about \(800\) points are drawn from six of the spaces. A hidden surjection onto two or three target objects builds each target from fresh samples of the classes it merges, at a gap of \(1.5\,T\) (separated) or \(0.5\,T\) (touching), at three noise levels, in \(\mathbb R^4\) or lifted to \(\mathbb R^{64}\), with \(40\) trials per cell and \(960\) per regime. We treat every surjection as a candidate and score it by the largest, over its target objects, of the truncated distance summed over degrees zero and one. The minimisers share the credit, \(1/k\) each when the truth is among \(k\) minimisers, so that chance is the mean of one over the number of surjections. A second selector breaks ties by the mean over the target objects. Covers of a circle by four arcs and of a torus by four sectors are recovered among all subsets of their nonempty overlaps. Each subset is scored at the chain level along its partial \v{C}ech diagram, and the selectors are the truncated score alone, the truncated score among the subsets with finite frozen score,~and~the~Betti~score.
\begin{enumerate}[hypothesis,start=9]
\item\label{hyp:recovery} \textbf{Merge Maps} (\Cref{prop:separated-vr-merge}). With separated targets the extension along the true merge is the barcode of a sample of the same union, so the true map is recovered far above chance and above the Betti score, which in degrees zero and one cannot tell a torus, a Klein bottle and a figure eight apart. With touching targets the prediction fails by \ref{hyp:sharp}, and recovery~drops~towards~chance.\keeptail
\item\label{hyp:covers} \textbf{Covers} (\Cref{prop:chain-gluing}). For a \(T\)-Lebesgue cover the full \v{C}ech category scores \(0\) at the chain level and is a minimiser. The covers of this experiment are not checked for the condition, so how often the truth scores \(0\) and how often it is recovered are measured.\tighter
\end{enumerate}

\textbf{Results.} With separated targets the truncated score recovers the hidden merge in \(68.0\,\%\) of the trials and in \(68.6\,\%\) with ties broken by the mean, against chance at \(1.4\,\%\) (\Cref{tab:recovery}). There the Betti score, which ties \(5.6\) candidates on average, recovers it in \(23.0\,\%\), and the local-scale statistic in \(38.8\,\%\). The objectwise score, which compares each source with the target without transport along \(J\), falls below chance, so the change of task is identified by the extension and not by the source invariants themselves. With touching targets the extension along the true map is not the observation, as \ref{hyp:sharp} predicts. Recovery drops to \(10.2\,\%\), still seven times chance, while the local-scale statistic, which is insensitive to the gap, keeps \(36.8\,\%\) and is the better selector in this regime. For the covers the truth scores \(0\) in \(121\) of \(240\) trials and is then always the only minimiser. The overlaps are too narrow for the condition at width \(0.02\) and, for the torus, at width \(0.1\). In these cases the truth is still recovered in most trials except for the torus at width \(0.02\), which stays at chance with every selector. The Betti score is perfect on the circle, where leaving out an overlap changes the essential signature, and on the torus it is comparable to the truncated score.

\begin{table}[!htb]
\centering
\caption{Experiment (b). Top: merge maps (\ref{hyp:recovery}), per regime and selector, the accuracy with credit \(1/k\) when the truth is among \(k\) minimisers, its bootstrap interval, the trials with a unique correct minimiser and the mean number of minimisers. Chance is \(0.014\) in both regimes. Bottom: covers (\ref{hyp:covers}), the trials in which the truth scores \(0\), the accuracy of each selector and chance.}
\label{tab:recovery}
\tablestyle
\begin{tabular}{@{}llcccc@{}}
\toprule
\colhead{Targets} & \colhead{Selector} & \colhead{Accuracy} & \colhead{\(95\,\%\) interval} & \colhead{Unique} & \colhead{Minimisers} \\
\midrule
Separated & \pill[l] & \(0.680\) & \([0.653,0.707]\) & \(575/960\) & \(1.17\) \\
 & \pill[l] & \(0.686\) & \([0.658,0.716]\) & \(659/960\) & \(1.00\) \\
 & Betti score & \(0.230\) & \([0.212,0.249]\) & \(80/960\) & \(5.62\) \\
 & Local scale & \(0.388\) & \([0.358,0.420]\) & \(336/960\) & \(1.35\) \\
 & Objectwise score & \(0.010\) & \([0.004,0.018]\) & \(10/960\) & \(4.65\) \\
 & Point count & \(0.073\) & \([0.056,0.091]\) & \(65/960\) & \(1.43\) \\
\grouprule
Touching & \pill[l] & \(0.102\) & \([0.087,0.119]\) & \(45/960\) & \(2.64\) \\
 & \pill[l] & \(0.105\) & \([0.087,0.126]\) & \(95/960\) & \(1.12\) \\
 & Betti score & \(0.040\) & \([0.032,0.049]\) & \(7/960\) & \(10.32\) \\
 & Local scale & \(0.368\) & \([0.338,0.396]\) & \(313/960\) & \(1.35\) \\
 & Objectwise score & \(0.012\) & \([0.005,0.019]\) & \(11/960\) & \(1.32\) \\
 & Point count & \(0.061\) & \([0.046,0.078]\) & \(54/960\) & \(1.37\) \\
\bottomrule
\end{tabular}\space\par\bigskip
\tablestyle
\begin{tabular}{@{}lccccc@{}}
\toprule
\colhead{Cover, width} & \colhead{Truth scores \(0\)} & \colhead{Truncated} & \colhead{Frozen filter} & \colhead{Betti} & \colhead{Chance} \\
\midrule
Circle, \(0.02\) & \(7/40\) & \(0.762\) & \(0.762\) & \(1.000\) & \(0.177\) \\
Circle, \(0.1\) & \(34/40\) & \(1.000\) & \(1.000\) & \(1.000\) & \(0.062\) \\
Circle, \(0.3\) & \(40/40\) & \(1.000\) & \(1.000\) & \(1.000\) & \(0.062\) \\
Torus, \(0.02\) & \(0/40\) & \(0.050\) & \(0.050\) & \(0.081\) & \(0.066\) \\
Torus, \(0.1\) & \(0/40\) & \(0.800\) & \(0.875\) & \(0.832\) & \(0.062\) \\
Torus, \(0.3\) & \(40/40\) & \(1.000\) & \(1.000\) & \(1.000\) & \(0.062\) \\
\bottomrule
\end{tabular}\space\space
\end{table}

\subsection{Gluing in Latent Spaces}
A circle of \(120\) points and a torus of \(800\), lifted to \(\mathbb R^{64}\) by a random orthonormal frame with ambient noise \(0.02\), are encoded into \(\mathbb R^3\) by principal component analysis, Isomap \citep{tenenbaum2000global}, an autoencoder, a denoising autoencoder \citep{vincent2008extracting} and a topological autoencoder \citep{moor2020topological}, \(20\) seeds each. None of the encoders is given a cover. The covers are fixed in the coordinates of the data and consist of two arcs of the circle at seven overlap widths from \(0.70\) to \(0.05\), two halves of the torus at four widths and a chart with its complement. We evaluate them on the latent cloud, centred and rescaled to the radius of the data, so that \(T\) keeps its meaning.
\begin{enumerate}[hypothesis,start=11]
\item\label{hyp:latent} \textbf{Latent Gluing} (\Cref{prop:chain-gluing,thm:mayer-vietoris,prop:cover-gluing,thm:bridge}). On every latent cloud and cover the degree-zero score vanishes iff the criterion holds, and the coproduct is at \(\infty\). If the latent cover is \(T\)-Lebesgue, the chain-level extension is exact and the homology-level defect equals the recursion. The cover is \(T\)-Lebesgue iff a \(1\)-Lipschitz head separates its exclusive parts with margin above \(T\) (end of \Cref{sec:bridge}). How often an encoder yields a \(T\)-Lebesgue latent cover is measured per width.
\end{enumerate}

\textbf{Results.} All statements hold on every latent cloud and cover (\Cref{tab:latent}), and the latent covers are \(T\)-Lebesgue in \(359\) of \(1{,}200\) cases. The rate follows the geometry of the data (\Cref{fig:latent}). On the noise-free circle, for example, the exclusive parts of the arcs at width \(w\) lie \(2\sin w\) apart, more than \(T\) iff \(w\) exceeds \(\arcsin(T/2)\approx0.33\). Every encoder keeps the condition at the width \(0.70\), in at least \(19\) of \(20\) seeds at \(0.45\), and loses it below \(0.30\). At \(0.35\), just above the threshold, principal components and Isomap keep it in \(19\) of \(20\) seeds and the three autoencoders in \(6\), \(6\) and \(3\), and at \(0.30\) Isomap keeps it in \(16\). For the torus only the widest halves survive, in \(89\) of \(100\) cases, and the chart in none. In particular, near the threshold a learned representation can violate the gluing hypothesis although the data satisfy it. Off the condition the degree-zero criterion still holds in \(728\) of \(841\) cases and the chain-level extension is exact in every degree in \(588\), so the condition is sufficient but far from necessary. The circle keeps its signature \((1,1)\) under every encoder. The torus keeps \((1,2,1)\) in every seed only under principal components, and in \(3\) of the \(80\) clouds of the nonlinear encoders, which lose a loop in \(76\) of them and create spurious voids in \(63\).

\begin{table}[!htb]
\centering
\caption{Experiment (c), \ref{hyp:latent}: each statement on every latent cloud and cover, and, below, the measured rates, with exact \(95\,\%\) intervals.}
\label{tab:latent}
\tablestyle
\begin{tabular}{@{}lcc@{}}
\toprule
\colhead{Check on every latent cloud and cover} & \colhead{Rows} & \colhead{\(95\,\%\) interval} \\
\midrule
Degree-0 score \(0\) iff criterion & \pill & \([0.996,1.000]\) \\
Coproduct at distance \(\infty\) & \pill & \([0.996,1.000]\) \\
\(T\)-Lebesgue: chain-level score \(0\) & \pill & \([0.989,1.000]\) \\
\(T\)-Lebesgue: signature equals the recursion & \pill & \([0.989,1.000]\) \\
\(T\)-Lebesgue iff class gap of the exclusive parts above \(T\) & \pill & \([0.996,1.000]\) \\
Merge score \(0\) iff gap above \(T\) iff distance head above \(T\) & \pill & \([0.996,1.000]\) \\
Smallest probe margin at most \(\delta\max_c\lVert w_c\rVert\) & \pill & \([0.996,1.000]\) \\
Margins at a closest pair sum to at most \(\lVert w_y-w_{y'}\rVert\delta\) & \pill & \([0.996,1.000]\) \\
\(\delta\leq T\): smallest probe margin at most \(T\max_c\lVert w_c\rVert\) & \pill & \([0.996,1.000]\) \\
\grouprule
Latent cover is \(T\)-Lebesgue & \(359/1{,}200\) & \([0.273,0.326]\) \\
Degree-0 criterion holds & \(1{,}087/1{,}200\) & \([0.887,0.922]\) \\
Latent cloud keeps the signature of the data & \(815/1{,}200\) & \([0.651,0.706]\) \\
\bottomrule
\end{tabular}\space\space
\end{table}

\begin{figure}[!htb]
\centering
\resultrow
\caption{Experiment (c), the rate of \(T\)-Lebesgue latent covers against the overlap width, per encoder, over \(20\) seeds. On the noise-free circle the exclusive parts lie more than \(T\) apart above the width \(\arcsin(T/2)\approx0.33\).}
\label{fig:latent}
\end{figure}

\subsection{Pretrained Image Models}
Fourteen ImageNet classifiers, convolutional networks and vision transformers (\Cref{app:details}), supply unit-normalised penultimate features, the inputs of their last linear layer, on CIFAR-100 \citep{krizhevsky2009learning}, \(100\) training and \(100\) test images per fine class, and on COIL-100 \citep{nene1996columbia}, \(100\) objects every \(5^\circ\). The \emph{CIFAR merge} has the \(100\) fine classes as source task, discrete, with \(50\) training images each, and the \(20\) superclasses as target task, discrete, observed on \(50\) test images per member class. Here \(J\) sends each fine class to its superclass. The \emph{COIL merge} sends the \(100\) objects, observed on their even poses, to \(20\) random groups of five, observed on the odd poses. The \emph{COIL gluing} glues the arcs \([-40^\circ,220^\circ]\) and \([140^\circ,400^\circ]\) of each rotation orbit along their overlap to the orbit. The scale \(T\) of a model is the \(0.1\)-quantile of the distances within the CIFAR source classes, respectively within the COIL orbits, so that the verdicts of the frozen score do not change when a representation is rescaled. Sensitivity is checked with the \(0.05\)- and \(0.2\)-quantiles. The score is the truncated score in degrees zero and one, and the negative control assigns to each target object the observation of another, by a fixed permutation without fixed points. The outcomes are the accuracy and the mean normalised margin \((s_y-s_j)/\lVert w_y-w_j\rVert\), with \(j\) the best rival, of a linear probe on \(256\) principal components, for the coarse CIFAR labels, for the COIL objects trained on the poses below \(180^\circ\), and for eight pose bins trained on half of the objects. The bridge checks of \itemref{thm:bridge}{it:bridge-equiv} to \labelcref{it:bridge-linear} run on every target object. Four hypotheses were fixed before any feature was extracted:
\begin{enumerate}[hypothesis,label=\textup{D\arabic*},ref=\textup{D\arabic*}]
\item\label{hyp:d1} Across the \(14\) models, Kendall's \(\tau\) between the truncated score of the CIFAR merge and the coarse probe accuracy is negative.
\item\label{hyp:d2} The same for the COIL merge and the object probe accuracy.
\item\label{hyp:d3} The same for the COIL gluing and the pose probe accuracy.
\item\label{hyp:d4} For each merge the intended target scores below the permuted control for more models than~expected~by~chance.
\end{enumerate}
No theorem implies \ref{hyp:d1} to \ref{hyp:d3}. \Cref{thm:bridge} constrains margins on the sample from which the score is computed, whereas these hypotheses ask whether the structural defect also orders models by their accuracy on held-out images. In particular, a failure of these hypotheses would contradict no proved statement. \ref{hyp:d4} rests on \Cref{prop:separated-vr-merge} and is expected to hold where the target groups are separated at \(T\). We did not pre-register any superiority~over~the~baselines.\keeptail

\textbf{Results.} \ref{hyp:d1} to \ref{hyp:d3} are not supported (\Cref{tab:transfer-primary}). For the CIFAR merge the association has the opposite sign, \(\tau=0.473\) with interval \([0.048,0.781]\): models with a larger discrepancy have the better probes. The score is not divided by \(T\), as pre-registered, and \(T\) varies between \(0.61\) and \(0.91\) across the models. A post hoc analysis, which was not pre-registered, shows that the association is carried by \(T\): across models \(\tau(\text{score},T)=0.648\) and \(\tau(T,\text{accuracy})=0.516\), while the score divided by \(T\) has \(\tau=-0.121\) with the accuracy (\Cref{fig:transfer}). For the COIL merge the object probe is at its ceiling, with accuracies between \(0.974\) and \(0.996\), and for the COIL gluing the pose probe is close to chance, between \(0.18\) and \(0.22\) for eight bins, so neither outcome orders the models. \ref{hyp:d4} is supported for both merges: the intended target scores below the permuted control for \(13\) and \(14\) of the \(14\) models, with \(p=9.2\times10^{-4}\) and \(6.1\times10^{-5}\). Among the baselines (\Cref{tab:transfer-scores}), \(20\)-nearest-neighbour accuracy, LogME and NCE track the coarse CIFAR accuracy closely, with \(\tau=0.91\), \(0.87\) and \(0.80\), as they were designed to do. On the ceiling and near-chance outcomes of COIL only one baseline association has an unadjusted two-sided permutation \(p\) below \(0.05\): that of \(20\)-nearest-neighbour accuracy with the object probe of the COIL merge, \(\tau=0.51\) with interval \([0.00,0.89]\) and \(p=0.014\). Pooled over the \(280\) superclasses of all models, the truncated degree-zero score of a superclass and the accuracy of the fine-label probe within it are associated with \(\tau=-0.407\), interval \([-0.477,-0.331]\), an exploratory result in the direction of \Cref{thm:bridge}. Every bridge check holds on all \(5{,}880\) target objects of all models and scales.

\begin{table}[!htb]
\centering
\caption{Experiment (d), the pre-registered hypotheses: Kendall's \(\tau\) across the \(14\) models with bootstrap interval, one-sided permutation \(p\), Holm-adjusted \(p\) and verdict, and for \ref{hyp:d4} the sign test.}
\label{tab:transfer-primary}
\tablestyle
\begin{tabular}{@{}llccccc@{}}
\toprule
 & \colhead{Task change} & \colhead{\(\tau\)} & \colhead{\(95\,\%\) interval} & \colhead{One-sided \(p\)} & \colhead{Holm \(p\)} & \colhead{Verdict} \\
\midrule
\ref{hyp:d1} & CIFAR-100 merge & \(0.473\) & \([0.048,0.781]\) & \(0.993\) & \(1\) & \pill \\
\ref{hyp:d2} & COIL-100 merge & \(0.034\) & \([-0.476,0.519]\) & \(0.589\) & \(1\) & \pill \\
\ref{hyp:d3} & COIL-100 gluing & \(0.055\) & \([-0.326,0.506]\) & \(0.631\) & \(1\) & \pill \\
\grouprule
\ref{hyp:d4} & CIFAR-100 merge & \multicolumn{2}{c}{Below control: \(13\) of \(14\)} & \(9.16\times10^{-4}\) & \textendash & \pill \\
\ref{hyp:d4} & COIL-100 merge & \multicolumn{2}{c}{Below control: \(14\) of \(14\)} & \(6.10\times10^{-5}\) & \textendash & \pill \\
\bottomrule
\end{tabular}\space\space
\end{table}

\begin{table}[!htb]
\centering
\caption{Experiment (d), Kendall's \(\tau\) of every score with the probe accuracy across the \(14\) models, with bootstrap interval. The truncated mean replaces the largest value over the target objects by the mean, and the column Sign gives the sign a useful score should have.}
\label{tab:transfer-scores}
\tablestyle
\begin{tabular}{@{}llccc@{}}
\toprule
\colhead{Score} & \colhead{Sign} & \colhead{CIFAR-100 merge} & \colhead{COIL-100 merge} & \colhead{COIL-100 gluing} \\
\midrule
Truncated score & \(-\) & \(0.47\)\ \([0.04,0.79]\) & \(0.03\)\ \([-0.48,0.52]\) & \(0.05\)\ \([-0.33,0.51]\) \\
Truncated, mean & \(-\) & \(0.41\)\ \([0.03,0.71]\) & \(0.28\)\ \([-0.15,0.68]\) & \(-0.08\)\ \([-0.51,0.40]\) \\
Truncated, degree 0 & \(-\) & \(0.36\)\ \([-0.05,0.70]\) & \(0.19\)\ \([-0.34,0.68]\) & \textendash \\
Truncated, degree 1 & \(-\) & \(0.32\)\ \([-0.14,0.68]\) & \(-0.03\)\ \([-0.47,0.41]\) & \textendash \\
Exact gluings & \(+\) & \textendash & \textendash & \(0.00\)\ \([-0.53,0.49]\) \\
\(T\)-Lebesgue covers & \(+\) & \textendash & \textendash & \(0.07\)\ \([-0.41,0.51]\) \\
\grouprule
LogME & \(+\) & \(0.87\)\ \([0.71,1.00]\) & \(0.12\)\ \([-0.33,0.61]\) & \(-0.08\)\ \([-0.53,0.40]\) \\
LEEP & \(+\) & \(0.65\)\ \([0.29,0.91]\) & \(-0.06\)\ \([-0.50,0.43]\) & \(0.38\)\ \([-0.07,0.78]\) \\
NCE & \(+\) & \(0.80\)\ \([0.58,0.98]\) & \(-0.12\)\ \([-0.56,0.33]\) & \(-0.21\)\ \([-0.69,0.23]\) \\
H-score & \(+\) & \(0.67\)\ \([0.34,0.91]\) & \(0.12\)\ \([-0.40,0.68]\) & \(-0.08\)\ \([-0.44,0.30]\) \\
\(20\)-NN accuracy & \(+\) & \(0.91\)\ \([0.75,1.00]\) & \(0.51\)\ \([0.00,0.89]\) & \(0.01\)\ \([-0.49,0.52]\) \\
Silhouette & \(+\) & \(0.74\)\ \([0.48,0.94]\) & \(0.33\)\ \([-0.21,0.76]\) & \(0.21\)\ \([-0.31,0.67]\) \\
Linear CKA & \(+\) & \(0.71\)\ \([0.40,0.93]\) & \(0.35\)\ \([-0.16,0.75]\) & \(-0.01\)\ \([-0.53,0.50]\) \\
RTD & \(-\) & \(0.25\)\ \([-0.12,0.60]\) & \(-0.33\)\ \([-0.75,0.25]\) & \(0.25\)\ \([-0.20,0.70]\) \\
\bottomrule
\end{tabular}\space\space
\end{table}

\begin{figure}[!htb]
\centering
\resultrow
\caption{Experiment (d), the CIFAR merge, one point per model: the coarse probe accuracy against the truncated score (left, \ref{hyp:d1}, \(\tau=0.473\)) and, post hoc, against the truncated score divided by the scale \(T\) of the model (right, \(\tau=-0.121\)).}
\label{fig:transfer}
\end{figure}

\subsection{Cost}
\begin{enumerate}[hypothesis,start=12]
\item\label{hyp:cost} \textbf{Cost} (\Cref{prop:cokernel,prop:chain-gluing}). The extension of a merge is a direct sum of source barcodes, that of a gluing the persistence of a colimit complex no larger than the union of the parts, so it costs a small multiple of the observed persistence, and the ambient dimension enters only through distances. Fifteen cases, up to twelve comma objects and \(5{,}000\) points in \(\mathbb R^{2048}\), are timed in the median of three repetitions.
\end{enumerate}

\textbf{Results.} The extension of a merge costs at most what the observed persistence costs, and that of a gluing between \(1.9\) and \(4.1\) times as much (\Cref{tab:runtime}). The ambient dimension enters only through the distances. For example, the halves of a torus take \(30.6\), \(30.7\) and \(28.7\) milliseconds in \(\mathbb R^3\), \(\mathbb R^{64}\) and \(\mathbb R^{512}\). The edge collapse removes up to \(99.8\,\%\) of the simplices. The largest case, \(5{,}000\) features in \(\mathbb R^{2048}\) merged from \(100\) classes into \(20\), takes \(17\) milliseconds for the extension and \(136\) for the score.

\begin{table}[!htb]
\centering
\caption{Experiment (e), \ref{hyp:cost}: points, ambient dimension, simplices of the frozen complexes of the extension before and after the edge collapse, in thousands, and the median milliseconds for the extension, the observed persistence and the score.}
\label{tab:runtime}
\tablestyle
\begin{tabular}{@{}lccrrrrr@{}}
\toprule
\colhead{Case} & \colhead{Points} & \colhead{\(D\)} & \colhead{Simplices} & \colhead{Collapsed} & \colhead{Extension} & \colhead{Observed} & \colhead{Score} \\
\midrule
\multicolumn{8}{@{}l}{\tablegroup{Merges}} \\
Circle and sphere & \(470\) & \(4\) & \(1{,}545.4\) & \(2.9\) & \(8.7\) & \(8.7\) & \(0.9\) \\
Torus and Klein bottle & \(1{,}600\) & \(4\) & \(3{,}475.1\) & \(11.9\) & \(31.2\) & \(35.4\) & \(2.6\) \\
Six spaces & \(3{,}230\) & \(4\) & \(969.7\) & \(21.1\) & \(59.1\) & \(78.2\) & \(148.2\) \\
Five feature classes & \(250\) & \(512\) & \(1.5\) & \(1.1\) & \(1.2\) & \(1.3\) & \(7.9\) \\
Five feature classes & \(250\) & \(2{,}048\) & \(1.6\) & \(1.1\) & \(0.9\) & \(1.8\) & \(7.8\) \\
Features, 100 into 20 & \(5{,}000\) & \(2{,}048\) & \(50.4\) & \(27.6\) & \(17.2\) & \(17.2\) & \(136.0\) \\
\grouprule
\multicolumn{8}{@{}l}{\tablegroup{Gluings}} \\
Circle, two arcs & \(120\) & \(3\) & \(9.7\) & \(0.2\) & \(1.3\) & \(0.4\) & \(0.4\) \\
Torus, two halves & \(800\) & \(3\) & \(1{,}807.3\) & \(5.9\) & \(30.6\) & \(15.6\) & \(1.4\) \\
Torus, two halves & \(800\) & \(64\) & \(1{,}763.3\) & \(6.0\) & \(30.7\) & \(15.6\) & \(1.4\) \\
Torus, two halves & \(800\) & \(512\) & \(1{,}484.6\) & \(5.6\) & \(28.7\) & \(14.4\) & \(1.4\) \\
Sphere, two caps & \(350\) & \(3\) & \(1{,}340.6\) & \(2.7\) & \(14.9\) & \(7.8\) & \(0.9\) \\
Genus two, neck & \(1{,}000\) & \(3\) & \(1{,}789.1\) & \(7.3\) & \(37.7\) & \(19.1\) & \(1.6\) \\
Circle, six arcs & \(120\) & \(3\) & \(9.0\) & \(0.2\) & \(1.6\) & \(0.4\) & \(0.4\) \\
Torus, two collars & \(800\) & \(3\) & \(1{,}720.0\) & \(6.0\) & \(30.4\) & \(15.5\) & \(1.5\) \\
Torus, three sectors & \(800\) & \(3\) & \(1{,}679.2\) & \(5.9\) & \(31.1\) & \(15.4\) & \(1.4\) \\
\bottomrule
\end{tabular}\space\space
\end{table}

\section{Conclusion}\label{sec:conclusion}
The transfer criterion makes the change of task part of the evaluation of a learned representation. The reference is the left Kan extension of the source invariant, through which every comparison of the source with an observation passes, and the transfer discrepancy measures how far the observation departs from this extension. Its vanishing is strictly weaker than a natural isomorphism and depends on the morphisms of the source (\Cref{prop:pw-not-nat,prop:structure,thm:main}). For persistent homology the extension reduces to finite linear algebra, the functorial discrepancy is \(1\)-Lipschitz, and the score is computed as a bottleneck distance (\Cref{prop:cokernel,thm:lan-interleaving-stability,prop:correctness}). When data are glued from two sets whose exclusive parts lie farther apart than the freezing scale, the Mayer--Vietoris sequence measures what is lost when homology is extended, while nothing is lost when chains are extended (\Cref{thm:mayer-vietoris,prop:chain-gluing}). For a merge of labelled clouds, a vanishing degree-zero discrepancy is equivalent to a \(1\)-Lipschitz margin above \(T\) on the sample (\Cref{thm:bridge}). The equivalence concerns the training sample and not held-out accuracy. The experiments confirm the exact predictions on every sample, recover hidden changes of task when the merged parts are separated and distinguish intended targets from permuted controls. Across pretrained models, however, the truncated score, which depends on a scale chosen for each model, does not track the accuracy of a linear probe.

The criterion depends on the choice of \(J\), which is a modelling input like the symmetry group of an equivariant network, and on the choice of \(T\). All statements about data concern the finite sample. None of them concerns a population, consistency as the sample grows, or held-out data. The homology up to degree \(q\) needs the \((q+1)\)-skeleta of the Vietoris--Rips complexes, whose size can grow like the number of points to the power \(q+2\), so the computation is confined to low degrees. The edge collapse reduces this cost but does not remove it. The baselines of (b) do not include the persistent homology of the pooled source points, which coincides with the extension whenever the pooled parts lie pairwise more than \(T\) apart (\Cref{prop:separated-vr-merge}). The comparison of \ref{hyp:d4} has no baseline. In the cross-model test (d) there are \(14\) models, two of the three probe accuracies barely vary, and the scale is chosen per model. Over \(\mathbb F_2\) a torus and a Klein bottle, and linked and unlinked circles, have the same essential Betti numbers, so the frozen score cannot separate them through an essential class. Whether the finite bars separate them is measured in \ref{hyp:controls}.

\clearpage
\appendix

\section{The Transfer Score Algorithm}\label{app:algorithm}

\begin{algorithm2e}[H]
\caption{The Kan-persistent transfer score \(\Comp^{d_B}_{J,q}\). Lines~2, 5 and~8 build \(F_n\) and \(G_n\), and lines~12 and~13 enumerate \(J\downarrow b\). Lines~16 to~19 form \eqref{eq:cokernel} at the finitely many scales of line~19, so that \(M\cong(\Lan_JF_n)(b)\), and lines~20 to~24 compare the barcodes and aggregate.}\label{alg:score}
\KwInput{Finite categories \(\A\) and \(\B\), a functor \(J:\A\to\B\), a field \(\K\), a maximal homological degree \(q\in\mathbb N\), weights \(w_0,\ldots,w_q\geq0\), a functor \(\mathcal K^{\A}:\A\to\Fun{(\mathbb R,\leq)}{\operatorname{Simp}}\) with \(\mathcal K^{\A}(a)\) a finite filtration for every \(a\in\Ob(\A)\), and a finite filtration \(\mathcal K^{\B}(b):(\mathbb R,\leq)\to\operatorname{Simp}\) for every \(b\in\Ob(\B)\).}
\KwOutput{The score \(\Comp^{d_B}_{J,q}\in[0,\infty]\).}
\ForEach{\(a\in\Ob(\A)\) and \(0\leq n\leq q\)}{
  \(F_n(a)\leftarrow\PHfun(\mathcal K^{\A}(a),n,\K)\in\Pers\)\eol
}
\ForEach{\(\bar\alpha\in\Mor(\A)\) with \(\bar\alpha:a\to a'\), and \(0\leq n\leq q\)}{
  \(F_n(\bar\alpha)\leftarrow H_n(\mathcal K^{\A}(\bar\alpha);\K):F_n(a)\to F_n(a')\)\eol
}
\ForEach{\(b\in\Ob(\B)\) and \(0\leq n\leq q\)}{
  \(G_n(b)\leftarrow\PHfun(\mathcal K^{\B}(b),n,\K)\in\Pers\)\eol
}
\(\mathrm{score}\leftarrow0\)\eol
\ForEach{\(b\in\Ob(\B)\)}{
  \(\Ob(J\downarrow b)\leftarrow\bigl\{(a,\beta)\bigm|a\in\Ob(\A),\ \beta\in\B(Ja,b)\bigr\}\)\eol
  \(\Mor(J\downarrow b)\leftarrow\bigl\{\bar\alpha:(a,\beta)\to(a',\beta')\bigm|\bar\alpha\in\A(a,a'),\ \beta'\circ J\bar\alpha=\beta\bigr\}\)\eol
  \(\Omega_b\leftarrow0\)\eol
  \For{\(n\leftarrow0\) \KwTo \(q\)}{
    \(P\leftarrow\bigoplus_{(a,\beta)\in\Ob(J\downarrow b)}F_n(a)\) in \(\Pers\), with injections \(\iota_{(a,\beta)}:F_n(a)\to P\)\eol
    \(R\leftarrow\bigoplus_{\gamma\in\Mor(J\downarrow b)}F_n(a_\gamma)\), with \(\operatorname{dom}\gamma=(a_\gamma,\beta_\gamma)\) and \(\kappa_\gamma:F_n(a_\gamma)\to R\)\eol
    \(r\leftarrow\) the unique \(r:R\to P\) with \(r\kappa_\gamma=\iota_{\operatorname{cod}\gamma}F_n(\bar\gamma)-\iota_{\operatorname{dom}\gamma}\) for all \(\gamma\), where \(\bar\gamma\coloneqq\pi_b(\gamma)\)\eol
    \(M\leftarrow\Ckfun(r)\), at the scales of \(S_b\coloneqq\bigcup_{(a,\beta)\in\Ob(J\downarrow b)}S(\mathcal K^{\A}(a))\) and at one scale in each component of \(\mathbb R\setminus S_b\), with transition maps induced by those of \(P\)\eol
    \(\mathcal D_1\leftarrow\Dgm(\Bcfun(M))\) and \(\mathcal D_2\leftarrow\Dgm(\Bcfun(G_n(b)))\)\eol
    \(d_{n,B}(b)\leftarrow\Bnfun(\mathcal D_1,\mathcal D_2)\)\eol
    \(\Omega_b\leftarrow \Omega_b+w_nd_{n,B}(b)\)\eol
  }
  \(\mathrm{score}\leftarrow\max\{\mathrm{score},\Omega_b\}\)\eol
}
\Return \(\mathrm{score}\)
\end{algorithm2e}

\section{Experimental Details}\label{app:details}
\Cref{tab:samplers} lists the samplers of experiment (a) with their controls, and \Cref{tab:parameters} the parameters of all experiments. Each coordinate carries Gaussian noise of standard deviation \(0.012\) unless a sweep varies it. The Klein bottle is
\[\bigl((R+r\cos v)\cos u,\,(R+r\cos v)\sin u,\,r\sin v\cos\tfrac u2,\,r\sin v\sin\tfrac u2\bigr)\]
with \((R,r)=(1.20,0.45)\). The genus-two surface consists of two tori with radii \((1.0,0.40)\) and centres \(2.4\) apart, without the parts inside the other tube.

\begin{table}[!htb]
\centering
\caption{The samplers of experiment (a) and their controls.}
\label{tab:samplers}
\tablestyle
\begin{tabular}{@{}l>{\raggedright\arraybackslash}p{4.7cm}r>{\raggedright\arraybackslash}p{2.9cm}l@{}}
\toprule
\colhead{Space} & \colhead{Construction} & \colhead{Points} & \colhead{Twin} & \colhead{Other} \\
\midrule
Circle & radius \(1\), equispaced angles jittered by \(\mathcal N(0,0.02^2)\) & \(120\) & radius \(1.3\) & figure eight\\
Torus & radii \((1.20,0.45)\), uniform for the area & \(800\) & Klein bottle & sphere\\
Sphere & radius \(0.75\), uniform & \(350\) & radius \(0.85\) & torus\\
Klein bottle & in \(\mathbb R^4\), uniform for the area & \(800\) & torus & sphere\\
Genus-two surface & two tori, tube radius \(0.40\) & \(1{,}000\) & tube radius \(0.45\) & torus\\
Figure eight & \((2\sin t,2\sin t\cos t,0)\), uniform in arc length & \(160\) & scaled by \(1.2\) & circle\\
Hopf link & two linked unit circles, uniform in arc length & \(160\) & unlinked, centres \(4\) apart & circle\\
\bottomrule
\end{tabular}
\end{table}

\begin{table}[!htb]
\centering
\caption{The parameters of the experiments.}
\label{tab:parameters}
\tablestyle
\begin{tabular}{@{}l>{\raggedright\arraybackslash}p{11.5cm}@{}}
\toprule
\colhead{Experiment} & \colhead{Parameters} \\
\midrule
(a) Known topology & \(20\) seeds, \(40\) for the degree-zero criterion and \(10\) for the sweeps. Two-set covers at the overlap widths \(\{0.05,0.1,0.2,0.4,0.7\}\) for arcs and lobes, \(\{0.05,0.1,0.2,0.35,0.6\}\) for halves of the torus and the Klein bottle, \(\{0.1,0.25,0.5,1.0,1.5\}\) for a chart of the torus, and \(\{0.02,0.05,0.1,0.2,0.4\}\) for caps of the sphere, the neck of the genus-two surface and covers by \(m\geq3\) arcs or sectors. The bridge probe is logistic regression with inverse regularisation \(100\), written as a linear head in the original coordinates.\\
\addlinespace[3pt]
(b) Recovery & Merges of five or six source classes onto two or three targets, gaps \(1.5\,T\) and \(0.5\,T\), noise \(\{0.012,0.03,0.06\}\) on the targets, \(\mathbb R^4\) or \(\mathbb R^{64}\) with lift noise \(0.005\), \(40\) trials per cell. Covers by four arcs or sectors at the widths \(\{0.02,0.1,0.3\}\) and noise \(\{0.012,0.06\}\), \(20\) trials per cell.\\
\addlinespace[3pt]
(c) Latent gluing & \(20\) seeds. Autoencoders \(\mathbb R^{64}\to\mathbb R^{128}\to\mathbb R^{64}\to\mathbb R^3\) with rectified linear units and a mirrored decoder, trained for \(400\) full-batch steps of Adam \citep{kingma2015adam} at rate \(10^{-3}\), the denoising one with input noise \(0.05\), the topological one with weight \(0.5\) on the degree-zero signature loss. Isomap with \(10\) neighbours. The chart of the torus has width \(1.0\).\\
\addlinespace[3pt]
(d) Image models & Probe: multinomial logistic regression on \(256\) standardised principal components. RTD averaged over \(10\) subsamples of \(200\) points, silhouette on \(2{,}500\) test points, \(10{,}000\) bootstrap rounds and permutations.\\
\addlinespace[3pt]
(e) Cost & Median of three repetitions per case, on the machine of \Cref{sec:experiments}.\\
\bottomrule
\end{tabular}
\end{table}

{\raggedright The models of (d) are \texttt{resnet18}, \texttt{resnet50}, \texttt{densenet121}, \texttt{googlenet}, \texttt{shufflenet\_v2\_x1\_0}, \texttt{mobilenet\_v3\_large}, \texttt{efficientnet\_b0}, \texttt{regnet\_y\_800mf}, \texttt{convnext\_tiny}, \texttt{swin\_t}, \texttt{vit\_b\_32} and \texttt{vit\_b\_16} from torchvision with their ImageNet weights, and \texttt{deit\_small\_patch16\_224.fb\_in1k} and \texttt{vit\_small\_patch16\_224.augreg\_in21k\_ft\_in1k} from timm, each with the preprocessing that accompanies its weights.\par}

\textbf{Reproduction and Availability.} The code is hosted at \href{https://codeberg.org/Jiren/CatTrans}{\nolinkurl{codeberg.org/Jiren/CatTrans}}.

\end{document}